\documentclass[12pt,dvipsnames]{colt2024} 

\title[Physics-informed machine learning as a kernel method]{Physics-informed machine learning as a kernel method}
\usepackage{times}

\usepackage{mathrsfs}
\usepackage{comment}
\usepackage{wrapfig}

\numberwithin{theorem}{section} 
\newtheorem{prop}[theorem]{Proposition}

\newtheorem{rem}[theorem]{Remark}
\newtheorem{lem}[theorem]{Lemma}
\newtheorem{defi}[theorem]{Definition}

\coltauthor{%
 \Name{Nathan Doumèche} \Email{nathan.doumeche@sorbonne-universite.fr}\\
 \addr Laboratory of Probability, Statistics, and Modeling, Sorbonne University, France 
 \AND
 \Name{Francis Bach} \Email{francis.bach@inria.fr}\\
 \addr Inria, Ecole Normale Supérieure,
PSL Research University, France
 \AND
 \Name{Gérard Biau} \Email{gerard.biau@sorbonne-universite.fr}\\
 \addr Laboratory of Probability, Statistics, and Modeling, Sorbonne University, France 
 \AND
 \Name{Claire Boyer} \Email{claire.boyer@sorbonne-universite.fr}\\
 \addr IUF, Laboratory of Probability, Statistics, and Modeling, Sorbonne University, France 
}

\usepackage{ulem}

\newcommand{\gb}[1]{\textcolor{Orange}{GB: #1}}

\begin{document}

\maketitle

\begin{abstract} 
Physics-informed machine learning combines the expressiveness of data-based approaches with the interpretability of physical models. In this context, we consider a general regression problem where the empirical risk is regularized by a partial differential equation that quantifies the physical inconsistency. We prove that for linear differential priors, the problem can be formulated as a kernel regression task. Taking advantage of kernel theory, we derive convergence rates for the minimizer~$\hat f_n$ of the regularized risk and show that $\hat f_n$ converges at least at the Sobolev minimax rate. However, faster rates can be achieved, depending on the physical error. This principle is illustrated with a one-dimensional example, supporting the claim that regularizing the empirical risk with physical information can be beneficial to the statistical performance of estimators.
\end{abstract}


\section{Introduction}

\paragraph{Physics-informed machine learning.} Physics-informed machine learning (PIML) refers to a subdomain of machine learning that combines physical knowledge and empirical data to enhance performance of tasks involving a physical mechanism. Following the influential work of \citet{raissi2019PINN}, the field has experienced a notable surge in popularity, largely driven by scientific computing and engineering applications. We refer the reader to the surveys by \citet{rai2020review}, \citet{karniadakis2021piml}, \citet{cuomo2022scientific}, and \citet{Hao2022review}.
In a nutshell, the success of PIML relies on the smart interaction between machine learning and physics. In its most standard form, this achievement is realized by integrating physical equations into the loss function. Three common use cases include solving systems of partial differential equations (PDEs), addressing inverse problems (e.g., learning the PDE governing an observed phenomenon), and further improving the statistical performance of empirical risk minimization. 
This article focuses on the latter approach, known as hybrid modeling \citep[e.g.,][]{rai2020review}. 

\paragraph{Hybrid modeling.} Consider the classical regression model $Y = f^\star(X) + \varepsilon$, where the function $f^\star: \mathbb R^d \to \mathbb R$ is unknown. The random variable $Y \in \mathbb{R}$ is the target, the random variable $X\in \Omega \subseteq [-L,L]^d$ the vector of features, and $\varepsilon$ a random noise. Given a sample $\{(X_1, Y_1), \hdots, (X_n, Y_n)\}$ of i.i.d.~copies of $(X,Y)$, the goal is to construct an estimator $\hat f_n$ of $f^\star$ based on these $n$ observations. The distinctive element of PIML is the inclusion of a prior on $f^\star$, asserting its compliance with a known PDE. Therefore, it is assumed that $f^\star$ is at least weakly differentiable, belonging to the Sobolev space $H^s(\Omega)$ for some integer $s > d/2$, and that there is a known differential operator $\mathscr{D}$ such that $\mathscr{D}(f^\star) \simeq 0$. For instance, if the desired solution $f^\star$ is intended to conform to the wave equation, then $\mathscr{D}(f)(x,t) = \partial^2_{t,t} f(x,t) - \partial^2_{x,x} f(x,t)$ for $(x,t)\in \Omega$. Overall, we are interested in the minimizer of the empirical risk function 
\begin{equation}
    R_n(f) = \frac{1}{n}\sum_{i=1}^n |f(X_i) - Y_i|^2 + \lambda_n \|f\|_{H^s_{\mathrm{per}}([-2L, 2L]^d)}^2 + \mu_n \|\mathscr{D}(f)\|_{L^2(\Omega)}^2 
    \label{eq:risk_function}
\end{equation} 
over the class $\mathscr{F} = H^s_{\mathrm{per}}([-2L, 2L]^d)$ of candidate functions, where $\lambda_n > 0$ and $\mu_n \geqslant 0$ are hyperparameters that weigh the relative importance of each term. We refer to the appendix for a precise definition of the periodic Sobolev space $H^s_{\mathrm{per}}([-2L, 2L]^d)$, as well as the continuous extension $H^s(\Omega) \hookrightarrow H^s_{\mathrm{per}}([-2L, 2L]^d)$. 
It is stressed that the $\|\cdot\|_{H^s_{\mathrm{per}}([-2L, 2L]^d)}$ norm is the standard $\|\cdot\|_{H^s([-2L, 2L]^d)}$ norm---the symbol ``$\mathrm{per}$'' highlights that we consider functions belonging to a \emph{periodic} Sobolev space.
The choice of the periodic Sobolev space $H^s_{\mathrm{per}}([-2L,2L]^d)$ is merely technical---the reader can be confident that all subsequent results remain applicable to the standard Sobolev space $H^s(\Omega)$, as will be stressed later.

The first term in \eqref{eq:risk_function} is the standard component of supervised learning, corresponding to a least-squares criterion that measures the prediction error over the training sample. The second term $\|f\|_{H^s_{\mathrm{per}}([-2L, 2L]^d)}^2$ corresponds to a Sobolev penalty for $s>d/2$, which enforces the regularity of the estimator. 
Finally, the $L^2$ penalty $\|\mathscr{D}(f)\|_{L^2(\Omega)}^2$ on $\Omega$ quantifies the physical inconsistency of $f$ with respect to the differential prior on $f^\star$: the more $f$ aligns with the PDE, the lower the value of $\|\mathscr{D}(f)\|_{L^2(\Omega)}^2$. It is this last term that marks the originality of the hybrid modeling problem.

In this context, beyond classical statistical analyses, an interesting question is to quantify the impact of the physical regularization $\|\mathscr{D}(f)\|_{L^2(\Omega)}^2$ on the empirical risk \eqref{eq:risk_function}, typically in terms of convergence rate of the resulting estimator. It is intuitively clear, for example, that if the target $f^{\star}$ satisfies $\mathscr{D}(f^\star) =0$ (i.e., $f^{\star}$ is a solution of the underlying PDE), then, under appropriate conditions, the estimator $\hat f_n$ should have better properties than a standard estimator of the empirical risk. This is the challenging problem that we address in this contribution.

\paragraph{Contributions.} 
We are interested in the statistical properties of the minimizer of \eqref{eq:risk_function} over the space $H^s_{\mathrm{per}}([-2L,2L]^d)$, denoted by
\begin{equation}
    \hat f_n = \mathop{\mathrm{argmin}}_{f \in H^s_{\mathrm{per}}([-2L, 2L]^d)}\; R_n(f).
    \label{eq:estimator_sob}
\end{equation}
We show in Section \ref{sec:kernel} that problem \eqref{eq:estimator_sob} can be formulated as a kernel regression task, with a kernel~$K$ that we specify. This allows us, in Section \ref{sec:bounds}, to use tools from kernel theory to determine an upper bound on the rate of convergence of $\hat f_n$ to $f^\star$ in $L^2(\Omega, \mathbb{P}_X)$, where $\mathbb{P}_X$ is the distribution of~$X$ on $\Omega$. In particular, this rate can be evaluated by bounding the eigenvalues of the integral operator associated with the kernel. 
The latter problem is studied in detail in Theorem \ref{eq:weak_pde}, where the corresponding eigenfunctions are characterized through a weak formulation. Overall, we show that~$\hat f_n$ converges to $f^\star$ \emph{at least} at the Sobolev minimax rate. The complete mechanics are illustrated in Section \ref{sec:experiment} for the  operator $\mathscr{D} = \frac{d}{dx}$ in dimension $d=1$, showcasing a simple but instructive case.
In such a setting, the convergence rate is shown to be 
\begin{align*}
        \mathbb{E}\int_{[-L,L]} |\hat f_n-f^\star|^2 d{\mathbb P}_X = &\|\mathscr{D}(f^\star)\|_{L^2(\Omega)}\;\mathcal{O}_n \big(  n^{-2/3}\log^3(n)\big) +\|f^\star\|_{H^s(\Omega)}^2\mathcal{O}_n \big( n^{-1} \log^3(n)\big) .
    \end{align*}
Thus, the lower the modeling error $\|\mathscr{D}(f^\star)\|_{L^2(\Omega)}$, the lower the estimation error. In particular, if $f^{\star}$ exactly satisfies the PDE, i.e., $\|\mathscr{D}(f^\star)\|_{L^2(\Omega)} = 0$, then the rate is $n^{-1}$ (up a to log factor), significantly better than the Sobolev rate of $n^{-2/3}$. This shows that the use of physical knowledge in the PIML framework has a quantifiable impact on the estimation error.

\section{Related works}  

\paragraph{Approximation classes and Sobolev spaces.} Since Sobolev spaces are often considered too expensive for practical implementation, various alternative classes of functions over which to minimize the empirical risk function \eqref{eq:risk_function} have been suggested in the literature. 
In the case of a second-order and coercive PDE in dimension $d=2$, and with  an additional prior on the boundary conditions, \citet{azzimonti2015blood}, \citet{arnone2022spatialRegression}, and \citet{ferraccioli2022some} propose finite-element-based methods to optimize the minimization over $H^2(\Omega)$. However, the most commonly used approach to minimize the risk functional involves neural networks, which leverage the backpropagation algorithm for efficient computation of successive derivatives and optimize \eqref{eq:risk_function} through gradient descent. The so-called PINNs (for physics-informed neural networks---\citealp{raissi2019PINN}) have been successfully applied to a diverse range of physical phenomena, including  sea temperature modeling \citep{bezenac2017processes},
image denoising \citep{wang2020superResolution},
turbulence \citep{wang2020turbulence},
blood streams \citep{arzani2021uncovering},
glacier dynamics \citep{riel2021glacier},
and heat transfers \citep{ramezankhani2022multifidelity}, among others.
The neural architecture of PINNs is often designed to be large \citep[e.g.,][]{arzani2021uncovering, krishnapriyan2021characterizing, xu2021extrapolation}, allowing it to approximate any function in $H^s(\Omega)$ \citep{ryck2021approximation, doumèche2023convergence}.

\paragraph{Sobolev regularization.} In the PIML literature, the Sobolev regularization is either directly implemented as such \citep{shin2020convergence, doumèche2023convergence} or in a more implicit manner, by assuming that the operator $\mathscr{D}$ is inherently regular (e.g., second-order elliptic, parabolic, or hyperbolic) and specifying boundary conditions \citep{azzimonti2015blood, shin2020convergence, arnone2022spatialRegression, ferraccioli2022some, wu2022convergence, mishra2022generalization, shin2023error}.  
It turns out, however, that the specific form taken by the Sobolev regularization is unimportant. This will be enlightened by our Theorem \ref{thm:eq_reg}, which shows that using equivalent Sobolev norms does not alter the convergence rate of the estimators. 
From a theoretical perspective, much of the literature delves into the properties of PINNs in the realm of PDE solvers, usually through the analysis of their generalization error \citep{shin2020convergence, ryck2022kolmogorov, wu2022convergence, doumèche2023convergence, mishra2022generalization, qian2023error, deryck2023operator, shin2023error}.
Overall, there are few theoretical guarantees available regarding hybrid modeling, with the exception of \citet{azzimonti2015blood}, \citet{shin2020convergence}, \citet{arnone2022spatialRegression}, and \citet{doumèche2023convergence}. 

\paragraph{PIML and kernels.} Other studies have revealed interesting connections between PIML and kernel methods. In noiseless scenarios, the use of kernel methods to construct meshless PDE solvers under the Sobolev regularity hypothesis has long been explored by, for example, \citet{schaback2006from}. Recently, \citet{batlle2023error} uncovered convergence rates under regularity assumptions on the differential operator equivalent to a Sobolev regularization. For inverse PIML problems, \citet{lu2022sobolev} and \citet{de2023convergence} take advantage of a kernel reformulation of PIML to establish convergence rates for differential operator learning. This generalizes results obtained by \citet{nickl2019convergence} using Bayesian inference methods. However, none of these works has specifically addressed hybrid modeling. To the best of our knowledge, the present study is the first to show that the  physical regularization term $\|\mathscr{D}(f)\|_{L^2(\Omega)}$ in the PIML loss \eqref{eq:risk_function} may lead to improved convergence rates.

\section{PIML as a kernel method}
\label{sec:kernel}
Throughout the article, we let $\Omega \subseteq [-L, L]^d$ ($L >0$) be a bounded Lipschitz domain. Assuming that $\Omega$ is Lipschitz allows a high level of generality regarding its regularity, encompassing $C^1$-manifolds (such as the Euclidean ball $\{ x \in \mathbb{R}^d \; | \; \|x\|_2 \leqslant L\}$), as well as domains with non-differentiable boundaries (such as the hypercube $[-L, L]^d$). (A summary of the mathematical notation and functional analysis concepts used in this paper is to be found in Appendix \ref{sec:appA}.) 
The target function $f^\star: \mathbb R^d\to \mathbb R$ is assumed to belong to the Sobolev space $H^s(\Omega)$ for some positive integer $s>d/2$. Furthermore, this function is assumed to approximately satisfy a linear PDE on $\Omega$ (the coefficients of which are potentially non-constant) with derivatives of order less than or equal to $s$. In other words, one has $\mathscr{D}(f^\star) \simeq 0$ for some known operator $\mathscr D$ of the following form:
\begin{defi}[Linear differential operator]
\label{def:linearoperator}
    Let $s \in \mathbb{N}$. An operator $\mathscr{D}: H^s(\Omega) \to L^2(\Omega)$ is a linear differential operator if, for all $f \in H^s(\Omega)$,
    \[\mathscr{D}(f) = \sum_{|\alpha|\leqslant s} p_\alpha \partial^\alpha f,\]
    where  $p_\alpha: \Omega \to \mathbb{R}$ are \emph{functions} such that $\max_\alpha \|p_\alpha\|_\infty < \infty$. (By definition, $\{|\alpha|\leqslant s\}=\{\alpha\in \mathbb{N}^d \;|\; \|\alpha\|_1 \leqslant s\}$ and $\|\cdot\|_{\infty}$ stands for the supremum norm of functions.) 
\end{defi}
Given $s$, the linear differential operator $\mathscr D$, and a training sample $\{(X_1, Y_1), \hdots, (X_n, Y_n)\}$, we consider the estimator $\hat f_n$ that minimizes the regularized empirical risk  \eqref{eq:risk_function} over the periodic Sobolev space $H^s_{\mathrm{per}}([-2L, 2L]^d)$. Recall that $H^s_{\mathrm{per}}([-2L, 2L]^d)$ is the subspace of $H^s([-2L, 2L]^d)$ consisting of functions whose $4L$-periodic extension is still $s$-times weakly differentiable. 
The important point to keep in mind is that any function of $ H^s(\Omega)$ can be extended to a function in $H^s_{\mathrm{per}}([-2L, 2L]^d)$ (see Proposition \ref{prop:dec_four_lip} in the appendix), which makes it equivalent to suppose that $f^\star \in H^s(\Omega)$ or $f^\star \in H^s_{\mathrm{per}}([-2L,2L]^d)$, as shown in Theorem \ref{thm:eq_reg}. The extension mechanism is illustrated in Figure \ref{fig:omega}.  
\begin{figure}
    \centering
    \includegraphics[width=0.8\linewidth]{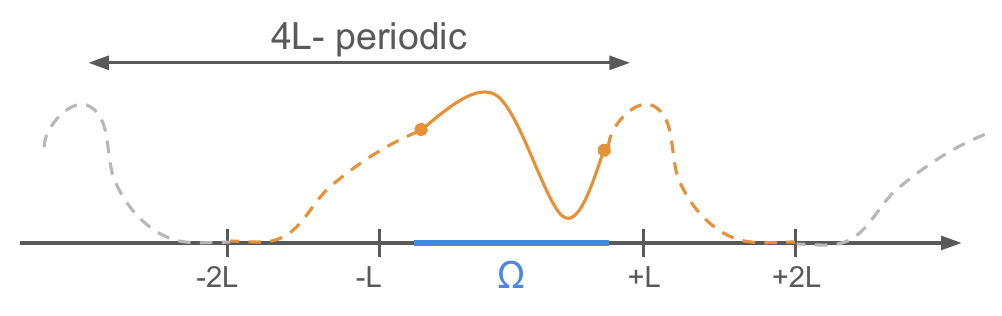}
    \caption{Illustration of a 4L-periodic extension of a function in $H^s(\Omega)$ to $H^s_{\mathrm{per}}([-2L,2L]^d)$ for $d=1$.}
    \label{fig:omega}
\end{figure}

The key step to turn the minimization of \eqref{eq:risk_function} into a kernel method is to observe that any function $f\in H^s_{\mathrm{per}}([-2L,2L]^d)$ can be linearly mapped in $L^2([-2L,2L]^d)$ in such a way that the norm $\|\cdot\|_{L^2([-2L,2L]^d)}$ of the embedding is equal to $\lambda_n \|f\|_{H^s_{\mathrm{per}}([-2L, 2L]^d)}^2 + \mu_n \|\mathscr{D}(f)\|_{L^2(\Omega)}^2$, i.e., the regularization term of \eqref{eq:risk_function}. Proposition \ref{prop:diff_op_main} below shows that this embedding takes the form of the inverse square root of a positive diagonalizable operator $\mathscr{O}_n$.
\begin{prop}[Differential operator]
    There exists a positive operator $\mathscr{O}_n$ on $L^2([-2L,2L]^d)$ such that $\mathscr{O}_n^{-1/2}: H^s_{\mathrm{per}}([-2L, 2L]^d) \to L^2([-2L,2L]^d)$ is well-defined and satisfies, for any $f \in H^s_{\mathrm{per}}([-2L,2L]^d)$, \[\|\mathscr{O}_n^{-1/2} (f)\|_{L^2([-2L,2L]^d)}^2 =  \lambda_n \|f\|_{H^s_{\mathrm{per}}([-2L, 2L]^d)}^2 + \mu_n \|\mathscr{D}(f)\|_{L^2(\Omega)}^2 .\]
    Moreover, there is an orthonormal basis of eigenfunctions $v_m \in H^s_{\mathrm{per}}([-2L,2L]^d)$ of $\mathscr{O}_n$ associated with eigenvalues $a_m > 0$ such that, for any $f \in L^2([-2L,2L]^d)$,
    \[\forall x \in [-2L,2L]^d, \quad \mathscr{O}_n(f)(x) = \sum_{m\in \mathbb{N}} a_m \langle f,  v_m\rangle_{L^2([-2L,2L]^d)} v_m(x).\]
    \label{prop:diff_op_main}
\end{prop}
Denote by $\delta_x$ the Dirac distribution at $x$. Informally, the properties of the embedding $\mathscr{O}_n^{-1/2}: H^s_{\mathrm{per}}([-2L, 2L]^d)\to L^2([-2L,2L]^d)$ in Proposition \ref{prop:diff_op_main} suggest that something like
\[f(x) \; ``="\; \langle f, \delta_x\rangle = \langle \mathscr{O}_n^{-1/2}(f), \mathscr{O}_n^{1/2}(\delta_x)\rangle_{L^2([-2L,2L]^d)}\] 
should be true. In other terms, still informally, we may write $f(x) = \langle z, \psi(x)\rangle_{L^2([-2L,2L]^d)}$, 
with $z = \mathscr{O}_n^{-1/2}(f)$, $\psi(x) = \mathscr{O}_n^{1/2}(\delta_x)$, and $\|z\|_{L^2([-2L,2L]^d)}^2 =   \lambda_n \|f\|_{H^s_{\mathrm{per}}([-2L, 2L]^d)}^2 + \mu_n \|\mathscr{D}(f)\|_{L^2(\Omega)}^2$. 
We recognize a reproducing property, turning $\psi$ into a kernel embedding associated with the risk~\eqref{eq:risk_function}. 
This mechanism is formalized in the following theorem. 
\begin{theorem}[Kernel of linear PDEs]
    Assume that $s>d/2$, and let $\lambda_n >0, \mu_n \geqslant 0$. Let $a_m$ and $v_m$ be the eigenvalues and eigenfunctions of $\mathscr{O}_n$.
    Then the space $H^s_{\mathrm{per}}([-2L,2L]^d)$, equipped with the inner product $\langle f, g\rangle_{\mathrm{RKHS}} = \langle \mathscr{O}_n^{-1/2}f,  \mathscr{O}_n^{-1/2}g\rangle_{L^2([-2L,2L]^d)}$, is a reproducing kernel Hilbert space. 
    In particular, 
    \begin{enumerate}
        \item[$(i)$] The kernel $K: [-2L, 2L]^d \times [-2L, 2L]^d \to \mathbb{R}$ is defined by 
        \[K(x,y) = \sum_{m\in \mathbb{N}} a_m  v_m(x)v_m(y). \]
        \item[$(ii)$] For all $x \in [-2L, 2L]^d$, $K(x, \cdot) \in H^s_{\mathrm{per}}([-2L,2L]^d)$.
        \item[$(iii)$] For all $f \in H^s_{\mathrm{per}}([-2L, 2L]^d)$, \[\forall x \in [-2L,2L]^d, \quad f(x) = \langle f, K(x, \cdot)\rangle_{\mathrm{RKHS}}.\]
        \item[$(iv)$] For all $f\in H^s_{\mathrm{per}}([-2L, 2L]^d)$, 
        \[\|f\|^2_{\mathrm{RKHS}} = \lambda_n \|f\|_{H^s_{\mathrm{per}}([-2L, 2L]^d)}^2 + \mu_n \|\mathscr{D}(f)\|_{L^2(\Omega)}^2.\]
    \end{enumerate}
    \label{thm:PDE_kernel}
\end{theorem}
\begin{proof}[{\bf sketch}]
    The complete proof is given in Appendix \ref{app:diif_of}.
    Only a rough sketch is given here by 
    examining the simplified case where  $L=\pi/2$, $\Omega = [-\pi,\pi]^d =[-2L,2L]^d$, 
    and $\mathscr{D}$ has constant coefficients. 
    This means that we consider functions with periodic derivatives on $\Omega$, penalized by the PDE on the whole domain $[-\pi,\pi]^d$. It turns out that, in this case, the corresponding operator $\mathscr{O}_n$, satisfying
    \[\| \mathscr{O}_n^{-1/2} (f)\|_{L^2([-\pi,\pi]^d)}^2 =  \lambda_n \|f\|_{H^s_{\mathrm{per}}([-\pi, \pi]^d)}^2 + \mu_n \|\mathscr{D}(f)\|_{L^2([-\pi,\pi]^d)}^2 := \|f\|^2_{\mathrm{RKHS}}\]
    has an explicit form. To see this, denote by $\mathrm{FS}$ the Fourier series operator. By the Parseval's theorem, for any frequency $k \in \mathbb{Z}^d$, one has $\mathrm{FS}( {\mathscr{O}}_n^{-1/2} (f))(k) = \sqrt{a_k} \; \mathrm{FS}(f)(k)$,
    where 
    \[a_k = \lambda_n \sum_{|\alpha|\leqslant s}   \prod_{j=1}^dk_j^{2\alpha_j} + \mu_n \Big(\sum_{|\alpha|\leqslant s}  p_\alpha \prod_{j=1}^dk_j^{\alpha_j}\Big)^2.\]
    Accordingly, $ {\mathscr{O}}_n$ is diagonalizable with eigenfunctions $v_k: x \mapsto \exp(i\langle k, x\rangle)$ associated with the eigenvalues $a_k^{-1}$. Next, using the Fourier decomposition of $f$, we have, for all $x \in [-\pi,\pi]^d$, 
    \[f(x) = \sum_{k\in \mathbb{Z}^d} \mathrm{FS}(f)(k) \exp(i \langle k, x\rangle) = \sum_{k\in \mathbb{Z}^d} \mathrm{FS}( {\mathscr{O}}_n^{-1/2} (f))(k) a_k^{-1/2}\exp(i \langle k, x\rangle).\]   
    Since $a_k^{-1} \leqslant \lambda_n^{-1} (\sum_{|\alpha|\leqslant s}   \prod_{j=1}^dk_j^{2\alpha_j})^{-1}$, it is easy to check that $\sum_{k \in \mathbb{Z}^d} a_k^{-1} < \infty$ and that the function $\psi_x$ such that $\mathrm{FS}(\psi_x)(k) = a_k^{-1/2} \exp(i \langle k, x\rangle)$ belongs to $H^s_{\mathrm{per}}([-\pi, \pi]^d)$. We therefore have the kernel formulation $f(x) = \langle  {\mathscr{O}}_n^{-1/2} (f), \psi_x\rangle_{L^2([-\pi,\pi]^d)}$, where $\| {\mathscr{O}}_n^{-1/2} (f)\|_{L^2([-\pi,\pi]^d)}^2 = \|f\|^2_{\mathrm{RKHS}}$. The corresponding kernel is then defined by \[K(x,y) = \langle \psi_x, \psi_y\rangle_{L^2([-\pi,\pi]^d)} = \sum_{k \in \mathbb{Z}} a_k v_k(x)\bar v_k(y).\] 
    The complete proof of Theorem \ref{thm:PDE_kernel} is more technical because, in the our case $\Omega \subsetneq [-2L,2L]^d$ and $\mathscr{D}$ may have non-constant coefficients. Thus, the operator $\mathscr{O}_n$ is not diagonal in the Fourier space. To characterize its eigenvalues $a_m$ and eigenfunctions $v_m$, we resort to classical results of PDE theory building upon functional analysis.
\end{proof}
The message of Theorem \ref{thm:PDE_kernel} is that minimizing the empirical risk \eqref{eq:risk_function} can be cast as a kernel method associated with the regularization $\lambda_n \|f\|_{H^s_{\mathrm{per}}([-2L, 2L]^d)}^2 + \mu_n \|\mathscr{D}(f)\|_{L^2(\Omega)}^2$. In other words, \eqref{eq:risk_function} can be rewritten as
    \[\hat f_n = \mathop{\mathrm{argmin}}_{f \in H^s_{\mathrm{per}}([-2L, 2L]^d)}\; \frac{1}{n}\sum_{i=1}^n |f(X_i) - Y_i|^2 + \|f\|^2_{\mathrm{RKHS}}.\]
This result is interesting in itself because it fundamentally shows that a PIML estimator (and therefore its variants implemented in practice, such as PINNs) can be regarded as a kernel estimator. 
Note however that computing $K(x,y)$ is not always straightforward and may require the use of numerical techniques. This kernel is characterized by the following weak formulation.
\begin{prop}[Kernel characterization]
    The kernel $K$ is the unique solution to the following weak formulation, valid for all test functions $\phi \in H^s_{\mathrm{per}}([-2L,2L]^d)$, \[\forall x \in \Omega,\quad \lambda_n \sum_{|\alpha|\leqslant s} \int_{[-2L, 2L]^d}\partial^\alpha K(x, \cdot)\; \partial^\alpha \phi + \mu_n \int_{\Omega}\mathscr{D}(K(x, \cdot)) \; \mathscr{D}(\phi) = \phi(x).\]
    \label{prop:kernel_characterization}
\end{prop}
Regardless of the analytical computation of $K$, formulating the problem as a minimization in a reproducing kernel Hilbert space provides a way to  
quantify the impact of the physical regularization on the estimator's convergence rate, which is our primary goal. 

\section{Convergence rates}
\label{sec:bounds}
   The results of the previous section allow us to draw on the existing literature on kernel learning to gain a deeper understanding of the properties of the estimator $\hat f_n$ and the influence of the operator~$\mathscr{D}$ on the convergence rate. 
   
\subsection{Eigenvalues of the integral operator}
The convergence rate of $\hat f_n$ to $f^\star$ is determined by the decay speed of the eigenvalues of the so-called integral operator $L_K: L^2(\Omega, \mathbb{P}_X) \to L^2(\Omega, \mathbb{P}_X),$ defined by
\[\forall f \in L^2(\Omega, \mathbb{P}_X), \forall x \in \Omega, \quad L_Kf(x) = \int_{\Omega} K(x,y) f(y) d\mathbb{P}_X(y),\]
where $\mathbb{P}_X$ is the distribution of $X$ on $\Omega$ \citep[e.g.,][]{caponnetto2007optimal}. Note that the integral in the definition of $L_K$ could also have been taken over $[-2L,2L]^d$ because the support of $\mathbb{P}_X$ is included in $\bar \Omega$.
However, finding the eigenvalues of $L_K$ is not an easy task---even when $X$ is uniformly distributed on $\Omega$---, not to mention the fact that $\mathbb{P}_X$ is usually unknown in real applications.
Nevertheless, we show in Theorem \ref{thm:eigenvalues} that these eigenvalues can be bounded by the eigenvalues of the operator $C\mathscr{O}_nC$, where $C$ is the projection on $\Omega$ defined below. Importantly, $C\mathscr{O}_nC$ no longer depends on $\mathbb{P}_X$. Moreover, its non-zero eigenvalues are characterized by a 
weak formulation, as we will see in Theorem \ref{prop:eigenfunction}. 
\begin{defi}[Projection on $\Omega$]
    Let $C$ be the operator on $L^2([-2L, 2L]^d)$ defined by $C f = f \mathbf{1}_\Omega$. 
    Then $C^2 = C$, i.e., $C$ is a projector, and 
    \[\langle f, C(g)\rangle_{L^2([-2L, 2L]^d)} = \int_{[-2L, 2L]^d} f g \mathbf{1}_\Omega =  \langle C(f), g\rangle_{L^2([-2L, 2L]^d)},\]
    i.e., $C$ is self-adjoint.
\end{defi} 
As for now, it is assumed that the distribution $\mathbb{P}_X$ of $X$ has a density $\frac{d\mathbb{P}_X}{dx}$ with respect to the Lebesgue measure on $\Omega$.
\begin{theorem}[Kernels and eigenvalues]
\label{thm:eigenvalues}
     Let $K: [-2L, 2L]^d \times [-2L, 2L]^d \to \mathbb{R}$ be the kernel of Theorem \ref{thm:PDE_kernel}. Assume that there exists $\kappa>0$ such that $\frac{d\mathbb{P}_X}{dx} \leqslant \kappa$. Then the eigenvalues $a_m(L_K)$ of $L_K$ are bounded by the eigenvalues $a_m(C\mathscr{O}_nC)$ of $ C\mathscr{O}_nC$ on $L^2([-2L, 2L]^d)$ in such a way that $a_m(L_K) \leqslant \kappa a_m(C\mathscr{O}_nC)$.
\end{theorem} 

\subsection{Effective dimension and convergence rate} 
We will see in the next subsection how to compute the eigenvalues of $C\mathscr{O}_nC$. Yet, assuming we have them at hand, it is then possible to obtain a bound on the rate of convergence of $\hat f_n$ to $f^\star$ by bounding the so-called effective dimension of the kernel \citep{caponnetto2007optimal}, defined by
\[
\mathscr{N}(\lambda_n, \mu_n) = \mathrm{tr}(L_{K}  (\mathrm{Id} + L_{K})^{-1}),
\]
where $\mathrm{Id}$ is the identity operator, i.e., $\mathrm{Id}(f) = f$, and the symbol $\mathrm{tr}$ stands for the trace, i.e., the sum of the eigenvalues. Lemma \ref{lem:eff_dim} in the appendix shows that, whenever $\frac{d\mathbb{P}_X}{dx} \leqslant \kappa$, 
\begin{equation}
        \mathscr{N}(\lambda_n, \mu_n) 
    \leqslant \sum_{m \in \mathbb{N}} \frac{ 1}{1+(\kappa a_m(C\mathscr{O}_nC))^{-1}}.
    \label{eq:eff_dim_vp}
\end{equation}
Putting all the pieces together, we have the following theorem, which bounds the estimation error between $\hat f_n$ and $f^\star$.
\begin{theorem}[Convergence rate]
    Assume that $s>d/2$, $f^\star \in H^s(\Omega)$, $\frac{d\mathbb{P}_X}{dx} \leqslant \kappa$ for some $\kappa > 0$, $\mu_n \geqslant 0$, $\lim_{n\to \infty}\lambda_n = \lim_{n\to \infty} \mu_n = \lim_{n\to \infty} \lambda_n / \mu_n = 0$, $\lambda_n \geqslant n^{-1}$, and  $\mathscr{N}(\lambda_n, \mu_n) \lambda_n^{-1} = o_n (n)$.
    Assume, in addition, that, for some $\sigma > 0$ and $M > 0$, the noise $\varepsilon$ satisfies 
    \begin{equation}
        \forall \ell \in \mathbb{N}, \quad \mathbb{E}(|\varepsilon|^\ell\; | \; X) \leqslant \frac{1}{2}\ell !\; \sigma^2\; M^{\ell-2}.
        \label{eq:conditionNoise}
    \end{equation}
    Then, for some constant $C_4 >0$ and $n$ large enough, 
    \begin{align*}
        &\mathbb{E}\int_\Omega |\hat f_n-f^\star|^2 d{\mathbb P}_X \\
        &\quad \leqslant C_4 \log^2(n)\Big(\lambda_n \|f^\star\|_{H^s(\Omega)}^2 + \mu_n \|\mathscr{D}(f^\star)\|_{L^2(\Omega)}^2 + \frac{M^2}{n^2 \lambda_n} + \frac{\sigma^2\mathscr{N}(\lambda_n, \mu_n)}{n}\Big).
    \end{align*}
\label{thm:boundexp}
\end{theorem}
The sub-Gamma assumption \eqref{eq:conditionNoise} on the noise $\varepsilon$ is quite general and is satisfied in particular when $\varepsilon$ is bounded (possibly depending on $X$), or when $\varepsilon$ is Gaussian and independent of $X$ \citep[e.g.,][Theorem 2.10]{boucheron2013concentration}. We stress that the result of Theorem \ref{thm:boundexp} is general and holds regardless of the form of the linear differential operator $\mathscr{D}$. 
A simple bound on $\mathscr{N}(\lambda_n, \mu_n)$, neglecting the dependence in $\mathscr{D}$, allows to show that the PIML estimator converges at least at the Sobolev minimax rate over the class $H^s(\Omega)$.
\begin{prop}[Minimum rate]
    Suppose that the assumptions of Theorem \ref{eq:conditionNoise} are verified, and let
    $\lambda_n = n^{-2s/(2s+d)} \sqrt{\log(n)}$ and $\mu_n = \lambda_n \sqrt{\log(n)}$.
    Then the estimator $\hat f_n$ converges at a rate at least larger than the Sobolev minimax rate, up to a $\log$ term, i.e., 
    \begin{align*}
        \mathbb{E}\int_\Omega |\hat f_n-f^\star|^2 d{\mathbb P}_X = \mathcal{O}_n \big(  n^{-2s/(2s+d)}\log^3(n)\big).
    \end{align*}
    \label{prop:minSpeed}
\end{prop}
 However, this is only an upper bound, and we expect situations where $\hat f_n$ has a faster convergence rate thanks to the inclusion of the physical penalty $\| \mathscr{D}(f)\|_{L^2(\Omega)}$. 
 Such an improvement will depend on the magnitude of the modeling error $\|\mathscr{D}(f^{\star})\|_{L^2(\Omega)}$ and on the effective dimension $\mathscr{N}(\lambda_n,\mu_n)$. To achieve this goal, the eigenvalues $a_m$ of $C\mathscr{O}_nC$ must be characterized and then plugged into inequality \eqref{eq:eff_dim_vp}.
 This is the problem addressed in the next subsection. 

\subsection{Characterizing the eigenvalues}
\label{sec:effectiveDim}
The goal of this section is to specify the spectrum of $C\mathscr{O}_n C$. 
It is worth noting that $\ker(C\mathscr{O}_n C)$ is not empty, as it encompasses every smooth function with compact support in $]\! -\! 2L, 2L[^d \backslash \bar\Omega$. 
The next theorem characterizes the eigenfunctions associated with non-zero eigenvalues and shows that they are in fact smooth functions on $\Omega$ and $(\bar \Omega)^c$ satisfying two PDEs.
\begin{theorem}[Eigenfunction characterization]
    Assume that $s>d/2$ and that the functions $p_\alpha$ in Definition \ref{def:linearoperator} belong to $C^\infty(\Omega)$. Let $a_m >0$ be a positive eigenvalue of the  operator $C\mathscr{O}_n C$. Then the corresponding eigenfunction $v_m$ satisfies $v_m = a_m^{-1}C w_m$, where $w_m \in H^{s}_{\mathrm{per}}([-2L, 2L]^d)$. Moreover, for any test function $\phi \in H^s_{\mathrm{per}}([-2L,2L]^d)$,
    \begin{equation}
        \lambda_n \sum_{|\alpha|\leqslant s} \int_{[-2L, 2L]^d}\partial^\alpha w_m\; \partial^\alpha \phi + \mu_n \int_{\Omega}\mathscr{D}(w_m) \; \mathscr{D}(\phi) 
    = a_m^{-1} \int_{\Omega} w_m \phi.
    \label{eq:weak_pde}
    \end{equation} 
    In particular, 
     any solution of the weak formulation \eqref{eq:weak_pde} satisfies the following PDE system:
    \begin{itemize}
        \item[$(i)$] $w_m \in C^\infty(\Omega)$ and 
        \[\forall x \in \Omega, \quad \lambda_n\sum_{|\alpha|\leqslant s}(-1)^{|\alpha|} \partial^{2\alpha} w_m(x) + \mu_n \mathscr{D}^\ast \mathscr{D} w_m(x) = a_m^{-1} w_m(x),\]
        where $\mathscr{D}^\ast(f) := \sum_{|\alpha|\leqslant s} (-1)^{|\alpha|}\partial^\alpha (p_\alpha  f)$ is the adjoint operator of $\mathscr{D}$.
        \item[$(ii)$]  $w_m \in C^\infty([-2L, 2L]^d \backslash \bar\Omega)$ and 
        \[\forall x \in [-2L, 2L]^d \backslash \bar \Omega,\quad  \sum_{|\alpha|\leqslant s}(-1)^{|\alpha|} \partial^{2\alpha} w_m(x)  = 0.\]
    \end{itemize}
    Notice that $w_m$ might be irregular on the boundary $\partial \Omega$, but only there.
    \label{prop:eigenfunction}
\end{theorem}
Theorem \ref{prop:eigenfunction} is important insofar as it allows to characterize the positive eigenvalues $a_m$ of the operator $C \mathscr{O}_n C$. Indeed, these eigenvalues are the only real numbers such that the weak formulation  \eqref{eq:weak_pde} admits a solution.
This weak formulation has to be solved in a case-by-case study, given the differential operator $\mathscr{D}$. As an illustration, an example is presented in the next section with $\mathscr{D}=\frac{d}{dx}$. 

\subsection{The choice of Sobolev regularization is unimportant}
So far, we have considered problem \eqref{eq:risk_function} with the Sobolev regularization $\| f\|_{H^s_{\mathrm{per}}([-2L,2L]^d)}^2$. However, other choices of Sobolev norms, such as $\| f\|_{H^s(\Omega)}^2$, are also possible. Fortunately, this choice does not affect the effective dimension $\mathscr{N}(\lambda_n, \mu_n)$, and thus the convergence rate in Theorem \ref{thm:boundexp}. 
\begin{theorem}[Equivalent regularities and effective dimension]
    Assume that $s > d/2$. Then the following three estimators correspond each to a kernel learning problem:
    \begin{align*}
        \hat f_n^{(1)} &= \mathop{\mathrm{argmin}}_{f \in H^s(\Omega)} \sum_{i=1}^n |f(X_i) -Y_i|^2  + \lambda_n \| f\|_{H^s(\Omega)}^2 + \mu_n \| \mathscr{D}(f)\|_{L^2(\Omega)}^2,\\
   \hat f_n^{(2)} &= \mathop{\mathrm{argmin}}_{f \in H^s_{\mathrm{per}}([-2L,2L]^d)} \sum_{i=1}^n |f(X_i) -Y_i|^2  + \lambda_n  \| f\|_{H^s_{\mathrm{per}}([-2L,2L]^d)}^2 + \mu_n \| \mathscr{D}(f)\|_{L^2(\Omega)}^2,\\
   \hat f_n^{(3)} &= \mathop{\mathrm{argmin}}_{f \in H^s(\Omega)} \sum_{i=1}^n |f(X_i) -Y_i|^2 + \lambda_n  \| f\|^2 + \mu_n \| \mathscr{D}(f)\|_{L^2(\Omega)}^2,
    \end{align*}
    where $\| \cdot \|$ is any of the equivalent Sobolev norms.
    \label{thm:eq_reg}
     Moreover, these three estimators share equivalent effective dimensions ${\mathscr{N}(\lambda_n, \mu_n)}$. Accordingly, they share the same upper bound on the convergence rate given by Theorem \ref{thm:boundexp}.
\end{theorem}
The incorporation of a Sobolev regularization in the empirical risk function is needed to guarantee that $\hat f_n$ has good statistical properties. For example, even with the simplest PDEs, the minimizer of $\sum_{i=1}^n |f(X_i) -Y_i|^2 + \mu_n \| \mathscr{D}(f)\|_{L^2(\Omega)}^2$ might always be $0$, independently of the data points $(X_i, Y_i)$ \citep[see, e.g.,][Example 5.1]{doumèche2023convergence}.
A way to overcome these statistical issues is to specify the boundary conditions, and to consider regular differential operators $\mathscr{D}$ and smooth domain $\Omega$. For example, \citet{azzimonti2015blood},  \citet{arnone2022spatialRegression}, and \citet{ferraccioli2022some} consider models such that $f^\star|_{\partial \Omega} = 0$, where $\Omega$ is an Euclidean ball of $\mathbb{R}^d$ and $\mathscr{D}$ are second-order elliptic operators. However, these assumptions amount to adding a Sobolev penalty, since, in this case, $ \| \mathscr{D}(f)\|_{L^2(\Omega)} $ and $\|f\|_{H^2_0(\Omega)}$ are equivalent norms \citep[e.g.,][Chapter 6.3, Theorem~4]{evans2010partial}. 
Similar results hold for second order parabolic PDEs \citep[][Chapter 7.1, Theorem 5]{evans2010partial} and for second order hyperbolic PDEs \citep[][Chapter 7.2, Theorem 2]{evans2010partial}.
The need for a Sobolev regularization is explained by the fact that the  Sobolev embedding $H^s(\Omega) \hookrightarrow C^0(\Omega)$ only holds for $s>d/2$. 
In other words, the Sobolev regularization is needed to give a sense to the pointwise evaluations $|f(X_i) -Y_i|$. 
\label{rem:eq_reg}

\section{Application: speed-up effect of the physical penalty}
\label{sec:experiment}
Our objective is to apply the framework presented above to the case $d=1$, $\Omega = [-L, L]$, 
$s = 1$, $f^\star \in H^1(\Omega)$, and $\mathscr{D} = \frac{d}{dx}$.
Of course, assuming that $\mathscr{D}(f^\star) \simeq 0$ is a strong assumption, equivalent to assuming that $f^\star$ is approximately constant. However, the goal of this section is to provide a simple illustration where the kernel $K$ of Theorem \ref{thm:PDE_kernel} can be  analytically computed and
the eigenvalues of the operator $L_K$ can be effectively bounded.
The next result is a consequence of Proposition \ref{prop:kernel_characterization}.
\begin{prop}[One-dimensional kernel]
    Assume that $s = 1$, $\Omega = [-L, L]$, and $\mathscr{D} = \frac{d}{dx}$. Then, letting $\gamma_n = \sqrt{\frac{\lambda_n}{\lambda_n + \mu_n}}$, one has, for all $x,y \in [-L,L]$, 
    \begin{align*}
        K(x,y)& = \frac{\gamma_n}{2\lambda_n \sinh(2 \gamma_n L)}\Big( (\cosh(2\gamma_n L)+\cosh(2\gamma_n x))\cosh(\gamma_n (x-y))\\
        &\qquad \qquad + ((1-2 \times \mathbf{1}_{x > y})\sinh(2\gamma_n L) - \sinh(2 \gamma_n x)) \sinh(\gamma_n(x-y))\Big).
    \end{align*} 
    \label{prop:1dkernel}
\end{prop}
An example of kernel $K$ with $L=1$ and $\lambda_n=\mu_n=1$ is shown in Figure \ref{fig:2Dkernel}. Following the strategy of Section \ref{sec:bounds}, it remains to bound the positive eigenvalues $a_m$ of the operator $C\mathscr{O}_nC$ using Theorem~\ref{prop:eigenfunction}. According to the latter, this is achieved by solving the weak formulation
\begin{equation*}
        \forall \phi \in H^1_{\mathrm{per}}([-2L,2L]), \quad \lambda_n \int_{[-2L, 2L]^d}w_m\phi + (\lambda_n +\mu_n) \int_{\Omega}\frac{d}{dx}w_m\;\frac{d}{dx}\phi
    = a_m^{-1} \int_{\Omega} w_m \phi.
    \end{equation*} 
\begin{prop}[One-dimensional eigenvalues]
    Assume that $s = 1$, $\Omega = [-L, L]$, and $\mathscr{D} = \frac{d}{dx}$. Then, for all $m \geqslant 3$, 
    \[ \frac{4L^2}{(\lambda_n + \mu_n)(m+4)^2 \pi^2} \leqslant a_m \leqslant \frac{4L^2}{(\lambda_n + \mu_n)(m-2)^2 \pi^2},\]
    where $a_m$ are the eigenvalues of $C\mathscr{O}_nC$.
    \label{prop:1ddiff}
\end{prop}
Using inequality \eqref{eq:eff_dim_vp}, we can then bound the effective dimension of the kernel. This allows us, via Theorem \ref{thm:boundexp}, to specify the convergence rate of $\hat f_n$ to $f^{\star}$.
\begin{theorem}[Kernel speed-up] 
   Assume that $f^\star \in H^1([-L,L])$, $\frac{d\mathbb{P}_X}{dx} \leqslant \kappa$ for some $\kappa > 0$, and the noise $\varepsilon$ satisfies the sub-Gamma condition \eqref{eq:conditionNoise}.
   Let $\lambda_n = n^{-1} \log(n)$ and 
    \begin{equation*}
        \mu_n = \left\{ 
    \begin{array}{llll}
        n^{-2/3}/ \|\mathscr{D}(f^\star)\|_{L^2(\Omega)} && \mathrm{ if}  &\|\mathscr{D}(f^\star)\|_{L^2(\Omega)} \neq 0 \\
        1/\log(n) && \mathrm{ if}  &\|\mathscr{D}(f^\star)\|_{L^2(\Omega)} = 0.  \\
    \end{array}\right.
    \end{equation*} 
    Then the estimator $\hat f_n$ of $f^\star$ minimizing the empirical risk function \eqref{eq:risk_function} with $s=1$ and $\mathscr{D} = \frac{d}{dx}$ satisfies
    \begin{align*}
        \mathbb{E}\int_{[-L,L]} |\hat f_n-f^\star|^2 d{\mathbb P}_X &= \|\mathscr{D}(f^\star)\|_{L^2(\Omega)}\;\mathcal{O}_n \big(  n^{-2/3}\log^3(n)\big) \\
        & \quad+ (\|f^\star\|_{H^1(\Omega)}^2 + \sigma^2 + M^2)\mathcal{O}_n \big( n^{-1} \log^3(n)\big) .
    \end{align*}
    \label{prop:kernel_speed_up}
\end{theorem}
\begin{wrapfigure}[16]{r}{0.35\textwidth}
    \vspace{-1.5em}
    \centering
    \includegraphics[width=0.35\textwidth]{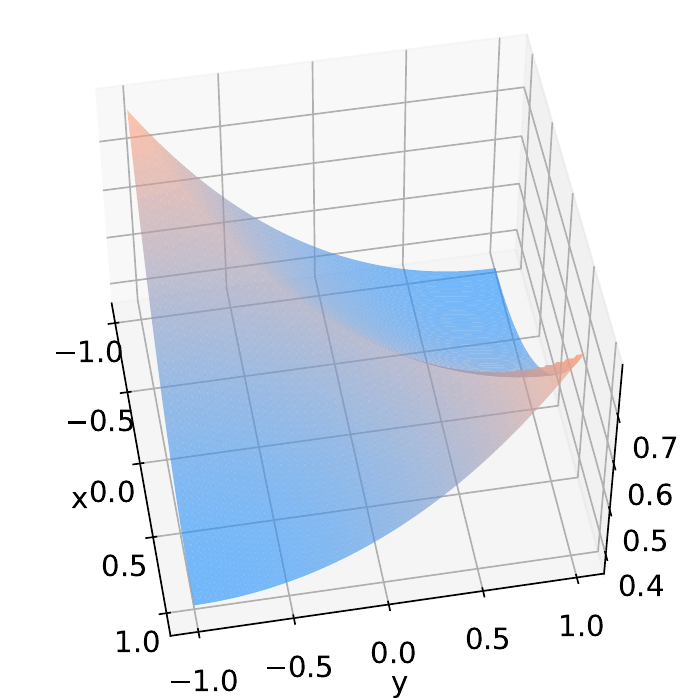}
    \caption{Kernel $K$ of Proposition \ref{prop:1dkernel} with $L=1$, $\lambda_n = \mu_n =1$.}
     \label{fig:2Dkernel}
 \end{wrapfigure}
This bound reflects the benefit of the physical penalty $\|\mathscr{D}(f^\star)\|_{L^2(\Omega)}$ on the performance of the estimator $\hat f_n$. 
Indeed, when $\|\mathscr{D}(f^\star)\|_{L^2(\Omega)} = 0$ (i.e., the physical model is perfect), then $f^\star$ is a constant function, and the PIML method recovers the parametric convergence rate of $n^{-1}$. Here, the physical information directly improves the convergence rate. 
Otherwise, when $\|\mathscr{D}(f^\star)\|_{L^2(\Omega)}>0$, we recover the Sobolev minimax convergence rate in $H^1(\Omega)$ of $n^{-2/3}$ \citep[up to a log factor---see][Theorem 2.11]{tsybakov2009introduction}. 
We emphasize that this rate is also optimal for our problem, since $\|\mathscr{D}(f^\star)\|_{L^2(\Omega)} \leqslant \|f^\star\|_{H^1(\Omega)}$, i.e., it is as hard to learn a function of bounded $\|\mathscr{D}(\cdot)\|_{L^2(\Omega)}$ norm as it is to learn a function of bounded $H^1(\Omega)$ norm. 
In this case, the benefit of physical modeling is carried by the constant $\|\mathscr{D}(f^\star)\|_{L^2(\Omega)}$ in front of the convergence rate, i.e., the better the modeling, the smaller the estimation error. Note however that the parameter $\mu_n$ in Theorem \ref{prop:kernel_speed_up} depends on the unknown physical inconsistency $\|\mathscr{D}(f^\star)\|_{L^2(\Omega)}$. In practice, on may resort to a cross-validation-type strategy to estimate $\mu_n$.

We conclude this section with a small numerical experiment\footnote{The code to reproduce all numerical experiments can be found \href{https://github.com/NathanDoumeche/PIML_as_a_kernel_method}{here}.} illustrating Theorem \ref{prop:kernel_speed_up}. We consider two  problems: a perfect modeling situation  where $Y = 1 + \varepsilon$, and an imperfect modeling one where $Y = 1 + 0.1 |X| + \varepsilon$. In both cases, $X \sim \mathscr{U}([-1,1])$ and $\varepsilon \sim \mathcal{N}(0, 1)$. The difference is that in the perfect modeling case, $\mathscr{D}(f^\star) = 0$, whereas in the imperfect situation $\|\mathscr{D}(f^\star)\|_{L^2([-1,1])}^2 = 2/300$. For each $n$, we let $\mathrm{err}(n) = \mathbb{E}\int_\Omega |\hat f_n-f^\star|^2 d{\mathbb P}_X$.
Figure~\ref{fig:numerical_experiment} shows the values of 
$\log(\mathrm{err})(n)$ as a function of $\log(n)$, for $n$ ranging from $10$ to $10000$ (the quantity $\log(\mathrm{err})(n)$ is estimated by an empirical mean over 500-sample Monte Carlo estimations, repeated ten times). 
The experimental convergence rates obtained by fitting linear regressions are $-1.02$ in the perfect modeling case and $-0.77$ in the imperfect one. These experimental rates are consistent with the results of Theorem~\ref{prop:kernel_speed_up}, insofar as
$-1.02 \leqslant -1$ and $-0.77 \leqslant -2/3$. 
\begin{figure}
    \centering
    \includegraphics[width=0.49\textwidth]{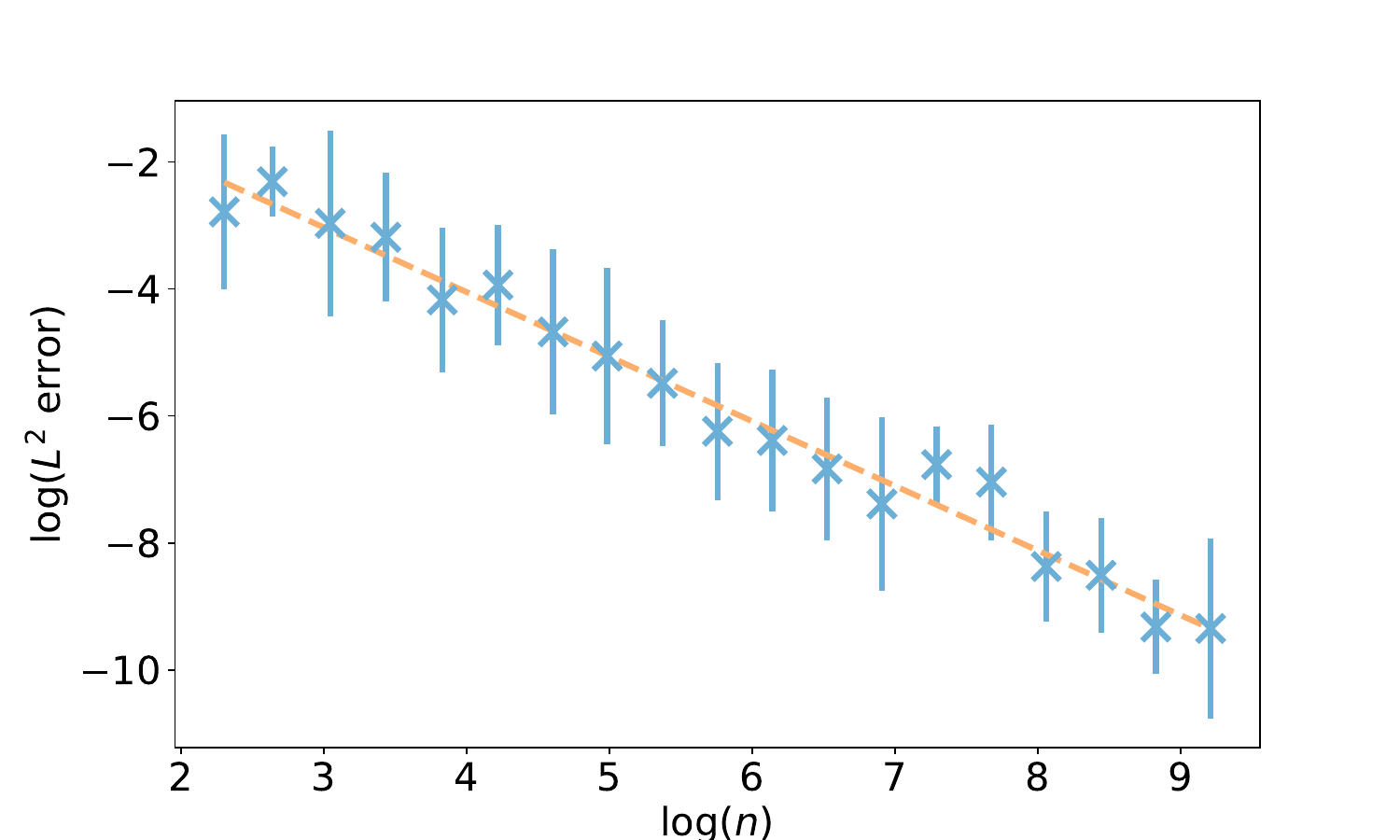}
    \includegraphics[width=0.49\textwidth]{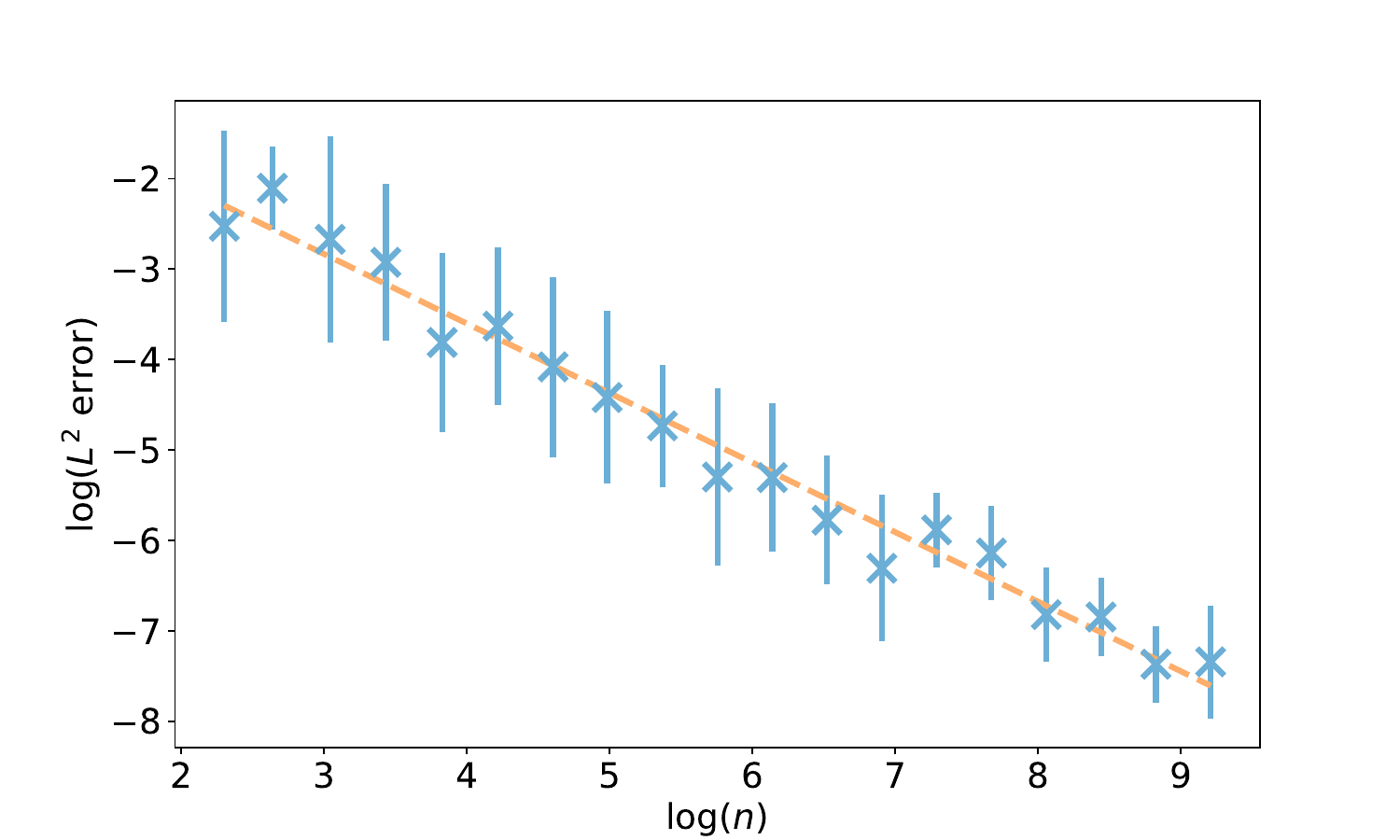}
    \caption{
    Error bounds $\mathrm{err}(n)$ (mean $\pm$ std over 10 runs) of the kernel estimator $\hat{f}_n$ with respect to the sample size $n$, in log-log scale, for the perfect modeling case (left) and the imperfect one (right). 
    The experimental convergence rates, obtained by fitting a linear regression, are displayed in orange dotted. 
    }
    \label{fig:numerical_experiment}
\end{figure}
\section{Conclusion}
From the physics-informed machine learning point of view, we have shown that minimizing the empirical risk regularized by a PDE can be viewed as a kernel method. Leveraging kernel theory, we have explained how to derive convergence rates. In particular, the simple but instructive example $\mathscr{D}=\frac{d}{dx}$ illustrates how to compute both the kernel and the convergence rate of the associated estimator. To the best of our knowledge, this is the first contribution that demonstrates tangible improvements in convergence rates by including a physical penalty in the risk function. 
Thus, the take-home message is that physical information can be beneficial to the statistical performance of the estimators. Note that our work does not include boundary conditions $h$, but they could easily be considered. A first solution is to add another penalty to $R_n$ of the form $\|h-f\|^2_{L^2(\partial \Omega)}$, which would insert the extra term $\int_{\partial \Omega} (K(x,\cdot)-h)\phi$ in Proposition~\ref{prop:kernel_characterization}. 
A second solution is to enforce the conditions at new data points $X^{(b)}_j$ sampled on $\partial \Omega$ \citep[as done, for example, in][]{raissi2019PINN}. Our theorems hold for this extended training sample, provided that $f^\star_{|\partial \Omega}=h$.

An important future research direction is to implement numerical strategies for computing the kernel $K$ in the general case. If successful, such strategies can then be used directly to solve general physics-informed machine learning problems. In order to derive theoretical guarantees, we need to go further by obtaining bounds on the eigenvalues of the operator associated with the problem. The key lies in Theorem \ref{prop:eigenfunction}, which characterizes the eigenvalues by a weak formulation. Once established, such bounds can be employed to obtain accurate rates for related techniques, typically physics-informed neural networks. It would also be interesting to derive rates of convergence in the setting $s\leqslant d/2$ using the so-called source condition \citep[e.g.,][]{blanchard2020kernel}. An even more ambitious goal is to generalize the approach to nonlinear differential systems, for example polynomial. Overall, we believe that our results pave the way for a deeper understanding of the impact of physical regularization on empirical risk minimization performance.
\bibliography{biblio}

\newpage
\noindent
    Supplement to \textit{Physics-informed machine learning as a kernel method}

\tableofcontents
\appendix
\section{Some fundamentals of functional analysis}
\label{sec:appA}
\subsection{Sobolev spaces}
\paragraph{Norms.} The $p$ norm $\|x\|_p$ of a $d$-dimensional vector $x = (x_1,\hdots, x_{d})$ is defined by $ \|x\|_p = (\frac{1}{d}\sum_{i=1}^{d} |x_i|^p)^{1/p}$.
For a function $f : \Omega \rightarrow \mathbb{R}$, we let $ \|f\|_{L^p(\Omega)} = (\frac{1}{|\Omega|}\int_\Omega |f|^p)^{1/p}$. Similarly, $\|f\|_{\infty, \Omega} = \sup_{x \in \Omega} |f(x)|$. For the sake of conciseness, we sometimes write $\|f\|_{\infty}$ instead of~$\|f\|_{\infty, \Omega}$.

\paragraph{Multi-indices and partial derivatives.} For a multi-index $\alpha = (\alpha_1, \hdots, \alpha_{d}) \in \mathbb{N}^{d}$ and a differentiable function $f:\mathbb R^{d}\to \mathbb R$, the $\alpha$ partial derivative of $f$ is defined by \[\partial^\alpha f = (\partial_{1})^{\alpha_1}\hdots (\partial_{d})^{\alpha_{d}} f.\] The set of multi-indices of sum less than $k$ is defined by \[\{|\alpha|\leqslant k\} = \{(\alpha_1, \hdots, \alpha_{d_1}) \in \mathbb{N}^{d}, \alpha_1 + \cdots +\alpha_{d_1} \leqslant k\}.\] If $\alpha = 0$, $\partial^\alpha f = f$. Given two multi-indices $\alpha$ and $\beta$, we write $\alpha \leqslant \beta$ when $\alpha_i \leqslant \beta_i$ for all $1\leqslant i \leqslant d$. 
The set of multi-indices less than $\alpha$ is denoted by $\{\beta \leqslant \alpha\}$. For a multi-index $\alpha$ such that $|\alpha|\leqslant k$, both sets $\{|\beta|\leqslant k\}$ and $\{\beta \leqslant \alpha\}$ are contained in $\{0, \hdots, k\}^{d}$ and are therefore finite.

\paragraph{Hölder norm.} For $K \in \mathbb{N}$, the Hölder norm of order $K$ of a function $f \in C^K(\Omega, \mathbb{R})$ is defined by $\|f\|_{C^K(\Omega)} = \max_{|\alpha|\leqslant K} \|\partial^\alpha f\|_{\infty, \Omega}$. 
This norm allows to bound a function as well as its derivatives. 
The space $C^K(\Omega, \mathbb{R})$ endowed with the Hölder norm $\|\cdot\|_{C^K(\Omega)}$ is a Banach space. 
The space $C^\infty(\bar{\Omega}, \mathbb{R}^{d_2})$ is defined as the subspace of continuous functions $f:\bar{\Omega} \to \mathbb{R}$ satisfying $f|_\Omega \in C^\infty(\Omega, \mathbb{R})$ and, for all $K\in \mathbb{N}$, $\|f\|_{C^K(\Omega)} < \infty$.

\paragraph{Lipschitz function.} Given a normed space $(V, \|\cdot\|)$, the Lipschitz norm of a function $f : V \rightarrow \mathbb{R}^{d}$ is defined by 
\[\|f\|_{\text{Lip}} = \sup_{x,y \in V}\frac{\|f(x)-f(y)\|_2}{\|x-y\|}.\] A function $f$ is Lipschitz if $\|f\|_{\mathrm{Lip}}<\infty$. The mean value theorem implies that for all $f \in C^1(V, \mathbb{R})$, $\|f\|_{\text{Lip}} \leqslant \|f\|_{C^1(V)}$.

\paragraph{Lipschitz surface  and domain.} A surface $\Gamma \subseteq \mathbb{R}^{d}$ is said to be Lipschitz if locally, in a neighborhood $U(x)$ of any point $x \in \Gamma$, an appropriate rotation $r_x$ of the coordinate system transforms $\Gamma$ into the graph of a Lipschitz function $\phi_{x}$, i.e., 
\[r_x(\Gamma \cap U(x)) = \{(x_1, \hdots, x_{d - 1}, \phi_x(x_1, \hdots, x_{d - 1})), \forall (x_1, \hdots, x_d)\in r_x(\Gamma \cap U_x)\}.\]
A domain $\Omega \subseteq \mathbb{R}^{d}$ is said to be Lipschitz if its has Lipschitz boundary and lies on one side of it, i.e., $\phi_x < 0$ or $\phi_x > 0$ on all intersections $\Omega \cap U_x$. All manifolds with $C^1$ boundary and all convex domains are Lipschitz domains \citep[e.g.,][]{agranovich2015lispchitz}.

\paragraph{Sobolev spaces.} Let $\Omega \subseteq \mathbb{R}^{d}$ be an open set. A function $g \in L^2(\Omega, \mathbb{R})$ is said to be the $\alpha$th weak derivative of  $f \in L^2(\Omega, \mathbb{R})$ if, for all $\phi \in C^\infty(\bar{\Omega}, \mathbb{R})$ with compact support in $\Omega$, one has
$\int_\Omega g \phi = (-1)^{|\alpha|} \int_\Omega f \partial^\alpha \phi$. 
This is denoted by $g = \partial^\alpha f$. For $s \in\mathbb{N}$, the Sobolev space $H^s(\Omega)$ is the space of all functions $f \in L^2(\Omega, \mathbb{R})$ such that $\partial^\alpha f$ exists for all $|\alpha|\leqslant s$. This space is naturally endowed with the norm \[\|f\|_{H^s(\Omega)} = \Big(\sum_{|\alpha|\leqslant s} \|\partial^\alpha u\|_{L^2(\Omega)} ^2\Big)^{1/2}.\] Of course, if a function $f$ belongs to the Hölder space $C^K(\bar \Omega, \mathbb{R})$, then it belongs to the Sobolev space $H^K(\Omega)$, and its weak derivatives are the usual derivatives. For more on Sobolev spaces, we refer the reader to \citet[Chapter 5]{evans2010partial}.

\paragraph{Fundamental results on Sobolev spaces.} Let $\Omega \subseteq \mathbb{R}^{d}$ be an open set and let $s \in \mathbb{N}$ be an order of differentiation. It is not straightforward to extend a function $f \in H^s(\Omega)$ to a function $\tilde{f} \in H^s(\mathbb{R}^{d})$ such that 
\[\tilde{f}|_\Omega = f|_\Omega \quad \text{and} \quad \|\tilde{f}\|_{H^s(\mathbb{R}^{d})} \leqslant C_\Omega \|f\|_{H^s(\Omega)},\]
for some constant $C_\Omega$ independent of $f$. This result is known as the extension theorem in \citet[][Chapter 5.4]{evans2010partial} when $\Omega$ is a manifold with $C^1$ boundary. 
However, the simplest domains in PDEs take the form $]0,L[^3\times ]0,T[$, the boundary of which is not $C^1$. Fortunately, \citet[][Theorem 5, Chapter VI.3.3]{stein1970lipschitz} provides an extension theorem for bounded Lipschitz domains. The following two theorems are proved in  \citet{doumèche2023convergence}.
\begin{theorem}[Sobolev inequalities]
    \label{thm:sobIneq}
    Let $\Omega \subseteq \mathbb{R}^{d}$ be a bounded Lipschitz domain and let $s \in \mathbb{N}$. If $s  > d_1/2$, then there is an operator $\tilde \Pi : H^{s}(\Omega) \to C^0(\Omega, \mathbb{R})$ such that, for all $f \in H^{s}(\Omega)$, $\tilde \Pi(f) = f$ almost everywhere. Moreover, there is a constant $C_\Omega >0$, depending only on $\Omega$, such that $\|\tilde \Pi(f)\|_{\infty, \Omega} \leqslant C_\Omega \|f\|_{H^{s}(\Omega)}.$
\end{theorem}

\begin{theorem}[Rellich-Kondrachov]
    \label{thm:rellichK}
    Let $\Omega \subseteq \mathbb{R}^{d}$ be a bounded Lipschitz domain and let $s \in \mathbb{N}$. Let $(f_p)_{p\in \mathbb{N}}\in H^{s+1}(\Omega)$ be a sequence such that $(\|f_p\|_{H^{s+1}(\Omega)})_{p\in \mathbb{N}}$ is bounded.
    There exists a function $f_\infty \in H^{s+1}(\Omega)$ and a subsequence of $(f_p)_{p\in \mathbb{N}}$ that converges to $f_\infty$ with respect to the $H^s(\Omega)$ norm.
\end{theorem}

\subsection{Fourier series on complex periodic Sobolev spaces}
Let $L>0$.
\begin{defi}[Periodic extension operator]
    Let $d \in \mathbb{N}^\star$. The periodic extension operator $E_{\mathrm{per}}: L^2([-2L, 2L]^d) \to L^2([-4L, 4L]^d)$ is defined, for all function $f: [-2L, 2L[^d \to \mathbb{R}$ and all $x=(x_1, \hdots, x_d) \in [-4L, 4L]^d$, by
    \[E_{\mathrm{per}}(f)(x) = f\Big(x_1 - 4L\Big\lfloor \frac{x_1}{4L}\Big\rfloor, \hdots, x_d - 4L\Big\lfloor \frac{x_d}{4L}\Big\rfloor\Big).\] 
\end{defi}

\begin{defi}[Periodic Sobolev spaces]
    Let $s \in \mathbb{N}$. The space of functions $f$ such that $E_{\mathrm{per}}(f) \in H^s([-4L, 4L]^d)$ is denoted by $H_{\mathrm{per}}^s([-2L, 2L]^d)$.
\end{defi}
If $s>0$, then $H_{\mathrm{per}}^s([-2L, 2L]^d)$ is a strict linear subspace of $H^s([-2L, 2L]^d)$. For example, for all $s\geqslant 1$, the function $f(x) = x_1^2 + \cdots + x_d^2$ belongs to $H^s([-2L, 2L]^d)$, but $f \notin H_{\mathrm{per}}^s([-2L, 2L]^d)$. Indeed, though $E_{\mathrm{per}}(f)$ is continuous, it is not weakly differentiable. The following characterization of periodic Sobolev spaces in terms of Fourier series are well-known \citep[see, e.g.,][Chapter 2.1]{temam1995navier}.

\begin{prop}[Fourier decomposition on periodic Sobolev spaces]
    Let $s \in \mathbb{N}$ and $d \geqslant 1$. For all function $f \in H_{\mathrm{per}}^s([-2L, 2L]^d)$, there exists a unique vector $z \in \mathbb{C}^{\mathbb{Z}^d}$ such that
$f(x) = \sum_{k \in \mathbb{Z}^d} z_k \exp(i\frac{\pi}{2L}\langle k, x\rangle)$,
and
\[\forall |\alpha|\leqslant s, \quad \partial^\alpha f(x) = \Big(i\frac{\pi}{2L}\Big)^{|\alpha|}\sum_{k \in \mathbb{Z}^d} z_k \exp(i\frac{\pi}{2L}\langle k, x\rangle) \prod_{j=1}^d k_j^{\alpha_j}.\]
Moreover, for all multi-index $|\alpha| \leqslant s$,
$\|\partial^\alpha f\|_{L^2([-2L, 2L]^d))}^2 = (\frac{\pi}{2L})^{2|\alpha|}\sum_{k \in \mathbb{Z}^d}  |z_k|^2 \prod_{j=1}^d k_j^{2 \alpha_j}$. Therefore,
$\|f\|_{H^s([-2L, 2L]^d)}^2 = \sum_{k \in \mathbb{Z}^d}  |z_k|^2  
\sum_{|\alpha|\leqslant s} (\frac{\pi}{2L})^{2|\alpha|} \prod_{j=1}^d k_j^{2\alpha_j}
$.
\label{prop:dec_four}


\end{prop}

\begin{proof}  The uniqueness of the decomposition is a consequence of
\[z_k = \frac{1}{4^dL^d}\int_{[-2L, 2L]^d} f(x) \exp(-i\frac{\pi}{2L}\langle k, x\rangle) dx.\] 
To prove the existence of such a decomposition, consider $f \in H_{\mathrm{per}}^s([-2L, 2L]^d)$. Since $f \in L^2([-2L, 2L]^d)$ and its derivative with respect to the first variable $\partial_1 f \in L^2([-2L, 2L]^d)$, $f$ and $\partial_1 f$ can be decomposed into the following multidimensional Fourier series \citep[see, e.g.,][Chapter 5.4]{brezis2010functional}:
\[\forall x \in [-2L, 2L]^d, \quad f(x) = \sum_{k \in \mathbb{Z}^d} z_k \exp(i\frac{\pi}{2L}\langle k, x\rangle),\] 
\[\forall x \in [-2L, 2L]^d, \quad \partial_1 f(x) = \sum_{k\in \mathbb{Z}^d} \tilde z_k \exp(i\frac{\pi}{2L}\langle k,x\rangle).\] 
Observe that $E_{\mathrm{per}}(f)$ has the same Fourier decomposition as $f$ and that $E_{\mathrm{per}}(\partial_1 f)$ has the same decomposition as $\partial_1 f$. The goal is to show that $\tilde z_k = i \frac{\pi}{2L} k_1 z_k$.
By definition of the weak derivative $\partial_1 E_{\mathrm{per}}(f)$, for any test function $\phi \in C^\infty([-4L, 4L]^d)$ with compact support in $[-4L, 4L]^d$, one has 
\[\int_{[-4L, 4L]^d}  \phi\partial_1 E_{\mathrm{per}}(f)   = -\int_{[-4L, 4L]^d} E_{\mathrm{per}}(f)  \partial_1  \phi.\]
Let 
\[\psi(u) = \left\{ \begin{array}{cl}
     0 &\text{ if } -4L \leqslant u \leqslant -1-2L \\
      \frac{\int_{-1-2L}^u 
      \exp(\frac{-1}{(2L+1+v)^2}) \exp(\frac{-1}{(2L+v)^2})dv}{(\int_{-1-2L}^{-2L} \exp(\frac{-1}{(2L + 1+v)^2}) \exp(\frac{-1}{(2L+v)^2}dv)^{-1} },& \text{ if } -1-2L \leqslant u \leqslant -2L,\\
      1 & \text{ if } -2L \leqslant u \leqslant 2L-1,\\
       1 - \psi(u-4L)& \text{ if } 2L-1 \leqslant u \leqslant 2L, \\
       0 &\text{ if } 2L+1 \leqslant u \leqslant 4L.
\end{array}\right.\]
One easily verifies that $\psi \in C^\infty([-4L, 4L])$ and that it has a compact support in $[-4L, 4L]$. Moreover, $\|\psi\|_\infty =1$. Notice that, for all function $g \in L^2([-2L, 2L])$ and any $4L$-periodic function $\phi \in C^\infty([-4L, 4L])$ whose support is not necessary compact, 
\begin{equation}
    \int_{[-4L, 4L]} g \phi \psi = \int_{[-2L, 2L]} g \phi \quad \mbox{and} \quad \int_{[-4L, 4L]} g (\phi \psi)' = \int_{[-2L, 2L]} g \phi'.
    \label{eq:unityPartition}
\end{equation} 
To generalize such a property in dimension $d$, we let $\psi_d(x) = \prod_{j=1}^d \psi(x_j)$. Then, for all $k \in \mathbb{Z}^d$, $\phi_{k,d}(x) := \psi_d(x) \exp(-i\frac{\pi}{2L}\langle k,x\rangle)$ is a smooth function with compact support. Thus, by definition of the weak derivative,
\begin{align*}
    \int_{[-4L, 4L]^d}  \phi_{k,d} \partial_1 E_{\mathrm{per}}(f)   = -\int_{[-4L, 4L]^d} E_{\mathrm{per}}(f)  \partial_1 \phi_{k,d}.
\end{align*}
Moreover, using the left-hand side of \eqref{eq:unityPartition}, we have that
\[\int_{[-4L, 4L]^d}  \phi_{k,d} \partial_1 E_{\mathrm{per}}(f) = \int_{[-2L, 2L]^d} \exp(-i\frac{\pi}{2L}\langle k,x\rangle) \partial_1 E_{\mathrm{per}}(f)(x) dx = (4L)^d  \tilde z_k,\] 
while, using the right-hand side of \eqref{eq:unityPartition}, we have that 
\[\int_{[-4L, 4L]^d} E_{\mathrm{per}}(f)  \partial_1 \phi_{k,d} =  \frac{-i\pi}{2L} k_1\int_{[-2L, 2L]^d} E_{\mathrm{per}}(f)(x)  \exp(\frac{-i\pi}{2L} \langle k,x\rangle)dx = (4L)^d \frac{-i\pi}{2L} k_1 z_k.\]
Therefore, $\tilde z_k = i \frac{\pi}{2L} k_1 z_k$.

The exact same reasoning holds for $\partial_j f$, for all $1 \leqslant j \leqslant d$. By iterating on the successive derivatives, we obtain that for all $|\alpha|\leqslant s$,
$ \partial^\alpha f(x) = (i\frac{\pi}{2L})^{|\alpha|}\sum_{k \in \mathbb{Z}^d} z_k \exp(i \frac{\pi}{2L} \langle k, x\rangle) \prod_{j=1}^d k_j^{\alpha_j}$, as desired. The last two equations of the proposition are direct consequences of Parseval's theorem.
\end{proof} 
This proposition states that there is a one-to-one mapping between $H_{\mathrm{per}}^s([-2L, 2L]^d)$ and $\{z \in \mathbb{C}^{\mathbb{Z}^d}\; | \; \sum_k |z_k|^2 (1+\|k\|_2^2)^s  < \infty \;\mathrm{ and }\; \bar z_{k} = z_{-k}\}$. In particular, this shows that for $s >0$,  $H_{\mathrm{per}}^s([-2L, 2L]^d)$ is an Hilbert space for the norm $\|\cdot\|_{H^s([-2L, 2L]^d)}$. 

\subsection{Fourier series on Lipschitz domains}
As for now, it is assumed that $\Omega \subseteq [-L, L]^d$ is a bounded Lipschitz domain. The objective of this section is to parameterize the Sobolev space $H^s(\Omega)$ by the space $\mathbb{C}^{\mathbb{Z}^d}$ of Fourier coefficients.
\begin{prop}[Fourier decomposition of $H^s(\Omega)$]
    Let $s\in \mathbb{N}$. For any function $ f \in H^s(\Omega)$, there is a vector $z \in \mathbb{C}^{\mathbb{Z}^d}$ such that $ \sum_{k \in \mathbb{Z}^d}  |z_k|^2 \|k\|_2^{2s} < \infty$ and
\begin{equation*}
    \forall |\alpha|\leqslant s, \forall x \in \Omega, \quad \partial^\alpha f(x)  =  \Big(i\frac{\pi}{2L}\Big)^{|\alpha|}\sum_{k \in \mathbb{Z}^d} z_k \exp(i\frac{\pi}{2L}\langle k, x\rangle) \prod_{j=1}^d k_j^{\alpha_j}.
\end{equation*}
Thus, $f$ can be linearly extended to the function  $\tilde E(f)(x) =  \sum_{k \in \mathbb{Z}^d} z_k \exp(i\frac{\pi}{2L}\langle k, x\rangle)$ which belongs to $H^s_{\mathrm{per}}([-2L,2L]^d)$.
Moreover, there is a constant $C_{s, \Omega}$, depending only on the domain $\Omega$ and the order of differentiation $s$, such that, for all $f \in H^s(\Omega)$,
\[\|\tilde E(f) \|_{H^s_{\mathrm{per}}([-2L, 2L]^d)}^2 = \sum_{k \in \mathbb{Z}^d}  |z_k|^2 \sum_{|\alpha|\leqslant s} \Big(\frac{\pi}{2L}\Big)^{2|\alpha|} \prod_{j=1}^d k_j^{2\alpha_j} \leqslant \tilde C_{s,\Omega} \|f \|_{H^s(\Omega)}^2.\]
\label{prop:dec_four_lip}
\end{prop}
\begin{proof}
    Let $ f \in H^s(\Omega)$. According to the Sobolev extension theorem \citep[][Chapter 5.4]{evans2010partial}, there is an extension operator $E:  H^s(\Omega) \to  H^s([-2L, 2L]^d)$ and a constant $C_{s, \Omega}$, depending only $\Omega$ and $s$, such that, for all $f \in H^s(\Omega)$,  $E(f) \in H^s([-2L, 2L]^d)$ and $\|E(f) \|_{H^s([-2L, 2L]^d)} \leqslant C_{s,\Omega} \|f \|_{H^s(\Omega)}$.
    Choose $\phi \in C^\infty([-2L, 2L]^d, [0,1])$ with compact support, and such that $\phi =1$ on $\Omega$ and $\phi = 0$ on $[-2L, 2L]^d \backslash [-3L/2, 3L/2]^d$. Then the extension operator $\tilde E(f) = \phi \times E(f)$ is such that $\tilde E(f) \in H^s_{\mathrm{per}}([-2L, 2L]^d)$. In addition, the Leibniz formula on weak derivatives shows that there is a constant $\tilde C_{s,\Omega}$ such that $\|\tilde E(f) \|_{H^s([-2L, 2L]^d)}^2 \leqslant \tilde C_{s,\Omega} \|f \|_{H^s(\Omega)}^2$. The result is then a direct consequence of Proposition \ref{prop:dec_four} applied to $\tilde E(f)$.
\end{proof}
Classical theorems on series differentiation show that given any vector 
$z \in \mathbb{C}^{\mathbb{Z}^d}$ satisfying 
$$ \sum_{k \in \mathbb{Z}^d}  |z_k|^2 \|k\|_2^{2s} < \infty\quad \text{and} \quad\bar z = -z,$$ the associated Fourier series belongs to $H^s(\Omega)$. This shows that one can identify $H^s(\Omega)$ with $\{z \in \mathbb{C}^{\mathbb{Z}^d}\; |\; \sum_{k \in \mathbb{Z}^d}  |z_k|^2 \|k\|_2^{2s} < \infty \; \mathrm{ and }\; \bar z = -z\}$, and the inner product  $\langle f, g\rangle_{H^s([-2L,2L]^d))} = \sum_{|\alpha|\leqslant s} \int_{[-2L,2L]^d} \partial^\alpha f \partial^\alpha g$ with $\langle \tilde z,  z\rangle_{\mathbb{C}^{\mathbb{Z}^d}} = \sum_{k \in \mathbb{Z}^d} \tilde z_k \bar z_k \sum_{|\alpha|\leqslant s} (\frac{\pi}{2L})^{2|\alpha|} \prod_{j=1}^d k_j^{2\alpha_j}$. 

\begin{prop}[Countable reindexing of $H^s(\Omega)$]
    There is a one-to-one mapping $k : \mathbb N \to \mathbb Z ^d$ such that, letting $e_j = (x \mapsto \exp(i \frac{\pi}{2L}\langle k(j), x\rangle))$, any function $f \in H^s(\Omega)$ can be written as $\sum_{j \in \mathbb{N}} z_j e_j$, with $ z \in \mathbb{C}^{\mathbb{N}}$ and $ \sum_{j\in\mathbb{N}} |z_j |^2 j^{2s/d} < \infty$.
    \label{prop:countable}
\end{prop}

\begin{proof}
    Let $ f \in L^2(\Omega)$. By Proposition \ref{prop:dec_four_lip}, we know that $ f \in H^s(\Omega)$ if and only if there is a vector $z \in \mathbb{C}^{\mathbb{Z}^d}$ such that $ \sum_{k \in \mathbb{Z}^d}  |z_k|^2 \|k\|_2^{2s} < \infty$, and  $f(x) = \sum_{k \in \mathbb{Z}^d} z_k \exp(i\frac{\pi}{2L}\langle k, x\rangle)$.
Let  $j \in \mathbb{N} \mapsto k(j) \in \mathbb{Z}^d$ be a one-to-one mapping such that $\|k(j)\|_1$ is increasing. 
Then, for all $K > 0$, \[ \begin{pmatrix}
    K + (d+1) - 1 \\ (d+1) - 1
\end{pmatrix} \leqslant \mathrm{argmin} \{j \in \mathbb{N} \; |\; \|k(j)\|_1 \geqslant K\} \leqslant 2^d \begin{pmatrix}
    K + (d+1) - 1 \\ (d+1) - 1
\end{pmatrix}.\] 
Indeed, $\begin{pmatrix}
    K + (d+1) - 1 \\ (d+1) - 1
\end{pmatrix}$ corresponds to the number of vectors $(n_0, \hdots, n_d) \in \mathbb{N}^{d+1} $ such that $n_0 + \cdots + n_d = K$, where $n_\ell$ represents the order of differentiation along the dimension $\ell$ and where $n_0$ is a fictive dimension to take into account derivatives of order less than $s$).
Since $\begin{pmatrix}
    K + (d+1) - 1 \\ (d+1) - 1
\end{pmatrix} \underset{j\to\infty}{\sim} \frac{K^d}{d!}$, we deduce that there are constants $C_1, C_2 > 0$ such that $C_1 j^{1/d} \leqslant \|k(j)\|_1 \leqslant C_2 j^{1/d}$.
Observe that 
$\|k\|_2^{2s} \geqslant (\max_{j=1}^d k_j)^{2s} \geqslant \|k\|_1^{2s}/ d^{2s}$, and that
$\|k\|_2^{2s} \leqslant (d \max_{j=1}^d k_j^2)^s \leqslant d^s \|k\|_1^{2s}$.
We conclude that $ f \in H^s(\Omega)$ if and only if $f$ can be written as $\sum_{j \in \mathbb{N}} z_{k(j)} \exp(i\frac{\pi}{2L}\langle k(j), x\rangle)$, where $\sum_{j \in \mathbb{N}} |z_{k(j)}|_2^2  j^{2s/d} < \infty$. 
\end{proof}

\subsection{Operator theory}
An operator is a linear function between two Hilbert spaces, potentially of infinite dimensions. The objective of this section is to give conditions on the regularity of such an operator so that it behaves similarly to matrices in finite dimension spaces.  For more advanced material, the reader is referred to the textbooks by \citet[Chapter D.6]{evans2010partial} and \citet[Problem 37 (6)]{brezis2010functional}. 
\begin{defi}[Hermitian spaces and Hermitian basis]
     $(H, \langle \cdot, \cdot \rangle)$ is a Hermitian space when $H$ is a complex Hilbert space endowed with an Hermitian inner product $\langle \cdot, \cdot \rangle$. This Hermitian inner product is associated with the norm $\|u\|^2 = \langle u, u \rangle$, defining a topology on $H$. We say that $(v_n)_{n \in \mathbb{N}} \in H^{\mathbb{N}}$ is a Hermitian basis of $H$ if $\langle v_n, v_m\rangle = \delta_{n,m}$, and if for all $u \in H$, there exists a sequence $(z_n)_{n \in \mathbb{N}} \in  \mathbb{C}^{\mathbb{N}}$ such that $\lim_{n \to \infty}\|u - \sum_{j=1}^n z_j v_j\| =0$. $H$ is said to be separable if it admits an Hermitian basis.
\end{defi}
\begin{defi}[Self-adjoint operator]
    Let  $(H, \langle \cdot, \cdot \rangle)$ be a Hermitian space. Let $\mathscr{O}: H \to H$ be an operator. We say that $\mathscr{O}$ is self-adjoint if, for all $u,v \in H$, one has $\langle \mathscr{O} u, v\rangle = \langle  u, \mathscr{O} v\rangle$.
    
\end{defi}
\begin{defi}[Compact operator]
    Let  $(H, \langle \cdot, \cdot \rangle)$ be a Hermitian space.  Let $\mathscr{O}: H \to H$ be an operator. We say that $\mathscr{O}$ is compact if, for any bounded set $S \subseteq H$, the closure of $\mathscr{O}(S)$ is compact.
\end{defi}
\begin{theorem}[Spectral theorem]
\label{thm:spectral}
    Let $\mathscr{O}$ be a compact self-adjoint operator on a separable Hermitian space $(H, \langle \cdot, \cdot \rangle)$. Then $\mathscr{O}$ is diagonalizable in an orthonormal basis with real eigenvalues, i.e., there is an Hermitian basis $(v_m)_{m \in \mathbb{N}}$ and real numbers $(a_m)_{m \in \mathbb{N}}$ such that, for all $u \in H$, $\mathscr{O}(u) = \sum_{m\in \mathbb{N}} a_m \langle v_m, u \rangle v_m$. 
\end{theorem}
\begin{defi}[Positive operator]
    An operator $\mathscr{O}$ on a Hermitian space $(H, \langle \cdot, \cdot \rangle)$ is positive if, for all $u\in H$, $\langle u, \mathscr{O} u\rangle \geqslant 0$.
\end{defi}
\begin{theorem}[Courant-Fischer min-max theorem]
\label{thm:courant_fischer}
    Let $\mathscr{O}$ be a positive compact self-adjoint operator on a separable Hermitian space $(H, \langle \cdot, \cdot \rangle)$. Then the eigenvalues of $\mathscr{O}$ are positive and, when reindexing them in a non-increasing order,
    \[a_m(\mathscr{O}) = \underset{\dim \Sigma = m}{\underset{\Sigma \subseteq H}{\max}} \underset{u \neq 0}{\min_{u \in \Sigma}}\|u\|^{-2}\langle u, \mathscr{O} u\rangle. \]
\end{theorem}
\begin{defi}[Order on Hermitian operators]
    Let $\mathscr{O}_1$ and $\mathscr{O}_2$ be two positive compact self-adjoint operators on a separable Hermitian space $(H, \langle \cdot, \cdot \rangle)$. We say that $\mathscr{O}_1 \succeq \mathscr{O}_2$ if, for all $u \in H$, $\langle u, \mathscr{O}_1 u\rangle \geqslant\langle u, \mathscr{O}_2 u\rangle$. According to the Courant-Fischer min-max theorem, this implies that, for all $m\in \mathbb N$, $a_m(\mathscr{O}_1) \geqslant a_m(\mathscr{O}_2)$.
\end{defi}

\subsection{Symmetry and PDEs}
The goal of this section is to recall various techniques useful for the determination of the eigenfunctions of a differential operator.
\begin{defi}[Symmetric operator]
    An operator $\mathscr{O}$ on a Hilbert space $\mathscr{H} \subseteq L^2([-2L,2L]^d)$ is said to be symmetric if, for all functions $f \in \mathscr{H}$ and for all $x \in [-2L,2L]^d$, $\mathscr{O}(f) (-x)= \mathscr{O}(f(-\cdot))(x)$, where $f(-\cdot)$ is the function such that $f(-\cdot)(x) = f(-x)$.
\end{defi}
For example, the Laplacian $\Delta$ in dimension $d=2$ is a symmetric operator, since $\Delta f(-\cdot)(x) = (\partial^2_{1,1}f(-\cdot))(x) + (\partial^2_{2,2}f(-\cdot))(x) = \Delta f(-x)$. However, $\partial_1$ is not symmetric, since $\partial_1 f(-\cdot)(x) = -(\partial_{1}f)(-x)$. 
\begin{prop}[Eigenfunctions of symmetric operators]
    Let $\mathscr{O}$ be a symmetric operator on a Hil\-bert space $\mathscr{H}$. Then, if $v$ is an eigenfunction of $\mathscr{O}$, $v^{\texttt{sym}} = v + v(-\cdot) $ and $v^{\texttt{antisym}} = v - v(-\cdot)$ are two eigenfunctions of $\mathscr{O}$ with the same eigenvalue as $v$, and $\int_{[-2L,2L]^d} v^{\texttt{sym}} v^{\texttt{antisym}} = 0$. Notice that $v = ({v^{\texttt{sym}} + v^{\texttt{antisym}}})/{2}$.
    \label{prop:sym}
\end{prop}
\begin{proof}
    Let $v$ be an eigenfunction of $\mathscr{O}$ for the eigenvalue $a \in \mathbb{R}$, i.e., $\mathscr{O}(v) = av$. Since $\mathscr{O}$ is symmetric, $\mathscr{O}(v(-\cdot)) = \mathscr{O}(v)(-\cdot) = a v(-\cdot)$. Therefore, $v^{\texttt{sym}}$ and $v^{\texttt{antisym}} $ are two eigenfunctions of $\mathscr{O}$ with $a$ as eigenvalue. Since $v^{\texttt{sym}}$ is symmetric and $v^{\texttt{antisym}} $ is antisymmetric, $\int_{[-2L,2L]^d} v^{\texttt{sym}} v^{\texttt{antisym}} = 0$, and so they are orthogonal.
\end{proof}

\section{The kernel point of view of PIML}
\label{app:diif_of}
This appendix is devoted to providing the tools of functional analysis relevant to our problem. 

\subsection{Properties of the differential operator}
Let $\lambda_n > 0$ and $\mu_n \geqslant 0$. We study in this section some of the properties of the differential operator~$\mathscr{O}_n$ such  that, for all $f \in H^s_{\mathrm{per}}([-2L,2L]^d)$, $\|\mathscr{O}_n^{-1/2}(f)\|_{L^2([-2L, 2L]^d)}^2 = \lambda_n \|f\|_{H^s_{\mathrm{per}}([-2L, 2L]^d)}^2 + \mu_n \|\mathscr{D}(f)\|_{L^2(\Omega)}^2$.

\begin{prop}[Differential operator]
    There is an injective operator $\mathscr{O}_n: L^2([-2L, 2L]^d) \to H^s_{\mathrm{per}}([-2L, 2L]^d)$ defined as follows: for all $f \in L^2([-2L, 2L]^d)$, $\mathscr{O}_n(f)$ is the unique element of $H^s_{\mathrm{per}}([-2L, 2L]^d)$ such that, for any test function $\phi \in H^s_{\mathrm{per}}([-2L, 2L]^d)$,
        \[  \lambda_n \sum_{|\alpha|\leqslant s} \int_{[-2L, 2L]^d}\partial^\alpha \phi\; \partial^\alpha \mathscr{O}_n(f) + \mu_n \int_{\Omega}\mathscr{D} \phi \; \mathscr{D}  \mathscr{O}_n(f) = \int_{[-2L, 2L]^d} \phi f. \]
        Moreover, $\|\mathscr{O}_n f\|_{H^s_{\mathrm{per}}([-2L, 2L]^d)} \leqslant \lambda_n^{-1} \|f\|_{L^2([-2L, 2L]^d)}$, i.e., $\mathscr{O}_n$ is bounded.
        \label{prop:diff_op}
\end{prop}
\begin{proof}
     We use the framework provided by \citet[][page 304]{evans2010partial} to prove the result. Let the bilinear form $B: H^s_{\mathrm{per}}([-2L, 2L]^d) \times H^s_{\mathrm{per}}([-2L, 2L]^d) \to \mathbb{R}$ be defined by
    \[B[u,v] = \lambda_n \sum_{|\alpha|\leqslant s} \int_{[-2L, 2L]^d}\partial^\alpha u\; \partial^\alpha v + \mu_n \int_{\Omega}\mathscr{D} u \; \mathscr{D}  v.
    \]
    Observe that $B$ is coercive since $B[u,u] \geqslant \lambda_n \|u\|_{H^s_{\mathrm{per}}([-2L, 2L]^d)}^2$. Moreover, using the Cauchy-Schwarz inequality $(x_1 + \cdots + x_N)^2 \leqslant N(x_1^2 + \cdots + x_N^2)$, we see that
    \begin{align*}
        \int_{\Omega}|\mathscr{D} u|^2 &=  \int_{\Omega}\Big|\sum_{|\alpha|\leqslant s} p_\alpha \partial^\alpha u\Big|^2 \\
        &\leqslant (\max_\alpha \|p_\alpha\|_\infty)^2    \int_{[-2L, 2L]^d}\Big(\sum_{|\alpha|\leqslant s}| \partial^\alpha u|\Big)^2\\
        &\leqslant (\max_\alpha \|p_\alpha\|_\infty)^2\; 2^{s}\;  \|u\|_{H^s_{\mathrm{per}}([-2L, 2L]^d)}^2.
    \end{align*}
    Therefore, using the Cauchy-Schwarz inequality, we have
    \[\Big|\int_{\Omega}\mathscr{D} u \; \mathscr{D}  v\Big| \leqslant (\max_\alpha \|p_\alpha\|_\infty)^2\; 2^{s}\;  \|u\|_{H^s_{\mathrm{per}}([-2L, 2L]^d)}\|v\|_{H^s_{\mathrm{per}}([-2L, 2L]^d)}.\]
    Thus, 
    \[|B[u,v]| \leqslant (\lambda_n + (\max_\alpha \|p_\alpha\|_\infty)^2\; 2^{s} \mu_n)  \|u\|_{H^s_{\mathrm{per}}([-2L, 2L]^d)}\|v\|_{H^s_{\mathrm{per}}([-2L, 2L]^d)},\] 
    showing thereby the continuity of $B$.

    Next, for $f \in L^2([-2L, 2L]^d)$, observe that $ \phi \mapsto \int_{[-2L, 2L]^d} \phi f$ is a bounded linear form on $H^s_{\mathrm{per}}([-2L, 2L]^d)$, since 
    \[\Big|\int_{[-2L, 2L]^d} \phi f\Big| \leqslant \|\phi\|_{L^2([-2L, 2L]^d)} \|f\|_{L^2([-2L, 2L]^d)} \leqslant \|\phi\|_{H^s_{\mathrm{per}}([-2L, 2L]^d)} \|f\|_{L^2([-2L, 2L]^d)}.\] Thus, the Lax-Milgram theorem \citep[Chapter 6.2, Theorem 1]{evans2010partial} ensures that for all $f \in L^2([-2L, 2L]^d)$, there is a unique element $w \in H^s_{\mathrm{per}}([-2L, 2L]^d)$ such that, for any test function $\phi \in H^s_{\mathrm{per}}([-2L, 2L]^d)$,
        \[  \lambda_n \sum_{|\alpha|\leqslant s} \int_{[-2L, 2L]^d}\partial^\alpha \phi\; \partial^\alpha w + \mu_n \int_{\Omega}\mathscr{D} \phi \; \mathscr{D}  w = \int_{[-2L, 2L]^d} \phi f. \]
     Call $\mathscr{O}_n$ the function associating $w$ to $f$. Then, by the uniqueness of $w$ provided by the Lax-Milgram theorem, we deduce that $\mathscr{O}_n$ is injective and linear. Moreover, using the coercivity of $B$, we have
     \begin{align*}
         \|\mathscr{O}_n f\|_{H^s_{\mathrm{per}}([-2L, 2L])}^2 &\leqslant \lambda_n^{-1} B[\mathscr{O}_nf,\mathscr{O}_nf] = \lambda_n^{-1} \langle \mathscr{O}_n f, f\rangle_{L^2([-2L, 2L]^d)} \\
         & \leqslant \lambda_n^{-1} \|f\|_{L^2([-2L, 2L]^d)} \|\mathscr{O}_n f\|_{L^2([-2L, 2L]^d)}\\
         & \leqslant \lambda_n^{-1} \|f\|_{L^2([-2L, 2L]^d)} \|\mathscr{O}_n f\|_{H^s_{\mathrm{per}}([-2L, 2L]^d)}.
     \end{align*}
     In particular, $\|\mathscr{O}_n f\|_{H^s_{\mathrm{per}}([-2L, 2L]^d)} \leqslant \lambda_n^{-1} \|f\|_{L^2([-2L, 2L]^d)}$, and the proof is complete.
\end{proof}

\begin{prop}[Diagonalization on $L^2$]
    There exists an orthonormal basis $(v_m)_{m\in \mathbb{N}}$ of the space $L^2([-2L, 2L]^d)$ of eigenfunctions of $\mathscr{O}_n$, associated with non-increasing strictly positive eigenvalues $(a_m)_{m\in \mathbb{N}}$, such that $\mathscr{O}_n = \sum_{m \in \mathbb{N}} a_m \langle v_m, \cdot\rangle_{L^2([-2L, 2L]^d)} v_m$. 
    \label{prop:dzL2}
\end{prop}
\begin{proof}
    By the Rellich-Kondrakov theorem (Theorem \ref{thm:rellichK}), the operator $\mathscr{O}_n: L^2([-2L, 2L]^d) \to L^2([-2L, 2L]^d)$ is compact. Moreover, by definition of $\mathscr{O}_n$, for all $f,g \in L^2([-2L, 2L]^d)$, one has $\langle f, \mathscr{O}_n g\rangle_{L^2([-2L, 2L]^d)} = B[\mathscr{O}_nf,\mathscr{O}_n g]  = \langle \mathscr{O}_n f,  g\rangle_{L^2([-2L, 2L]^d)}$. Therefore, $\mathscr{O}_n$ is self-adjoint. Furthermore, $\langle f, \mathscr{O}_n f\rangle_{L^2([-2L, 2L]^d)} = B[\mathscr{O}_nf,\mathscr{O}_n f] \geqslant \lambda_n \|\mathscr{O}_n f\|_{H^s_{\mathrm{per}}([-2L, 2L])} > 0$, since $\mathscr{O}_n$ is injective. This means that $\mathscr{O}_n$ is strictly positive.  The result is then a consequence of the spectral theorem (Theorem \ref{thm:spectral}).
\end{proof}

\begin{prop}[Diagonalization on $H^s$]
    The orthonormal basis $(v_m)_{m\in \mathbb{N}}$ of Proposition \ref{prop:dzL2} is in fact a basis of $H^s_{\mathrm{per}}([-2L, 2L]^d)$. 
    Moreover, letting $C_1 = (\lambda_n + (\max_\alpha \|p_\alpha\|_\infty)^2\; 2^{s} \mu_n)$, we have that for all $f\in H^s_{\mathrm{per}}([-2L, 2L]^d)$, \[\sum_{m \in \mathbb{N}} a_m^{-1}\langle f , v_m\rangle_{L^2([-2L, 2L]^d)}^2 \leqslant  C_1 \|f\|_{H^s_{\mathrm{per}}([-2L, 2L]^d)}^2.\]
    \label{prop:diagH}
\end{prop}
\begin{proof}
    We follow the framework of \citet[][page 337]{evans2010partial}. First observe that $\mathscr{O}_n v_m = a_m v_m$ and $\|v_m\|_{L^2([-2L, 2L]^d)} =1$, implies that $\|\mathscr{O}_n v_m\|_{H^s_{\mathrm{per}}([-2L, 2L]^d)} \leqslant \lambda_n^{-1} \|v_m\|_{L^2([-2L, 2L])}$ implies that $\|v_m\|_{H^s_{\mathrm{per}}([-2L, 2L]^d)} \leqslant \lambda_n^{-1} a_m^{-1}$. Therefore, $v_m \in H^s_{\mathrm{per}}([-2L, 2L]^d)$, and we can apply $B$ to it. For all $m\in \mathbb{N}$,
    \[B[v_m, v_m] = B[v_m, a_m^{-1}\mathscr{O}_n v_m] = a_m^{-1} \langle v_m, v_m\rangle_{L^2([-2L, 2L]^d)} =  a_m^{-1}.\]
    Similarly, if $m \neq \ell$, $B[v_m, v_\ell] = a_m^{-1} \langle v_m, v_\ell\rangle_{L^2([-2L, 2L]^d)} =  0$. Remark that $B$ is a inner product on $H^s_{\mathrm{per}}([-2L,2L]^d)$ and $(\sqrt{a_m}v_m)_{m\in \mathbb{N}}$ is an orthonormal family for the $B$-inner product. Indeed, notice that if, for a fixed $u \in H^s_{\mathrm{per}}([-2L, 2L]^d)$, one has $B[v_m, u]= 0$ for all $m \in \mathbb N$, then $\langle v_m, u\rangle_{L^2([-2L, 2L]^d)} = 0$. Thus, since $(v_m)_{m\in \mathbb{N}}$ is an orthonormal basis of $L^2([-2L, 2L]^d)$, $u=0$.

    Let, for $N \in \mathbb{N}$ and $u \in H^s_{\mathrm{per}}([-2L, 2L]^d)$,  
    \[u_N = \sum_{m=0}^N B[u, a_m^{1/2} v_m] a_m^{1/2} v_m.\]
    Since $v_m \in H^s_{\mathrm{per}}([-2L, 2L]^d)$, one has $u_N \in H^s_{\mathrm{per}}([-2L, 2L]^d)$.
    Upon noting that $B[u-u_N, u-u_N] \geqslant 0$ and that $B[u-u_N, u-u_N] = B[u, u] - \sum_{m=0}^N B[u, a_m^{1/2}v_m]^2$ (using the bilinearity of~$B$), we derive the following Bessel's inequality for $B$ 
    \begin{equation}
        \sum_{m=0}^\infty B[u, a_m^{1/2}v_m]^2 \leqslant B[u,u] \leqslant C_1 \|u\|_{H^s_{\mathrm{per}}([-2L, 2L]^d)}^2.
        \label{eq:bessel}
    \end{equation}
    Then, for all $\ell \geqslant p$, 
    \[B[u_\ell - u_p, u_\ell - u_p] = \sum_{m = p}^\ell B[u, a_m^{1/2} v_m]^2 \leqslant \sum_{m = p}^\infty B[u, a_m^{1/2} v_m]^2 \xrightarrow{p\to \infty} 0.\]
    This shows that $(u_N)_{N\in \mathbb{N}}$ is a Cauchy sequence for the $B$-inner product. Since, $B[u_\ell - u_p, u_\ell - u_p] \geqslant \lambda_n^{-1} \|u_\ell - u_p\|_{H^s([-2L, 2L]^d)}^2$, $(u_N)_{N\in \mathbb{N}}$ is also a Cauchy sequence for the $\|\cdot\|_{H^s_{\mathrm{per}}([-2L, 2L]^d)}$ norm. Recalling that $H^s_{\mathrm{per}}([-2L, 2L]^d)$ is a Banach space, we deduce that $u_\infty := \lim_{N\to \infty} u_N$ exists and belongs to $H^s_{\mathrm{per}}([-2L, 2L]^d)$. 
    Since $B$ is continuous with respect to the  $\|\cdot\|_{H^s_{\mathrm{per}}([-2L, 2L]^d)}$ norm, we also deduce that, for all $m \in \mathbb{N}$, $B[u-u_\infty, v_m] = 0$, i.e., $u = u_\infty$. 
    In conclusion, \[u = \sum_{m\in \mathbb{N}} B[u, a_m^{1/2} v_m] a_m^{1/2} v_m.\]
    This means that $(v_m)_{m\in \mathbb{N}}$ is a basis of $H^s_{\mathrm{per}}([-2L, 2L]^d)$. Moreover, using the Bessel's inequality \eqref{eq:bessel}, we have that
    $\sum_{m=0}^\infty B[u, a_m^{1/2}v_m]^2 = \sum_{m\in \mathbb{N}} a_m^{-1}\langle u , v_m\rangle_{L^2([-2L,2L]^d)}^2 \leqslant C_1 \|u\|_{H^s_{\mathrm{per}}([-2L, 2L]^d)}^2$.
\end{proof}

\begin{prop}[Differential inner product]
    The operators
    \begin{itemize}
        \item $\mathscr{O}_n^{-1/2}:H^s_{\mathrm{per}}([-2L, 2L]^d) \to L^2([-2L, 2L]^d)$, defined by \[\mathscr{O}_n^{-1/2} = \sum_{m\in\mathbb{N}} a_m^{-1/2} \langle v_m, \cdot\rangle_{L^2([-2L,2L]^d)} v_m,\]
        \item \mbox{and } $\mathscr{O}_n^{1/2}:L^2([-2L, 2L]^d) \to H^s_{\mathrm{per}}([-2L, 2L]^d)$, defined by \[\mathscr{O}_n^{1/2} = \sum_{m\in\mathbb{N}} a_m^{1/2} \langle v_m, \cdot\rangle_{L^2([-2L,2L]^d)} v_m,\]
    \end{itemize} are well-defined and bounded. Moreover, for all $u \in H^s_{\mathrm{per}}([-2L, 2L]^d)$, \[\mathscr{O}_n^{1/2} \mathscr{O}_n^{-1/2} (u) = u,\] and
    \[\|\mathscr{O}_n^{-1/2}(u)\|_{L^2([-2L, 2L]^d)}^2 = \lambda_n \|u\|_{H^s_{\mathrm{per}}([-2L, 2L]^d)}^2 + \mu_n \|\mathscr{D}(u)\|_{L^2(\Omega)}^2.\]
    \label{prop:scalar_product}
\end{prop}
\begin{proof}
    Proposition \ref{prop:diagH} shows that, for all $u \in H^s_{\mathrm{per}}([-2L, 2L]^d)$, $\sum_{m\in \mathbb{N}} a_m^{-1}\langle u , v_m\rangle_{L^2([-2L,2L]^d)}^2 \leqslant C_1 \|u\|_{H^s([-2L, 2L]^d)}^2$. Therefore, $(\sum_{m=0}^N a_m^{-1/2} \langle v_m, u\rangle_{L^2([-2L,2L]^d)} v_m)_{N\in \mathbb{N}}$ is a Cauchy sequence converging in $L^2([-2L, 2L]^d)$. Denote the limit of this sequence  by $\mathscr{O}_n^{-1/2} (u)$.
    Since $B[u,u] = \sum_{m\in \mathbb{N}} a_m^{-1}\langle u , v_m\rangle_{L^2([-2L,2L]^d)}^2$, we deduce that \[\|\mathscr{O}_n^{-1/2}(u)\|_{L^2([-2L, 2L]^d)}^2 = \lambda_n \|u\|_{H^s_{\mathrm{per}}([-2L, 2L]^d)}^2 + \mu_n \|\mathscr{D}(u)\|_{L^2(\Omega)}^2.\] 
    Finally, using $\|\mathscr{O}_n^{-1/2} (u)\|_{L^2([-2L, 2L]^d)}^2  \leqslant C_1 \|u\|_{H^s_{\mathrm{per}}([-2L, 2L]^d)}^2$, we conclude that the operator~$\mathscr{O}_n^{-1/2}$ is bounded.
    
    Moreover, $(\sum_{m=0}^N a_m^{1/2} \langle v_m, u\rangle_{L^2([-2L,2L]^d)} v_m)_{N\in \mathbb{N}}$ is also a Cauchy sequence in the space $H^s_{\mathrm{per}}([-2L, 2L]^d)$. To see this, 
    note that $\sum_{m=0}^N a_m^{1/2} \langle v_m, u\rangle_{L^2([-2L,2L]^d)} v_m \in H^s_{\mathrm{per}}([-2L, 2L]^d) $, and that this sequence is a Cauchy sequence for the $B$ inner product, because
    \begin{align*}
        &B\Big[\sum_{m=p}^\ell a_m^{1/2} \langle v_m, u\rangle_{L^2([-2L,2L]^d)} v_m, \sum_{m=p}^\ell a_m^{1/2} \langle v_m, u\rangle_{L^2([-2L,2L]^d)} v_m\Big] \\
        &\quad = \sum_{m=p}^\ell a_m \langle v_m, u\rangle_{L^2([-2L,2L]^d)} ^2 B[v_m, v_m]\\
        &\quad = \sum_{m=p}^\ell  \langle v_m, u\rangle_{L^2([-2L,2L]^d)} ^2 \\
        &\quad \leqslant \sum_{m = p}^\infty \langle v_m, u\rangle_{L^2([-2L,2L]^d)}^2  \xrightarrow{p\to \infty} 0.
    \end{align*}
    Thus, it $\sum_{m=0}^N a_m^{1/2} \langle v_m, u\rangle_{L^2([-2L,2L]^d)} v_m$ converges in $H^s_{\mathrm{per}}([-2L, 2L]^d)$ to a limit that we denote by $\mathscr{O}_n^{1/2}(u)$. By the continuity of $B$, $B[\mathscr{O}_n^{1/2} (u),\mathscr{O}_n^{1/2}(u)] \leqslant \|u\|_{L^2([-2L, 2L]^d)}^2$. Therefore, $\|\mathscr{O}_n^{1/2}(u)\|_{H^s_{\mathrm{per}}([-2L, 2L]^d)}^2 \leqslant \lambda_n^{-1} \|u\|_{L^2([-2L, 2L]^d)}^2$, i.e., $\mathscr{O}_n^{1/2}$ is bounded.
    
    To conclude the proof, observe that since $(\sum_{m=0}^N a_m^{-1/2} \langle v_m, u\rangle_{L^2([-2L,2L]^d)} v_m)_{N\in \mathbb{N}}$ converges to $ \mathscr{O}_n^{-1/2}(u)$ in $L^2([-2L, 2L]^d)$, and since the inner product $\langle\cdot, \cdot\rangle$ is continuous with respect to the $L^2([-2L, 2L]^d)$ norm (by the Cauchy-Schwarz inequality), then, for all $u \in H^s_{\mathrm{per}}([-2L, 2L]^d)$, one can write $\langle \mathscr{O}_n^{-1/2}(u), v_m\rangle_{L^2([-2L,2L]^d)} = a_m^{-1/2} \langle v_m, u\rangle_{L^2([-2L,2L]^d)}$. Besides, since the sequence $(\sum_{m=0}^N a_m^{1/2} \langle v_m, \mathscr{O}_n^{-1/2}(u)\rangle_{L^2([-2L,2L]^d)} v_m)_{N\in \mathbb{N}}$ converges to $\mathscr{O}_n^{1/2}  \mathscr{O}_n^{-1/2}(u)$ in $L^2([-2L, 2L]^d)$, one has that $\langle \mathscr{O}_n^{1/2} \mathscr{O}_n^{-1/2}(u), v_m\rangle_{L^2([-2L,2L]^d)} = \langle u, v_m\rangle_{L^2([-2L,2L]^d)}$. Finally, this shows that $\mathscr{O}_n^{1/2} \mathscr{O}_n^{-1/2}(u) = u$, and the proof is complete.
\end{proof}

Recall from Proposition \ref{prop:countable} that there exists a countable re-indexing $k: \mathbb{N} \to \mathbb{Z}^d$ such that $\|k\|_1$ is non-decreasing. Recall that we have let $e_{k(\ell)}(x) := \exp(i\langle k(\ell), x\rangle)$.

\begin{lem}[Non-empty intersection]
    Let $\mathscr{V}$ be a linear subspace of $H^s_{\mathrm{per}}([-2L,2L]^d)$ such that $\dim \mathscr{V} = m+1$.
    Then $\mathscr{V} \cap \mathrm{Span}( e_{k(\ell)})_{\ell \geqslant m} \neq \emptyset$.
    \label{lem:non_empty}
\end{lem}
\begin{proof}
    Let $z_0, \hdots, z_m$ be a basis of $\mathscr{V}$. 
    Let us consider the linear function \[T: (x_0, \hdots, x_m) \in \mathbb{R}^{m+1} \mapsto \Big( \langle e_{k(j)}, \sum_{\ell = 0}^m x_\ell z_\ell\rangle_{L^2([-2L,2L]^d)}\Big)_{0\leqslant j \leqslant m-1} \in \mathbb{R}^m.\]
    The rank–nullity theorem ensures that the dimension of the kernel of $T$ is at least 1. Thus, there is a linear combination $z = x_0 z_0 + \cdots + x_m z_m$ such that, for all $\ell \leqslant m - 1$, $\langle z, e_{k(\ell)}\rangle_{L^2([-2L,2L]^d)} = 0$ and $z \neq 0$. Since $(e_{k(\ell)})_{\ell \geqslant 0}$ is a basis of $H^s_{\mathrm{per}}([-2L,2L]^d)$, we conclude that $z \in \mathrm{Span}( e_{k(\ell)})_{\ell \geqslant m} \cap \mathscr{V}$.
\end{proof}

\begin{prop}[Eigenvalues of the differential operator]
    There is a constant $C_2>0$, depending only on $d$ and $s$, such that, for all $m\in \mathbb{N}$, 
    \[a_m \leqslant C_2 \lambda_n^{-1} m^{-2s/d}.\]
    In particular, $\sum_{m\in \mathbb{N}} a_m < \infty$ if $s > d/2$.
    \label{prop:eigenvalue}
\end{prop}
\begin{proof}
    From the proof of Proposition \ref{prop:countable}, we know that there exists a constant $C_1 > 0$, depending on $d$ and $s$ such that $(1+\|k(m)\|_2^2)^{s/2} \geqslant \|k(m)\|_2^{s} \geqslant \|k(m)\|_1^{s} / d^{s} \geqslant C_1 m^{s/d}$. 
    Therefore, there exists a constant $C_3 > 0$, depending on $d$ and $s$, such that  \begin{equation}
        \sum_{|\alpha|\leqslant s} \Big(\frac{\pi}{2L}\Big)^{2|\alpha|} \prod_{j=1}^d k_j(m)^{2\alpha_j} \geqslant C_3 m^{2s/d}.
        \label{eq:controlEigenValues}
    \end{equation}
    Let $C_2 = 2C_3^{-1}$, and let us prove Proposition \ref{prop:eigenvalue} by contradiction. Thus, suppose that there is an integer $m$ such that $a_m > 2C_3^{-1} \lambda_n^{-1} m^{-2s/d}$. Then, for all $\ell \leqslant m$, $a_\ell^{-1} < C_3 \lambda_n m^{2s/d}/2$ and $C_3 \lambda_n m^{2s/d}/2 > B[v_\ell, v_\ell] \geqslant \lambda_n \|v_\ell \|_{H^s([-2L,2L]^d)}^2$. Thus, $\mathscr{V} = \mathrm{Span}(v_0, \hdots, v_m)$ is a subspace of $H^s_{\mathrm{per}}([-2L,2L]^d)$ of dimension $m+1$. In particular, for all $z \in \mathscr{V}$, there are weights $\beta_\ell \in \mathbb{R}$ such that $z = \sum_{j=0}^m \beta_\ell v_\ell$, and so $\|z \|_{L^2([-2L,2L]^d)}^2 = \sum_{\ell=0}^m \beta_\ell^2 $. Hence, 
    \begin{equation}
        \lambda_n \|z \|_{H^s([-2L,2L]^d)}^2 \leqslant B[z,z] = \sum_{\ell=0}^m \beta_j^2 a_\ell^{-1} \leqslant  \|z \|_{L^2([-2L,2L]^d)}^2 \lambda_n C_3 m^{2s/d}/2.
        \label{eq:bounded_energy}
    \end{equation}
    Let $S_s : H^s_{\mathrm{per}}([-2L,2L]^d) \to L^2([-2L,2L]^d)$ be the operator such that \[S_s(e_{k(\ell)}) = \Big(\sum_{|\alpha|\leqslant s} \Big(\frac{\pi}{2L}\Big)^{2|\alpha|} \prod_{j=1}^d k_j(\ell)^{2\alpha_j} \Big)^{1/2} e_{k(\ell)}.\] Then, by definition, $S_s$
    is diagonalizable on $H^s_{\mathrm{per}}([-2L,2L]^d)$ with eigenfunctions $e_{k(\ell)}$ and, for all $f \in H^s_{per}([-2L,2L]^d)$, $\|S_s(f)\|_{L^2([-2L,2L]^d)} = \|f\|_{H^s_{per}([-2L,2L]^d)}$. 
    Since $\dim \mathscr{V} = m+1$, Lemma \ref{lem:non_empty} ensures that $\mathscr{V} \cap \mathrm{Span}( e_{k(\ell)})_{\ell \geqslant m} \neq \emptyset$. However, any $z \in \mathrm{Span}( e_{k(\ell)})_{\ell \geqslant m}$ can be written $z = \sum_{j\geqslant m} \beta_\ell e_{k(\ell)}$, for weights $\beta_\ell \in \mathbb{R}$. Thus, using \eqref{eq:controlEigenValues}, we have that
    \begin{align*}
        \|z \|_{H^s([-2L,2L]^d)}^2 = \sum_{j\geqslant m} \beta_\ell^2 \|e_{k(\ell)} \|_{H^s([-2L,2L]^d)}^2 &\geqslant C_3 m^{2s/d}\sum_{\ell \geqslant m} \beta_\ell ^2\\
        &=  C_3 m^{2s/d}\|z \|_{L^2([-2L,2L]^d)}^2.
    \end{align*}
    Since, by assumption, $z \in \mathscr{V}$, this contradicts \eqref{eq:bounded_energy}.
\end{proof}

\begin{rem}[Lower bound on $a_m^{-1}$]
    Using similar arguments but bounding the eigenvalues of $S_s$ by $\sum_{|\alpha|\leqslant s} \Big(\frac{\pi}{2L}\Big)^{2|\alpha|} \prod_{j=1}^d k_j(\ell)^{2\alpha_j} \geqslant 1$, or directly applying the so-called Rayleigh's formula \citep[Chapter 6.5, Theorem 2]{evans2010partial}, one shows that, for all $m \geqslant 0$, $a_m^{-1} \geqslant \lambda_n$.
    \label{rem:rayleigh}
\end{rem}
Let $x \in [-2L,2L]^d$. Let $\delta_x$ be the Dirac distribution, i.e., the linear form on $C^0([-2L,2L]^d)$ such that, for all $f \in C^0([-2L,2L]^d)$, $\langle \delta_x, f\rangle = f(x)$. Notice that $\delta_x$ is continuous with respect to the $\|\cdot \|_\infty$ norm. In the sequel, with a slight abuse of notation, we replace $\delta_x(f)$ by $\langle \delta_x, f\rangle$. In effect, $\delta_x$ can be approximated by a regularizing sequence $(\xi^x_m)_{m\in \mathbb{N}}$ with respect to the $L^2([-2L,2L]^d)$ inner product, i.e., \[\forall f \in C^0([-2L,2L]^d), \quad \lim_{m\to\infty} \langle \xi^x_m, f\rangle_{L^2([-2L,2L]^d)} = f(x).\] Therefore, the action of $\delta_x$ on $f$ behaves like an inner product on $L^2([-2L,2L]^d)$, and this intuition will be fruitful in the next Proposition. Moreover, since $H^s_{\mathrm{per}}([-2L,2L]^d) \subseteq H^s([-2L,2L]^d) \subseteq C^0([-2L,2L]^d)$, when ``applied" to any $f \in H^s_{\mathrm{per}}([-2L,2L]^d)$, $\delta_x$ can be considered as the evaluation at $x$ of the unique continuous representation of $f$. The following proposition shows that $\mathscr{O}_n^{1/2}$ can be extended to $\delta_x$ in such a way that this extension stays self-adjoint.
\begin{prop}[Self-adjoint operator extension]
    Let $s>d/2$ and, for $x \in [-2L, 2L]^d$, let $\mathscr{O}_n^{1/2} (\delta_x) = \sum_{m\in \mathbb{N}} a_m^{1/2} v_m(x) v_m$. Then, almost everywhere in $x$ according to the Lebesgue measure on $[-2L, 2L]^d$, $\mathscr{O}_n^{1/2} (\delta_x) \in L^2([-2L,2L]^d)$ and, for all $f \in L^2([-2L,2L]^d)$,  
    \[\langle \mathscr{O}_n^{1/2}(f), \delta_x\rangle = \langle   f, \mathscr{O}_n^{1/2}(\delta_x) \rangle_{L^2([-2L,2L]^d)}.\] 
    \label{prop:extension}
\end{prop}
\begin{proof}
    Let $\psi_N(x,y) = \sum_{m=0}^N \alpha_m^{1/2} v_m(x) v_m(y)$. Then, for all $N_1 \leqslant N_2$,
    \begin{align*}
        &\int_{[-2L, 2L]^d}\int_{[-2L, 2L]^d} |\psi_{N_2}(x,y)-\psi_{N_1}(x,y)|^2dx dy \\
        &= \int_{[-2L, 2L]^d}\int_{[-2L, 2L]^d} \Big|\sum_{m=N_1+1}^{N_2} a_m^{1/2} v_m(x) v_m(y)\Big|^2dx dy\\
        &= \sum_{m, \ell =N_1+1}^{N_2} a_m \int_{[-2L, 2L]^d} v_m(x)v_\ell(x) dx \int_{[-2L, 2L]^d}v_m(y) v_\ell(y) dy\\
        &= \sum_{m=N_1+1}^{N_2} a_m \leqslant \sum_{m=N_1+1}^{\infty} a_m.
    \end{align*}
    Proposition \ref{prop:eigenvalue} shows that $\lim_{N_1 \to \infty}\sum_{m=N_1}^{\infty} a_m = 0$, hence  $(\psi_N)_{N\in \mathbb{N}}$ is a Cauchy sequence. Therefore, $\psi_\infty (x,y) = \sum_{m\in \mathbb{N}} a_m^{1/2} v_m(x) v_m(y)$ converges in $L^2([-2L,2L]^d \times [-2L,2L]^d)$ and 
    \begin{equation*}
        \int_{[-2L, 2L]^{2d}} |\psi_\infty(x,y)|^2dx dy  = \sum_{m\in \mathbb{N}} a_m.
    \end{equation*}
    Thus, by the Fubini-Lebesgue theorem, almost everywhere in $x$ according to the Lebesgue measure on $[-2L, 2L]^d$, one has $\mathscr{O}_n^{1/2}(\delta_x):= \psi_\infty(x,\cdot) \in L^2([-2L,2L]^d)$. Recall that, by definition, \[\mathscr{O}_n^{1/2} (f) = \sum_{m\in \mathbb{N}} a_m^{1/2}\langle f, v_m\rangle_{L^2([-2L,2L]^d)} v_m \in H^s_{\mathrm{per}}([-2L,2L]^d),\] so that $\langle \mathscr{O}_n^{1/2} (f), \delta_x\rangle = \sum_{m\in \mathbb{N}} a_m^{1/2}\langle f, v_m\rangle_{L^2([-2L,2L]^d)} v_m(x)$. Moreover, for any function $f \in L^2([-2L,2L]^d)$, 
    \begin{align*}
        &\int_{[-2L, 2L]^d} |\langle f, \psi_N(x, \cdot)\rangle_{L^2([-2L,2L]^d)} - \langle \mathscr{O}_n^{1/2} (f), \delta_x\rangle| dx\\
        &\quad = \int_{[-2L, 2L]^d}  \Big|\int_{[-2L, 2L]^d} f(y) \sum_{m=0}^N a_m^{1/2} v_m(x) v_m(y)dy \\
        & \qquad \qquad \qquad \quad - \sum_{m\in \mathbb{N}} a_m^{1/2}\langle f, v_m\rangle_{L^2([-2L,2L]^d)} v_m(x)\Big|dx\\
         &\quad = \int_{[-2L, 2L]^d}  \Big|\int_{[-2L, 2L]^d} f(y) \sum_{m > N} a_m^{1/2} v_m(x) v_m(y)\Big|dx dy \xrightarrow{ N \to \infty} 0.
    \end{align*}
    Therefore, since
    \begin{align*}
        &\int_{[-2L, 2L]^d} |\langle f, \psi_\infty(x, \cdot)\rangle_{L^2([-2L,2L]^d)} - \langle \mathscr{O}_n^{1/2} (f), \delta_x\rangle| dx \\
        &\quad \leqslant \int_{[-2L, 2L]^d} |\langle f, \psi_N(x, \cdot)\rangle_{L^2([-2L,2L]^d)} - \langle \mathscr{O}_n^{1/2} (f), \delta_x\rangle| dx  \\
        &\qquad + \int_{[-2L, 2L]^d} |\langle f, \psi_\infty(x, \cdot) - \psi_N(x, \cdot)\rangle_{L^2([-2L,2L]^d)} | dx,
    \end{align*}
    and since 
    \begin{align*}
        &\int_{[-2L, 2L]^d} |\langle f, \psi_\infty(x, \cdot) - \psi_N(x, \cdot)\rangle_{L^2([-2L,2L]^d)} | dx \\
        &\quad \leqslant \Big(\int_{[-2L, 2L]^d}\int_{[-2L, 2L]^d} |f(y)|^2dydx\Big)^{1/2} \\
        &\qquad \times \Big(\int_{[-2L, 2L]^d}\int_{[-2L, 2L]^d}| \psi_\infty(x,y) - \psi_N(x,y)|^2 dydx\Big)^{1/2} \\
        & \quad \xrightarrow{N \to \infty} 0,
    \end{align*} 
    we deduce that $\int_{[-2L, 2L]^d} |\langle f, \psi_\infty(x, \cdot)\rangle_{L^2([-2L,2L]^d)} - \langle \mathscr{O}_n^{1/2} (f), \delta_x\rangle| dx = 0$. Hence, almost everywhere in $x$ according to the Lebesgue measure on $[-2L, 2L]^d$,  we get that  the operator $\mathscr{O}_n^{1/2}$ is self-adjoint, i.e., 
    $\langle \mathscr{O}_n^{1/2} (f), \delta_x\rangle = \langle   f, \mathscr{O}_n^{1/2}(\delta_x)\rangle$.
\end{proof}

\subsection{Proof of Theorem \ref{thm:PDE_kernel}}
Let $s>d/2$, $n\in\mathbb{N}$, $\lambda_n >0$, $\mu_n \geqslant 0$, and consider a linear partial differential operator $\mathscr{D}(u) = \sum_{|\alpha|\leqslant s} p_\alpha \partial^\alpha u$ of order $s$ such that $\max_\alpha \|p_\alpha\|_\infty < \infty$. Proposition \ref{prop:diff_op} and \ref{prop:scalar_product} show that there exists a compact self-adjoint differential operator $\mathscr{O}_n$ such that, for all $f \in H^s_{\mathrm{per}}([-2L,2L]^d)$, \[\|\mathscr{O}_n^{-1/2}(f)\|_{L^2([-2L,2L]^d)}^2 = \lambda_n \|f\|_{H^s([-2L,2L]^d)}^2+ \mu_n \|\mathscr{D}(f)\|_{L^2(\Omega)}^2.\]
Consider any target function $f \in H^s_{\mathrm{per}}([-2L,2L]^d)$. The Sobolev embedding theorem states that $H^s_{\mathrm{per}}([-2L,2L]^d) \subseteq H^s([-2L,2L]^d) \subseteq C^0([-2L,2L]^d)$. Thus, for all $x \in \Omega$, we have that $f(x) = \langle f, \delta_x\rangle$.
Proposition \ref{prop:scalar_product} ensures that $f(x) = \langle \mathscr{O}_n^{1/2}\mathscr{O}_n^{-1/2}(f), \delta_x\rangle$ and Proposition \ref{prop:extension} that, for almost every $x \in \Omega$ with respect to the Lebesgue measure, 
\[f(x) = \langle \mathscr{O}_n^{-1/2}(f), \mathscr{O}_n^{1/2}(\delta_x)\rangle_{L^2([-2L,2L]^d)},\]
with $\mathscr{O}_n^{-1/2}(f) \in L^2([-2L,2L]^d)$ and $\mathscr{O}_n^{1/2}(\delta_x) \in L^2([-2L,2L]^d)$. 
Proposition \ref{prop:scalar_product} shows that 
$\mathscr{O}_n^{1/2}(\delta_x) = \mathscr{O}_n^{-1/2}\mathscr{O}_n(\delta_x)$. Thus, 
\[f(x) = \langle f, \mathscr{O}_n(\delta_x)\rangle_{\mathrm{RKHS}},\]
where the RKHS inner product is defined by $\langle g, h\rangle_{\mathrm{RKHS}} = \langle \mathscr{O}_n^{-1/2}(g), \mathscr{O}_n^{-1/2}(h)\rangle_{L^2([-2L,2L]^d)}$.
Since $\mathscr{O}_n^{1/2}(\delta_x) \in L^2([-2L,2L]^d)$, Proposition \ref{prop:scalar_product} shows that $\mathscr{O}_n(\delta_x) \in H^s_{\mathrm{per}}([-2L,2L]^d)$.
We can therefore define the kernel
\begin{align*}
    K(x,y) &= \langle \mathscr{O}_n(\delta_x), \mathscr{O}_n(\delta_y)\rangle_{\mathrm{RKHS}} \\
    &= \langle \mathscr{O}_n^{1/2}(\delta_x), \mathscr{O}_n^{1/2}(\delta_y)\rangle_{L^2([-2L,2L]^d)}.
\end{align*}
Proposition \ref{prop:extension} ensures that $K(x,y) = \langle \mathscr{O}_n(\delta_x), \delta_y\rangle = \mathscr{O}_n(\delta_x)(y) = \sum_{m\in \mathbb{N}} a_m  v_m(x)v_m(y)$. Therefore, we know that $K(x,\cdot) \in H^s_{\mathrm{per}}([-2L,2L]^d)$, and we recognize the reproducing property stating that, for all $f \in H^s_{\mathrm{per}}([-2L,2L]^d)$ and all $x \in [-2L,2L]^d$, $f(x) = \langle f, K(x, \cdot)\rangle_{\mathrm{RKHS}}$.

\subsection{Proof of Proposition \ref{prop:kernel_characterization}}
 Recall that $K(x,\cdot) = \mathscr{O}_n(\delta_x)$.
 It was proven in Proposition \ref{prop:extension} that $\mathscr{O}_n^{1/2}(\delta_x) \in L^2([-2L,2L]^d)$. By Proposition \ref{prop:scalar_product}, $\sum_{m=0}^N a_m^{1/2}\langle v_m, \mathscr{O}_n^{1/2}(\delta_x)\rangle_{L^2([-2L,2L]^d)} v_m$ converges in $H^s_{\mathrm{per}}([-2L,2L]^d)$ to  $K(x, \cdot)$. Let $\phi \in H^s_{\mathrm{per}}([-2L,2L]^d)$ be a test function. Since $B$ is continuous on $H^s_{\mathrm{per}}([-2L,2L]^d)$, 
 \[\lim_{N\to \infty} B\Big[\sum_{m=0}^N a_m^{1/2}\langle v_m, \mathscr{O}_n^{1/2}(\delta_x)\rangle_{L^2([-2L,2L]^d)}v_m, \phi\Big] = B[K(x, \cdot), \phi].\]
    Then,
    \begin{align*}
        &B\Big[\sum_{m=0}^N a_m^{1/2}\langle v_m, \mathscr{O}_n^{1/2}(\delta_x)\rangle_{L^2([-2L,2L]^d)}v_m, \phi\Big] \\
        &\quad = \sum_{m=0}^N a_m^{1/2}\langle v_m, \mathscr{O}_n^{1/2}(\delta_x)\rangle_{L^2([-2L,2L]^d)}B[v_m, \phi] \quad  \hbox{(by bilinearity)}\\
         &\quad = \sum_{m=0}^N a_m^{-1/2}\langle v_m, \mathscr{O}_n^{1/2}(\delta_x)\rangle_{L^2([-2L,2L]^d)} \langle v_m, \phi \rangle_{L^2([-2L,2L]^d)} \quad  \hbox{(since $v_m$ is an eigenfunction)}\\
         &\quad = \sum_{m=0}^N a_m^{-1/2}\langle \mathscr{O}_n^{1/2}(v_m), \delta_x\rangle \langle v_m, \phi \rangle_{L^2([-2L,2L]^d)} \quad  \hbox{(by Proposition \ref{prop:extension})}\\
         &\quad = \sum_{m=0}^N \langle v_m, \phi \rangle_{L^2([-2L,2L]^d)} v_m(x).
    \end{align*}
    Notice that the expression above is the decomposition of $\phi$ on the $v_m$ basis. We conclude, as desired, that
    $B[K(x, \cdot), \phi] = \lim_{N\to \infty} B[\sum_{m=0}^N a_m^{1/2}\langle v_m, \mathscr{O}_n^{1/2}(\delta_x)\rangle_{L^2([-2L,2L]^d)}v_m, \phi] = \phi(x)$.

\section{Integral operator and eigenvalues}
\subsection{Compactness of $C\mathscr{O}_nC$}
\begin{lem}[Compactness]
    The operator $C\mathscr{O}_nC: L^2([-2L, 2L]^d) \to L^2([-2L, 2L]^d)$ is positive, compact, and self-adjoint.
    \label{lem:compactness}
\end{lem}
\begin{proof}
    Since $C$ is a self-adjoint projector, then, for all $f \in L^2([-2L, 2L]^d)$, $\|C(f)\|_{L^2([-2L,2L]^d)}^2 \leqslant \|f\|_{L^2([-2L,2L]^d)}^2$. 
    Thus, for any bounded sequence $(f_m)_{m\in \mathbb{N}}$ in $L^2([-2L, 2L]^d)$, the sequence $(C(f_m))_{m\in \mathbb{N}}$ is bounded.
    Since $\mathscr{O}_n$ is compact, upon passing to a subsequence, $(\mathscr{O}_nC(f_m))_{m\in \mathbb{N}}$ converges to $f_\infty \in L^2([-2L, 2L]^d)$.
    Therefore, $\lim_{m\to \infty}\|C\mathscr{O}_nC(f_m) - C(f_\infty)\|_{L^2([-2L,2L]^d)}^2 \leqslant \lim_{m\to \infty}\|\mathscr{O}_nC(f_m) - f_\infty\|_{L^2([-2L,2L]^d)}^2 = 0$, i.e., $(C\mathscr{O}_nC(f_m))_{m\in \mathbb{N}}$ converges to the function $C(f_\infty) \in L^2([-2L, 2L]^d)$. 
    So, $C\mathscr{O}_nC: L^2([-2L, 2L]^d) \to L^2([-2L, 2L]^d)$ is a compact operator.
    Moreover, given any $f \in L^2([-2L, 2L]^d)$, we have that $\langle f, C\mathscr{O}_nC (f)\rangle_{L^2([-2L,2L]^d)} = \| \mathscr{O}_n^{1/2} C(f)\|_{L^2([-2L,2L]^d)}^2 \geqslant 0$, which means that $C\mathscr{O}_nC$ is positive. Finally, $C\mathscr{O}_nC$ is self-adjoint, since $C$ and $\mathscr{O}_n$ are self-adjoint.
\end{proof}

\subsection{Proof of Theorem \ref{thm:eigenvalues}}
For clarity, the proof is divided into 4 steps. Steps 1 and 2 ensures that we can apply the Courant Fischer min-max theorem to the integral operator. 
Step 3 connects the Courant Fischer estimates of $L_K$ and $C\mathscr{O}_nC$. Finally, Step 4 establishes the result on the eigenvalues.

\paragraph{Step 1: Compactness of the integral operator.}
Let $L_{K, \mathscr{U}}$ be the integral operator
associated with the uniform distribution on $\Omega$, i.e.,
    \[\forall f \in L^2(\Omega), \forall x \in \Omega, \quad L_{K, \mathscr{U}}f(x) = \frac{1}{|\Omega|}\int_{\Omega} K(x,y) f(y) dy.\]
Since $K(x,y) = \sum_{m\in \mathbb{N}}a_m(\mathscr{O}_n) v_m(x) v_m(y)$, $\int_{[-2L,2L]^d} v_\ell v_m  = \mathbf{1}_{\ell = m}$, and $\sum_{m\in \mathbb{N}}a_m(\mathscr{O}_n) < \infty$, the Fubini-Lebesgue theorem states that
\[\int_{\Omega^2} |K(x,y)|^2 dx dy \leqslant \int_{[-2L,2L]^{2d}} |K(x,y)|^2 dx dy = \sum_{m\in \mathbb{N}}a_m^2(\mathscr{O}_n) < \infty,\] which implies that 
$L_{K, \mathscr{U}}$ is a Hilbert-Schmidt operator \citep[][Lemma 8.20]{renardy2004an}.
As a consequence, $L_{K, \mathscr{U}}$ is compact
\citep[][Theorem 8.83]{renardy2004an}.
Observe that $L_K f = L_{K, \mathscr{U}} f \frac{d\mathbb{P}_X}{dx}$. Let $C_2 >0$. Given any sequence $(f_n)_{n\in \mathbb{N}}$ such that $\|f_n\|_{L^2(\Omega, \mathbb{P}_X)} \leqslant C_2$, then, clearly, $\|f_n  \frac{d\mathbb{P}_X}{dx}\|_{L^2(\Omega)} \leqslant \kappa C_2$. This shows that the sequence $(f_n \frac{d\mathbb{P}_X}{dx})_{n\in \mathbb{N}}$ is bounded in $L^2(\Omega)$. Thus, since $L_{K, \mathscr{U}}$ is compact, upon passing to a  subsequence, $L_{K, \mathscr{U}}(f_n \frac{d\mathbb{P}_X}{dx}) = L_{K}(f_n)$ converges in $L^2(\Omega)$, and therefore in $L^2(\Omega, \mathbb{P}_X)$. This shows that the integral operator $L_K$ is compact.

\paragraph{Step 2: Courant Fischer min-max theorem.} Using $\frac{d\mathbb{P}_X}{dx} \leqslant \kappa$ and letting $\mathrm{FS}$ be the Fourier series operator, i.e., $\mathrm{FS}(f)(k) = \langle f, \exp(-i\frac{\pi}{2L}\langle k, \cdot\rangle)\rangle_{L^2([-2L,2L]^d)}$, we see that for all $f \in L^2(\Omega, \mathbb{P}_X)$,
\begin{align*}
&\lim_{n \to \infty} \Big\|f - \sum_{\|k\|_2 \leqslant n} \mathrm{FS}(f)(k) \exp(i \pi L^{-1} \langle k, \cdot\rangle)\Big\|_{L^2(\Omega, \mathbb{P}_X)} \\
& \quad \leqslant \kappa \lim_{n \to \infty}\Big\|f - \sum_{k \in \mathbb{Z}^d,\; \|k\| \leqslant n} \mathrm{FS}(f)(k) \exp(i \pi L^{-1} \langle k, \cdot\rangle)\Big\|_{L^2(\Omega)} = 0.
\end{align*}
Therefore, the Gram-Schmidt algorithm applied to the $(\exp(i \pi L^{-1} \langle k, \cdot\rangle))_{k \in \mathbb{Z}^d}$ family provides a Hermitian basis of $L^2(\Omega, \mathbb{P}_X)$. In particular, the space $L^2(\Omega, \mathbb{P}_X)$ is separable.
Since $L_K$ is a positive compact self-adjoint operator on $L^2(\Omega, \mathbb{P}_X)$, Theorem \ref{thm:spectral} and \ref{thm:courant_fischer} show that $L_K$ is diagonalizable with positive eigenvalues $(a_n(L_K))_{n \in \mathbb{N}}$, with 
\[a_n(L_K) = \underset{\dim \Sigma = n}{\underset{\Sigma \subseteq L^2(\Omega, \mathbb{P}_X)}{\max}} 
\underset{f \neq 0}{\min_{f \in \Sigma}} \|f\|_{L^2(\Omega, \mathbb{P}_X)}^{-2}\langle f, L_K f\rangle_{L^2(\Omega, \mathbb{P}_X)} . \]

\paragraph{Step 3: Switching integrals.}
Observe that, for all  $f\in L^2(\Omega, \mathbb{P}_X)$, 
\begin{align*}
    &\int_{\Omega^2} \sum_{m\in \mathbb{N}} a_m(\mathscr{O}_n) |f(x)| |f(y)| |v_m(x)| |v_m(y)| d\mathbb{P}_X(x) d\mathbb{P}_X(y)\\ &= \sum_{m \in \mathbb{N}}  a_m(\mathscr{O}_n) \Big(\int_{\Omega} |f(x)||v_m(x)|d\mathbb{P}_X(x)\Big)^2\\
    &\leqslant \sum_{m \in \mathbb{N}}  a_m(\mathscr{O}_n) \|f\|_{L^2(\Omega, \mathbb{P}_X)}^2 \int_{\Omega} |v_m(x)|^2 d\mathbb{P}_X(x)\\
    &\leqslant \|f\|_{L^2(\Omega, \mathbb{P}_X)}^2  \kappa \sum_{m \in \mathbb{N}}  a_m(\mathscr{O}_n) < \infty.
\end{align*}
In the last inequality, we used the fact that $\int_{\Omega} |v_m(x)|^2 d\mathbb{P}_X(x) \leqslant \kappa \int_{[-2L,2L]^d} |v_m(x)|^2 dx = \kappa$.
Therefore, according to the Fubini-Lebesgue theorem,
\begin{align*}
    \langle f, L_K f\rangle_{L^2(\Omega, \mathbb{P}_X)} &= \int_{\Omega^2} f(x) \Big(\sum_{m\in \mathbb{N}}  a_m(\mathscr{O}_n) v_m(x) v_m(y)\Big)  f(y)d\mathbb{P}_X(y) d\mathbb{P}_X(x)\\
    &= \sum_{m\in \mathbb{N}}  a_m(\mathscr{O}_n) \Big(\int_{\Omega^2} f(x)v_m(x) d\mathbb{P}_X(x)\Big)^2\\
    &= \Big\|\mathscr{O}_n^{1/2} \Big( f\frac{d\mathbb{P}_X}{dx}\Big)\Big\|_{L^2([-2L,2L]^d)}^2.
\end{align*}

\paragraph{Step 4: Comparison using Courant Fischer.}
Let $z =  f\frac{d\mathbb{P}_X}{dx}$. By noting that
 $f \frac{d\mathbb{P}}{dx} = f \frac{d\mathbb{P}}{dx} \mathbf{1}_\Omega$, we see that $Cz = z$ and $\langle f, L_K f\rangle_{L^2(\Omega, \mathbb{P}_X)} = \langle z, C\mathscr{O}_nC(z)\rangle_{L^2([-2L,2L]^d)}$. Therefore, for any $\Sigma \subseteq L^2(\Omega, \mathbb{P}_X)$, we have  
\begin{align*}
    \underset{f \neq 0}{\min_{f \in \Sigma}} \|f\|_{L^2(\Omega, \mathbb{P}_X)}^{-2}\langle f, L_K f \rangle_{L^2(\Omega, \mathbb{P}_X)} &= \underset{f \neq 0}{\min_{f \in \Sigma}}\|f\|_{L^2(\Omega, \mathbb{P}_X)}^{-2}\Big\langle f \frac{d\mathbb{P}_X}{dx}, C\mathscr{O}_nC\Big(f \frac{d\mathbb{P}_X}{dx} \Big)\Big\rangle_{L^2([-2L,2L]^d)}\\
    &\leqslant \underset{z \neq 0}{\min_{z \in \frac{d\mathbb{P}_X}{dx}\Sigma}} \kappa \|z\|_{L^2([-2L,2L]^d)}^{-2}\langle z, C\mathscr{O}_nC(z)\rangle_{L^2([-2L,2L]^d)},
\end{align*}
where the inequality is a consequence of $\|z\|_{L^2([-2L,2L]^d)}^2 = \int_\Omega |f|^2 (\frac{d\mathbb{P}_X}{dx})^2 \leqslant \kappa \|f\|_{L^2(\Omega, \mathbb{P}_X)}^{2}$.
Using $\frac{d\mathbb{P}_X}{dx} L^2([-2L,2L]^d) \subseteq L^2([-2L,2L]^d)$, we conclude that
\begin{align}
    &\underset{\dim \Sigma = m}{\underset{\Sigma \subseteq L^2(\Omega, \mathbb{P}_X)}{\max}} 
\underset{f \neq 0}{\min_{f \in \Sigma}}\|f\|_{L^2(\Omega, \mathbb{P}_X)}^{-2}\langle f
, L_K f\rangle_{L^2(\Omega, \mathbb{P}_X)} \nonumber\\
&\quad \leqslant \kappa \underset{\dim \Sigma = m}{\underset{\Sigma \subseteq L^2([-2L,2L]^d)}{\max}} 
\underset{z \neq 0}{\min_{z \in \Sigma}}\|z\|^{-2}_{L^2([-2L,2L]^d)}\langle z,  C\mathscr{O}_nC(z) \rangle_{L^2([-2L,2L]^d)}.
\label{eq:courantFischerComparison}
\end{align}
According to Lemma \ref{lem:compactness}, the operator $C\mathscr{O}_nC$ is compact, self-adjoint, and positive, and thus its eigenvalues are given by the Courant-Fischer min-max theorem. Remark that the left-hand side (resp.\ the right-hand term) of inequality \eqref{eq:courantFischerComparison} corresponds to the Courant-Fischer min-max characterization of the $m$th eigenvalue of $L_K$ (resp.\ $C\mathscr{O}_nC$).
Therefore, we deduce that $a_m(L_K) \leqslant \kappa a_m(C\mathscr{O}_nC)$.

\subsection{Bounding the kernel}
The goal of this section is to upper bound the kernel $K(x,y)$ defined in Theorem \ref{thm:PDE_kernel}.

\begin{prop}[Partial continuity of the kernel]
    Let $x,y \in [-2L,2L]^d$. Both functions $K(x,\cdot)$ and $K(\cdot, y)$ are continuous. 
\end{prop}
    
\begin{proof}
     It is shown in the proof of  Proposition \ref{prop:extension} that $\psi_\infty(x,y) := \sum_{m\in\mathbb{N}} a_m^{1/2} v_m(x) v_m(y)$  converges in $L^2([-2L,2L]^d \times [-2L,2L]^d)$, that $\int_{[-2L, 2L]^{2d}} |\psi_\infty(x,y)|^2dx dy  = \sum_{m\in \mathbb{N}} a_m$, and that $\psi_\infty(x, \cdot) := \sum_{m\in\mathbb{N}} a_m^{1/2} v_m(x) v_m$ converges in $L^2([-2L,2L]^d)$ almost everywhere in $x$.
    By definition, $K(x, \cdot) = \mathscr{O}_n^{1/2}\psi_\infty(x, \cdot)$. Using Proposition \ref{prop:scalar_product}, this implies  \[\|K(x, \cdot)\|_{H^s_{\mathrm{per}}([-2L,2L]^d)}^2 \leqslant \lambda_n^{-1} \|\psi_\infty(x, \cdot)\|_{L^2([-2L,2L]^d)}^2.\]
    The Sobolev embedding theorem then ensures that $K(x, \cdot)$ is continuous for any $x$. One shows with the same argument that $K(\cdot, y)$ is continuous for any $y$.
\end{proof}

\begin{lem}[Trace reconstruction]
    Let $z \in [-2L,2L]^d$. Let $(\psi_\ell)_{\ell \in \mathbb{N}}$ be a sequence of functions in $L^2([-2L,2L]^d)$ such that $\int_{[-2L,2L]^d} \psi_\ell^2 =1$ and $\lim_{\ell \to \infty}\psi_\ell = \delta_z$. Then 
    \[\lim_{\ell \to \infty} \int_{[-2L,2L]^{2d}} K(x,y) \psi_\ell(x) \psi_\ell(y) dx dy = K(z,z). \]
    \label{lem:trace_rec}
\end{lem}
\begin{proof}
    \begin{align*}
        &\Big|\int_{[-2L,2L]^{2d}} K(x,y) \psi_\ell(x) \psi_\ell(y) dx dy - \int_{[-2L,2L]^{d}} \psi_\ell(x) K(x,z) dx \Big|\\
        &\quad =\Big| \int_{[-2L,2L]^{d}} \psi_\ell(x) \Big(K(x,z) -\int_{[-2L,2L]^{d}} K(x,y)  \psi_\ell(y)dy\Big) dx\Big|\\
        &\quad \leqslant \Big(\int_{[-2L,2L]^{d}} \psi_\ell^2(x) dx \Big)^{1/2} \Big(\int_{[-2L,2L]^{d}} \Big(K(x,z) -\int_{[-2L,2L]^{d}} K(x,y)  \psi_\ell(y)dy\Big)^2 dx \Big)^{1/2}.
    \end{align*}
    Recall that $\int_{[-2L,2L]^{d}} \psi_\ell^2(x) dx = 1$ and $\lim_{\ell \to \infty}\int_{[-2L,2L]^{d}} K(x,y)  \psi_\ell(y)dy = K(x,z)$.
    Let \[g_\ell(x) = \bigg(K(x,z) -\int_{[-2L,2L]^{d}} K(x,y)  \psi_\ell(y)dy\bigg)^2.\] Notice that 
    \begin{align*}
        |g_\ell(x)| &\leqslant 2 K^2(x,z) + 2 \Big|\int_{[-2L,2L]^{d}} K(x,y)  \psi_\ell(y)dy\Big|^2\\
        &\leqslant 2 K^2(x,z) + 2 \int_{[-2L,2L]^{d}} K^2(x,y)dy, 
    \end{align*}
    where we use the Cauchy-Schwarz inequality 
    \[\Big|\int_{[-2L,2L]^{d}} K(x,y)  \psi_\ell(y)dy\Big|^2 \leqslant \int_{[-2L,2L]^{d}} K^2(x,y) dy \times \int_{[-2L,2L]^{d}} \psi_\ell^2(y) dy\]
    and $\int_{[-2L,2L]^{d}} \psi_\ell^2(y) dy = 1$. Moreover, for almost every $z$, 
    \begin{align*}  
    &\int_{[-2L,2L]^{d}} \Big(2 K^2(x,z) + 2 \int_{[-2L,2L]^{d}} K^2(x,y)dy\Big)dx \\
 & \quad \leqslant 2 \int_{[-2L,2L]^{d}} K^2(x,z)dx + 2 \int_{[-2L,2L]^{2d}} K^2(x,y)dxdy 
 < \infty.
 \end{align*}
 Therefore, using the dominated convergence theorem, we see that $\lim_{\ell \to \infty} \int_{[-2L,2L]^{d}} g_\ell (x)dx = \int_{[-2L,2L]^{d}} \lim_{\ell \to \infty}  g_\ell (x)dx$. Since $\lim_{\ell \to \infty}\psi_\ell = \delta_z$, by the partial continuity of the kernel, we know that $\lim_{\ell \to \infty}  g_\ell (x) = 0$. So, 
 \[\lim_{\ell \to \infty} \Big|\int_{[-2L,2L]^{2d}} K(x,y) \psi_\ell(x) \psi_\ell(y) dx dy - \int_{[-2L,2L]^{d}} \psi_\ell(x) K(x,z) dx \Big| = 0,\]
 and
 \[\lim_{\ell \to \infty} \int_{[-2L,2L]^{2d}} K(x,y) \psi_\ell(x) \psi_\ell(y) dx dy = K(z,z). \]
\end{proof}
\begin{prop}[Bounding the kernel]
    Let $z \in [-2L,2L]^d$. One has $|K(z,z)|\leqslant \lambda_n^{-1}$.
    \label{prop:bound_kernel}
\end{prop}
\begin{proof}
    As in the proof of Theorem \ref{thm:eigenvalues}, it is easy to show that the operator $  L: f \mapsto (x \mapsto \int_{[-2L,2L]^d} K(x,y) f(y) dy)$ is compact and that $\langle f, L(f)\rangle_{L^2([-2L,2L]^d)}  = \langle f, \mathscr{O}_n(f)\rangle_{L^2([-2L,2L]^d)}$. Thus, the eigenvalues of $L$ are upper bounded by those of $\mathscr{O}_n$, and in turn, using Remark \ref{rem:rayleigh}, by $\lambda_n^{-1}$. Lemma \ref{lem:trace_rec} states that \[ \lim_{\ell \to \infty} \langle \psi_\ell, L(\psi_\ell)\rangle_{L^2([-2L,2L]^d)} = K(z,z).\]  Thus, the Courant-Fischer min-max theorem states that $\langle \psi_\ell, L(\psi_\ell)\rangle_{L^2([-2L,2L]^d)} \leqslant \lambda_n^{-1}$, and that $K(z,z) \leqslant \lambda_n^{-1}$.
\end{proof}

\subsection{Proof of Theorem \ref{prop:eigenfunction}}

For clarity, the proof will be divided into three steps.

\paragraph{Step 1: Weak formulation.}
 According to Lemma \ref{lem:compactness}, the operator $C\mathscr{O}_n C$ can be diagonalized in an orthonormal basis. Therefore, there are eigenfunctions $v_m \in L^2([-2L,2L]^d)$ and eigenvalues $a_m$ such that 
    \[ C\mathscr{O}_n C(v_m) = a_m v_m.\]
    Define $w_m = \mathscr{O}_n C(v_m)$. Given that $C(v_m) \in L^2([-2L,2L]^d)$, Proposition \ref{prop:diff_op} shows that $w_m\in H^s_{\mathrm{per}}([-2L,2L]^d)$. 
    Notice that $Cw_m = a_m v_m$. Since $C^2 = C$, we have \[v_m = C(v_m) = a_m^{-1} C(w_m).\] 
    By definition of the operator $\mathscr{O}_n$, for any test function $\phi \in H^s_{\mathrm{per}}([-2L,2L]^d)$, 
    \[B[w_m, \phi] = \langle C(v_m), \phi\rangle_{L^2([-2L,2L]^d)} = a_m^{-1} \langle C(w_m), \phi\rangle_{L^2([-2L,2L]^d)}.\]
    This means that $w_m$ is a weak solution to the PDE
    \[  \lambda_n \sum_{|\alpha|\leqslant s} \int_{[-2L, 2L]^d}\partial^\alpha \phi\; \partial^\alpha w_m + \mu_n \int_{\Omega}\mathscr{D} \phi \; \mathscr{D}  w_m = a_m^{-1}\int_{\Omega} \phi w_m. \]
    This proves \eqref{eq:weak_pde}.

\paragraph{Step 2: PDE in $\Omega$.}
    Next, for any Euclidian ball $\mathscr{B} \subseteq \Omega$
    and any function $\phi \in C^\infty(\Omega)$ with compact support in $\Omega$,  $w_m$ is a weak solution to the PDE
    \[  \lambda_n \sum_{|\alpha|\leqslant s} \int_{\mathscr{B}}\partial^\alpha \phi\; \partial^\alpha w_m + \mu_n \int_{\mathscr{B}}\mathscr{D} \phi \; \mathscr{D}  w_m = a_m^{-1}\int_{\mathscr{B}} \phi w_m. \]
    Noting that the ball, as a smooth manifold, is already its own map  with the canonical coordinates. The principal symbol \citep[see, e.g., Chapter 2.9][]{taylor2010partial} of this PDE is defined for all $x \in \Omega$ and $\xi \in \mathbb{R}^d$ by \[\sigma(x, \xi) = \lambda_n (-1)^s \sum_{|\alpha|=2s}\xi^{2\alpha} + \mu_n (-1)^s\sum_{|\alpha|=2s} p_\alpha(x)^2 \xi^{2\alpha},\] where  $\xi^{2\alpha} = \prod_{j=1}^d \xi_j^{2\alpha_j}$. Clearly, $|\sigma(x, \xi)| \neq 0$ whenever $\xi \neq 0$. Hence, the symbol function defined by $u \mapsto \sigma(x,\xi) \times u$ is an isomorphism from $\mathbb R$ to $\mathbb R$ whenever $\xi \neq 0$. This is the definition of a general elliptic PDE. Since $\mathscr{B}$ is a smooth manifold with $C^\infty$-boundary and  $p_\alpha \in C^\infty(\bar \Omega)$, the elliptic regularity theorem \citep[][Chapter 5, Theorem 11.1]{taylor2010partial} states that $w_m \in C^\infty(\mathscr{B})$. Therefore, $w_m \in C^\infty(\Omega)$. Overall,
    \[\forall x \in \Omega, \quad  \lambda_n\sum_{|\alpha|\leqslant s}(-1)^{|\alpha|} \partial^{2\alpha} w_m(x) + \mu_n \mathscr{D}^\ast \mathscr{D} w_m(x) = a_m^{-1} w_m(x).\]
    This proves $(i)$.

\paragraph{Step 3: PDE outside $\Omega$.}
    To show the second statement of the proposition, fix $\varepsilon > 0$ such that $d(\Omega, \partial [-2L,2L]^d) > \varepsilon$. Observe that any function $\phi \in C^\infty(]-2L-\varepsilon, 2L+\varepsilon[^d \backslash \bar\Omega)$ with compact support in $]-2L-\varepsilon, 2L+\varepsilon[^d \backslash \bar\Omega$ can be linearly mapped into the function $\tilde \phi(x) = \sum_{k \in (4L\mathbb{Z})^d} \phi(x + k) $ in $H^s_{\mathrm{per}}([-2L,2L]^d)$. This function $\tilde \phi$ is such that, for any $u \in L^2([-2L,2L]^d)$, $\int_{]-2L-\varepsilon, 2L+\varepsilon[^d} \phi u = \int_{[-2L, 2L]^d} \tilde \phi u$.  
    We deduce that, for any ball $\mathscr{B}$ included in $]-2L-\varepsilon, 2L+\varepsilon[^d \backslash \bar\Omega$, for any function $\phi \in C^\infty(]-2L-\varepsilon, 2L+\varepsilon[^d \backslash \bar\Omega)$ with compact support in $]-2L-\varepsilon, 2L+\varepsilon[^d \backslash \bar\Omega$, $w_m$ is a weak solution to the PDE
    \[  \lambda_n \sum_{|\alpha|\leqslant s} \int_{\mathscr{B}}\partial^\alpha \phi\; \partial^\alpha w_m  = \lambda_n \sum_{|\alpha|\leqslant s} \int_{[-2L,2L]^d}\partial^\alpha \tilde \phi\; \partial^\alpha w_m = 0. \]
    This PDE is elliptic and $\mathscr{B}$ is a smooth manifold with $C^\infty$-boundary. Therefore, the elliptic regularity theorem \citep[][Chapter 5, Theorem 11.1]{taylor2010partial} states that $w_m \in C^\infty(\mathscr{B})$. So, $w_m \in C^\infty([-2L, 2L]^d \backslash \bar\Omega)$ and
    \[\forall x \in [-2L, 2L]^d \backslash \bar \Omega,\quad  \lambda_n\sum_{|\alpha|\leqslant s}(-1)^{|\alpha|} \partial^{2\alpha} w_m(x)  = 0.\]
    This proves $(ii)$.
   

\subsection{High regularity in dimension 1}
In this section, we assume that $d=1$, $s \geqslant 1$, $p_\alpha \in C^\infty(\bar \Omega)$, and the domain $\Omega$ is a segment, i.e., $\Omega = [L_1, L_2] \subseteq [-L,L]$ for some $-L \leqslant L_1, L_2 \leqslant L $.
\begin{prop}[Regularity of the eigenfunctions of $C\mathscr{O}_nC$]
The functions $(w_m)_{N\in\mathbb{N}}$ of Theorem \ref{prop:eigenfunction} associated with non-zero eigenvalues satisfy the following properties:
     \begin{itemize}
         \item[$(i)$] $w_m \in C^{s-1}([-2L,2L])$,
         \item[$(ii)$] $w_m|_\Omega \in C^{\infty}(\bar \Omega)$,
         \item[$(iii)$] $w_m|_{\Omega^c} \in C^{\infty}(\bar{\Omega^c})$.
     \end{itemize}
     \label{prop:1dreg}
\end{prop}
\begin{proof}
    Since $d = 1$ and $w_m \in H^s([-2L,2L])$, the Sobolev embedding theorem states that $w_m \in C^{s-1}([-2L,2L])$. Moreover, since  $w_m \in C^\infty(\Omega)$, since 
    \[\mathscr{D}^\ast \mathscr{D} u = \sum_{\alpha= 0}^s p_\alpha \Big(\frac{d}{dt}\Big)^\alpha \Big(\sum_{\tilde \alpha = 0}^s p_{\tilde \alpha} \Big(\frac{d}{dt}\Big)^{\tilde \alpha} u\Big)\] 
    is a linear differential operator with coefficients in $C^\infty(\bar \Omega)$, and since $w_m$ is the solution to the ordinary differential equation
        \[\forall x \in \Omega, \quad  \lambda_n\sum_{j =1}^s(-1)^{j} \frac{d^j}{dt^j} w_m(x) + \mu_n \mathscr{D}^\ast \mathscr{D} w_m(x) = a_m^{-1} w_m(x),\]
        the Picard-Lindelöf theorem (or the Grönwall inequality) ensures that $w_m|_\Omega \in C^\infty(\bar \Omega)$. Similarly, since $w_m \in C^\infty([-2L, 2L]^d \backslash \bar\Omega)$ and 
        \[\forall x \in [-2L, 2L]^d \backslash \bar \Omega,\quad  \sum_{j =1}^s(-1)^{j} \frac{d^j}{dt^j}  w_m  = 0,\]
        we have $w_m|_{\Omega^c} \in C^{\infty}(\bar{\Omega^c})$.
\end{proof}
\begin{rem} 
   As a by-product, the limits $\lim_{\genfrac{}{}{0 pt}{2}{x \to L_1 }{ x > L_1}} w_m(x)$, $\lim_{\genfrac{}{}{0 pt}{2}{x \to L_2}{x < L_2}} w_m(x)$, $\lim_{\genfrac{}{}{0 pt}{2}{x \to L_1 }{ x < L_1}} w_m(x)$, and $\lim_{\genfrac{}{}{0 pt}{2}{x \to L_2 }{ x > L_2}} w_m(x)$ exist.
\end{rem}

\section{From eigenvalues of the integral operator to minimax convergence rates}
\subsection{Effective dimension}
We recall that the effective dimension $\mathscr{N}$ of the kernel $K$ is defined by
\[
\mathscr{N}(\lambda_n, \mu_n) = \mathrm{tr}(L_{K}  (\mathrm{Id} + L_{K})^{-1}),
\]
where $\mathrm{Id}$ is the identity operator and the symbol $\mathrm{tr}$ stands for the trace, i.e., the sum of the eigenvalues \citep{caponnetto2007optimal}. So,
\begin{align*}
    \mathscr{N}(\lambda_n, \mu_n) &= \mathrm{tr}(L_{K} \times (\mathrm{Id} + L_{K})^{-1})\\
    &= \sum_{m\in \mathbb{N}} \frac{ a_m(L_K)}{1+a_m(L_K)}\\
    &= \sum_{m\in \mathbb{N}} \frac{ 1}{1+a_m(L_K)^{-1}},
\end{align*}
where $a_m(L_K)$ stands for the eigenvalues of the operator $L_K$. The second equality is a consequence of the fact that $\mathrm{Id}$ and $L_{K}$ are co-diagonalizable, and so are $\mathrm{Id}$, $L_{K}$, and $(\mathrm{Id} + L_{K})^{-1}$.
\begin{lem}
    Assume that $\frac{d\mathbb{P}}{dx} \leqslant \kappa$. Then 
\begin{equation*}
    \mathscr{N}(\lambda_n, \mu_n) \leqslant \sum_{m \in \mathbb N} \frac{ 1}{1+(\kappa a_m(C\mathscr{O}_nC))^{-1}}.
\end{equation*}
\label{lem:eff_dim}
\end{lem}
\begin{proof}
     Apply Theorem \ref{thm:eigenvalues} and observe that $0 < a_m(L_K) \leqslant \kappa a_m(C\mathscr{O}_nC) \Leftrightarrow a_m(L_K)^{-1} \geqslant  (\kappa a_m(C\mathscr{O}_nC))^{-1} \Leftrightarrow 1+ a_m(L_K)^{-1} \geqslant  1+ (\kappa a_m(C\mathscr{O}_nC)) ^{-1} \Leftrightarrow (1+ a_m(L_K))^{-1} \leqslant  (1+ (\kappa a_m(C\mathscr{O}_nC))^{-1})^{-1}$.
\end{proof}

\subsection{Lower bound on the eigenvalues of the integral kernel}

\begin{lem}[Explicit computation of $\mathscr{O}_n^{-1}$]
    Let $f \in C^\infty(\Omega)$ with compact support in $\Omega$. Then 
    \[\mathscr{O}_n^{-1}(f) = \lambda_n \sum_{|\alpha|\leqslant s} (-1)^{|\alpha|} \partial^{2\alpha} f + \mu_n \mathscr{D}^\ast \mathscr{D} f. \]
    \label{lem:explicit_calculus}
\end{lem}
\begin{proof}
    Let $\phi \in H^s_{\mathrm{per}}([-2L,2L]^d)$ be a test function. Since  the successive derivatives of $f$ are smooth with compact support, by definition of the weak derivatives of $\phi$, we may write
    \begin{align*}
        \lambda_n \sum_{|\alpha|\leqslant s}\int_{[-2L,2L]^d} \partial^{\alpha} f \partial^{\alpha}\phi = \int_{[-2L,2L]^d} \Big(\lambda_n \sum_{|\alpha|\leqslant s} (-1)^{|\alpha|} \partial^{2\alpha} f\Big) \phi.
    \end{align*}
    Moreover, because the support of $f$ is included in $\Omega$, we have that
    \[  \mu_n  \int_{\Omega} \mathscr{D} f \;\mathscr{D} \phi = \mu_n  \int_{[-2L,2L]^d} \mathscr{D} f \;\mathscr{D} \phi = \mu_n\int_{[-2L,2L]^d} ( \mathscr{D}^\ast \mathscr{D} f)\; \phi. \]
    We deduce that $B[f, \phi] = \int_{[-2L,2L]^d} (\lambda_n \sum_{|\alpha|\leqslant s} (-1)^{|\alpha|} \partial^{2\alpha} f + \mu_n \mathscr{D}^\ast \mathscr{D} f) \phi$. Since this identity holds for all $\phi \in H^s_{\mathrm{per}}([-2L,2L]^d)$, and since there is a unique Lax-Milgram inverse satisfying this condition, we conclude that
    \[\lambda_n \sum_{|\alpha|\leqslant s} (-1)^{|\alpha|} \partial^{2\alpha} f + \mu_n \mathscr{D}^\ast \mathscr{D} f = \mathscr{O}_n^{-1}(f). \]
\end{proof}

\begin{lem}[Lower bound on the integral operator norm]
    Assume that \[\lim_{n\to \infty}\lambda_n = \lim_{n\to \infty} \mu_n = \lim_{n\to \infty} \lambda_n / \mu_n = 0.\] Then there is a constant $C_5 > 0$ such that 
    \[\|L_{K}\|_{\mathrm{op}, L^2(\Omega, \mathbb{P}_X)} := \sup_{\|f\|_{L^2(\Omega, \mathbb{P}_X)}=1}{\|L_{K}f\|_{L^2(\Omega, \mathbb{P}_X)}} \geqslant C_5 \mu_n^{-1} \rightarrow \infty.\]
    \label{lem:lowerbound}
\end{lem}
\begin{proof}
    The operator $L_K$ is diagonalizable according to Theorem \ref{thm:eigenvalues}, and thus its operator norm $\sup_{\|f\|_{L^2(\Omega, \mathbb{P}_X)}=1}{\|L_{K}f\|_{L^2(\Omega, \mathbb{P}_X)}}$ is larger than the largest eigenvalue of $L_{K}$. The Courant-Fischer min-max theorem states that 
    this eigenvalue is larger than $\langle f, L_K f\rangle$ for any function $f$ such that $\|f\|_{L^2(\Omega, \mathbb{P}_X)}=1$.
    By the proof of Theorem \ref{thm:eigenvalues}, we know that
    \begin{equation*}
    \langle f, L_K f\rangle_{L^2(\Omega, \mathbb{P}_X)}= \Big\|\mathscr{O}_n^{1/2} \Big( f\frac{d\mathbb{P}_X}{dx}\Big)\Big\|_{L^2([-2L,2L]^d)}^2 = \Big\langle \mathscr{O}_n\Big(f\frac{d\mathbb{P}_X}{dx}\Big), f\frac{d\mathbb{P}_X}{dx} \Big\rangle_{L^2([-2L,2L]^d)}. 
\end{equation*}
Consider a smooth function $g$ with compact support in the set $E = \{z \in [-2L,2L]^d \; |\; \frac{d\mathbb{P}_X}{dx} \geqslant (4L)^{-d} /2\}$. 
Let \begin{equation}
    f = \Big(\lambda_n \sum_{|\alpha|\leqslant s} (-1)^{|\alpha|} \partial^{2\alpha} g + \mu_n \mathscr{D}^\ast \mathscr{D} g\Big) \times \Big(\frac{d\mathbb{P}_X}{dx}\Big)^{-1}.
    \label{eq:def_f}
\end{equation} Since $g$ is smooth and, on $E$, $(\frac{d\mathbb{P}_X}{dx})^{-1} \leqslant 2 (4L)^d$, we deduce that $f \in L^2(\Omega, \mathbb{P}_X)$. According to Lemma \ref{lem:explicit_calculus}, $\mathscr{O}_n(f\frac{d\mathbb{P}_X}{dx}) = g$.
Thus,
\begin{align*}
    \langle f, L_K f\rangle_{L^2(\Omega, \mathbb{P}_X)} &=\Big\langle g, \lambda_n \sum_{|\alpha|\leqslant s} (-1)^{|\alpha|} \partial^{2\alpha} g + \mu_n \mathscr{D}^\ast \mathscr{D} g \Big\rangle_{L^2([-2L,2L]^d)}\\
    &=\lambda_n \|g\|_{H^s([-2L,2L]^d)}^2 +\mu_n \|\mathscr{D} g\|_{L^2(\Omega)}^2.
\end{align*}
Recall that 
\begin{align*}
    \|L_{K}\|_{\mathrm{op}, L^2(\Omega, \mathbb{P}_X)} &\geqslant (\|f\|_{L^2(\Omega, \mathbb{P}_X)})^{-2} \langle f, L_K f\rangle_{L^2(\Omega, \mathbb{P}_X)}.
\end{align*}
On the one hand, if $\mathscr{D}^\ast \mathscr{D} g = 0$, then identity \eqref{eq:def_f} implies that $\|f\|_{L^2(\Omega, \mathbb{P}_X)}^2 = \Theta_{n\to \infty}(\lambda_n^2)$, and thus
\[(\|f\|_{L^2(\Omega, \mathbb{P}_X)})^{-2} \langle f, L_K f\rangle_{L^2(\Omega, \mathbb{P}_X)} = \Theta_{n\to \infty}(\lambda_n^{-1}).\]
On the other hand, if $\mathscr{D}^\ast \mathscr{D} g \neq 0$, since $\mu_n/\lambda_n \to \infty$, \eqref{eq:def_f} implies that $\|f\|_{L^2(\Omega, \mathbb{P}_X)}^2 = \Theta_{n\to \infty}(\mu_n^2)$, and thus 
\[(\|f\|_{L^2(\Omega, \mathbb{P}_X)})^{-2} \langle f, L_K f\rangle_{L^2(\Omega, \mathbb{P}_X)} = \Theta_{n\to \infty}(\mu_n^{-1}).\]
Overall, we conclude that there is a constant $C_5 > 0$, such that 
\[\|L_{K}\|_{\mathrm{op}, L^2(\Omega, \mathbb{P}_X)}  \geqslant C_5 \mu_n^{-1}.\]
\end{proof}

\subsection{Bounds on the convergence rate}
\begin{theorem}[High-probability bound]
\label{thm:speedup}
    Assume that the following four assumptions are satisfied:
    \begin{itemize}
        \item[$(i)$] $\lim_{n\to \infty}\lambda_n = \lim_{n\to \infty} \mu_n = \lim_{n\to \infty} \lambda_n / \mu_n = 0$,
        \item[$(ii)$] $\lambda_n \geqslant n^{-1}$,  
        \item[$(iii)$]  $\mathscr{N}(\lambda_n, \mu_n) \lambda_n^{-1} = o_n (n)$,
        \item[$(iv)$] for some $\sigma > 0$ and $M > 0$, the noise $\varepsilon$ satisfies 
    \[\forall \ell \in \mathbb{N}, \quad \mathbb{E}(|\varepsilon|^\ell\; | \; X) \leqslant \frac{1}{2}\ell !\; \sigma^2\; M^{\ell-2}.\]
    \end{itemize} 
    Then, letting $C_3 = 96 \log(6)$, for $n$ large enough, for all $\eta > 0$, with probability at least $1-\eta$, 
    \begin{align*}
        &\int_\Omega\|\hat f_n(x) - f^\star(x)\|_2^2 d\mathbb{P}_X(x) 
        \\
        & \quad \leqslant C_3 \log^2(\eta)\Big(\lambda_n \|f^\star\|_{H^s_{\mathrm{per}}([-2L, 2L]^d)}^2 + \mu_n \|\mathscr{D}(f^\star)\|_{L^2(\Omega)}^2 + \frac{M^2}{n^2 \lambda_n} + \frac{\sigma^2\mathscr{N}(\lambda_n, \mu_n)}{n}\Big).
    \end{align*}
\end{theorem}
\begin{proof}
Observe that the kernel $K$ of Theorem \ref{thm:PDE_kernel}  depends on $n$ and that the function $f^\star$ belongs to a ball of radius $R_n = (\lambda_n \|f^\star\|_{H^s_{\mathrm{per}}([-2L, 2L]^d)}^2 + \mu_n \|\mathscr{D}(f^\star)\|_{L^2(\Omega)}^2)^{1/2}$.
Consider the non-asymptotic bound of \citet[][Theorem 4]{caponnetto2007optimal} applied to $K$  (that can be interpreted as a regular kernel for the norm $\|f\|_{\mathrm{RKHS}}^2 = \lambda_n \|f\|_{H^s([-2L, 2L]^d)}^2 + \mu_n \|\mathscr{D}(f)\|_{L^2(\Omega)}^2$, with an hyperparameter set to 1). 
    Thus,  we have, with probability at least $1-\eta$,
    \begin{equation}
        \mathcal{E}(\hat f_n) - \mathcal{E}(f^\star) \leqslant 32 \log^2(6\eta^{-1})\Big(\mathcal{A}(1) + \frac{\kappa_n^2 \mathcal{B}(1)}{n^2} + \frac{\kappa_n\mathcal{A}(1)}{n} + \frac{\kappa_n M^2}{n^2} + \frac{\sigma^2 \mathscr{N}(\lambda_n, \mu_n)}{n}\Big),
        \label{eq:caponettoBound}
    \end{equation}
    where
    \begin{itemize}
        \item[$(i)$] $\mathcal{E}(f) = \int_\Omega\|f(x) - y\|_2^2 d\mathbb{P}_{(X,Y)}(x,y)$,
        \item[$(ii)$] $\kappa_n = \sup_{x \in \Omega} K(x,x) \leqslant \lambda_n^{-1}$, according to Proposition \ref{prop:bound_kernel},
        \item[$(iii)$] and $\mathscr{A}(1) \leqslant R_n^2$ and $\mathscr{B}(1) \leqslant R_n^2$ (take $c=1$ and $\lambda=1$ in  \citealp[Proposition 3]{caponnetto2007optimal}). 
    \end{itemize}
    Inequality \eqref{eq:caponettoBound} is true as long as
    \begin{itemize}
        \item[$(i)$] $n \geqslant 64 \log^2(6/\eta) \kappa_n \mathscr{N}(\lambda_n, \mu_n)$, which holds for $n$ large enough since $\kappa_n \mathscr{N}(\lambda_n, \mu_n) = \mathcal{O}_n(\lambda_n^{-1}\mathscr{N}(\lambda_n, \mu_n)) = o_n(n)$ by assumption,
        \item[$(ii)$] $\|L_{K}\|_{\mathrm{op}, L^2(\Omega, \mathbb{P}_X)} \geqslant 1$, which holds for $n$ large enough by Lemma \ref{lem:lowerbound}, because, by assumption, $\lim_{n\to \infty}\lambda_n = \lim_{n\to \infty} \mu_n = \lim_{n\to \infty} \lambda_n / \mu_n = 0$.
    \end{itemize}
    Since $\lambda_n \geqslant n^{-1}$, we deduce that $n^{-1}\kappa_n \leqslant 1$, and so \[\mathcal{A}(1) + \frac{\kappa_n^2 \mathcal{B}(1)}{n^2} + \frac{\kappa_n\mathcal{A}(1)}{n} \leqslant 3 (\lambda_n \|f^\star\|_{H^s_{\mathrm{per}}([-2L, 2L]^d)}^2  + \mu_n \|\mathscr{D}(f^\star)\|_{L^2(\Omega)}^2).\]
    It follows that, letting $C_3 = 96 \log(6)$, for $n$ large enough, for all $\eta > 0$, with probability at least $1-\eta$, 
    \begin{align*}
        &\mathcal{E}(\hat f_n) - \mathcal{E}(f^\star) \\
        & \quad \leqslant C_3 \log^2(\eta)\Big(\lambda_n \|f^\star\|_{H^s_{\mathrm{per}}([-2L, 2L]^d)}^2 + \mu_n \|\mathscr{D}(f^\star)\|_{L^2(\Omega)}^2 + \frac{M^2}{n^2 \lambda_n} + \frac{\sigma^2\mathscr{N}(\lambda_n, \mu_n)}{n}\Big).
    \end{align*}
    The conclusion is then a consequence of the identity $\mathcal{E}(\hat f_n) - \mathcal{E}(f^\star) = \int_\Omega\|\hat f_n(x) - f^\star(x)\|_2^2 d\mathbb{P}_X(x)$.
\end{proof} 

\subsection{Proof of Theorem \ref{thm:boundexp}}
    Note that, for all $f \in H^s_{\mathrm{per}}([-2L,2L]^d)$, $ \lambda_n \|f\|_{L^2(\Omega)}^2 \leqslant R_n(f)$. Since,  $\hat f_n$ is defined as minimizing~$R_n$, we have that
     \[
         \lambda_n \|\hat f_n\|_{L^2(\Omega)}^2 \leqslant  R_n(\hat f_n) \leqslant R_n( f^\star)= \lambda_n \|f^\star\|_{H^s(\Omega)}^2+ \frac{1}{n}\sum_{j=1}^n \|f^\star(X_i) - Y_i\|_2^2.
     \]
     By taking the expectation on these inequalities, we obtained that \[\mathbb{E}\|\hat f_n\|_{L^2(\Omega)}^2 \leqslant \|f^\star\|_{H^s(\Omega)}^2+ \lambda_n^{-1}\mathbb{E}\|\varepsilon\|_2^2.\] 
     We therefore have the following bound on the risk, where the expectation is taken with respect to the distribution of $(\hat f_n, X)$, where $X$ is a random variable independent from $\hat f_n$ with distribution~$\mathbb{P}_X$:
    \begin{align*}
        &\mathbb{E}\|\hat f_n(X) - f^\star(X)\|_2^2 
        \\
        &\leqslant  C_3 \log^2(\eta)\Big(\lambda_n \|f^\star\|_{H^s_{\mathrm{per}}([-2L, 2L]^d)}^2 + \mu_n \|\mathscr{D}(f^\star)\|_{L^2(\Omega)}^2 + \frac{M^2}{n^2 \lambda_n} + \frac{\sigma^2\mathscr{N}(\lambda_n, \mu_n)}{n}\Big)\\
        &\quad + 2\eta (2\|f^\star\|_{H^s(\Omega)}^2+ \lambda_n^{-1}\mathbb{E}\|\varepsilon\|_2^2).
    \end{align*}
    Take 
    $\eta = n^{-2}$, i.e.,  $\log(1/\eta) = 2\log(n)$. Thus, letting $C_4 = 4C_3 = 384 \log(6)$, for $n$ large enough, 
    \begin{align*}
        &\mathbb{E}\|\hat f_n(X) - f^\star(X)\|_2^2\\
        &\quad \leqslant C_4 \log^2(n)\Big(\lambda_n \|f^\star\|_{H^s_{\mathrm{per}}([-2L, 2L]^d)}^2 + \mu_n \|\mathscr{D}(f^\star)\|_{L^2(\Omega)}^2 + \frac{M^2}{n^2 \lambda_n} + \frac{\sigma^2\mathscr{N}(\lambda_n, \mu_n)}{n}\Big)\\
        &\quad \leqslant C_4 C_{s, \Omega}\log^2(n)\Big(\lambda_n \|f^\star\|_{H^s(\Omega)}^2 + \mu_n \|\mathscr{D}(f^\star)\|_{L^2(\Omega)}^2 + \frac{M^2}{n^2 \lambda_n} + \frac{\sigma^2\mathscr{N}(\lambda_n, \mu_n)}{n}\Big),
    \end{align*}
    where $C_{s, \Omega}$ is the constant in the Sobolev extension.

\subsection{Proof of Proposition \ref{prop:minSpeed}}

    According to \citet[Proposition 3]{caponnetto2007optimal}, if
$a_m = \mathcal{O}_m( m^{1/b})$, then  
\begin{equation}
    \label{eq:bornepoly}
    \sum_{m \in \mathbb{N}} \frac{1}{ 1+\lambda_n a_m} = \mathcal{O}_n(\lambda_n^{-b}).
\end{equation}
In particular, Proposition \ref{prop:eigenvalue} implies that
\[\mathscr{N}(\lambda_n, \mu_n)  = \mathcal{O}_n(\lambda_n^{-d/2s}) .\]
Combining this bound with Theorem \ref{thm:speedup} shows that the PDE kernel approaches $f^\star$ at least at the minimax rate on $H^s(\Omega)$, i.e., $n^{-2s/(2s+d)}$ (up to a log-term).

\section{About the choice of regularization}

\subsection{Kernel equivalence}
\begin{lemma}[Minimal Sobolev norm extension]
    Let $s \in \mathbb{N}$. There is an extension $E: H^s(\Omega) \to H^s_{\mathrm{per}}([-2L,2L]^d)$ such that 
    \[E(f) = \mathrm{argmin}_{g\in H^s_{\mathrm{per}}([-2L,2L]^d), \; g|_\Omega = f} \; \|g\|_{H^s_{\mathrm{per}}([-2L,2L]^d)}.\]
    Moreover, $E$ is linear and bounded, which means that $\|f\|_{H^s(\Omega)}$ and $\| E(f)\|_{H^s_{\mathrm{per}}([-2L,2L]^d)}$ are equivalent norms on $H^s(\Omega)$.
    \label{lem:minimal_norm}
\end{lemma}
\begin{proof}
We have already constructed an extension $\tilde E: H^s(\Omega) \to H^s_{\mathrm{per}}([-2L,2L]^d)$ in Proposition~\ref{prop:dec_four_lip}. However, $\tilde E$ does not minimize the Sobolev norm on $\Omega^c$. 
Let $f \in H^s(\Omega)$ and $\mathscr{H}_0 = \{g\in H^s_{\mathrm{per}}([-2L,2L]^d), \; g|_\Omega = 0\}$. Clearly, $(\mathscr{H}_0, \|\cdot\|_{H^s_{\mathrm{per}}([-2L,2L]^d)})$ is a Banach space. One has 
\begin{align*}
    &\min_{g\in H^s_{\mathrm{per}}([-2L,2L]^d), \; g|_\Omega = f} \; \|g\|_{H^s_{\mathrm{per}}([-2L,2L]^d)} \\
    & \quad = \min_{g\in \mathscr{H}_0 } \; \|\tilde E(f) + g\|_{H^s_{\mathrm{per}}([-2L,2L]^d)}\\
    & \quad = \min_{g\in \mathscr{H}_0 } \; \|\tilde E(f) + g\|_{H^s_{\mathrm{per}}([-2L,2L]^d)}^2\\
    &\quad = \min_{g\in \mathscr{H}_0} \; \| g\|_{H^s_{\mathrm{per}}([-2L,2L]^d)}^2  + 2\langle \tilde E(f), g\rangle_{H^s_{\mathrm{per}}([-2L,2L]^d)}.
\end{align*}
The form $\langle \cdot, \cdot\rangle_{H^s_{\mathrm{per}}([-2L,2L]^d)}$ is bilinear, symmetric, continuous, and coercive on $\mathscr{H}_0 \times \mathscr{H}_0$.
Thus, according to the Lax-Milgram theorem \citep[e.g., Corollary 5.8]{brezis2010functional}, there exists a unique element $u(f)$ of $\mathscr{H}_0$ such that, for all $g \in \mathscr{H}_0$,
\begin{equation}
    \langle u(f), g\rangle_{H^s_{\mathrm{per}}([-2L,2L]^d)} = -\langle \tilde{E}(f), g\rangle_{H^s_{\mathrm{per}}([-2L,2L]^d)}.
    \label{eq:minSobExtension}
\end{equation}
Thus, $\langle u(f) + \tilde{E}(f), g\rangle_{H^s_{\mathrm{per}}([-2L,2L]^d)} = 0$.
Moreover, $u(f)$ is the unique minimum of $g \mapsto  \| g\|_{H^s_{\mathrm{per}}([-2L,2L]^d)}^2  + 2\langle \tilde E(f), g\rangle_{H^s_{\mathrm{per}}([-2L,2L]^d)}$. Therefore, $E(f) :=  \tilde{E}(f) + u(f)$ satisfies
 \[E(f)= \mathrm{argmin}_{g\in H^s_{\mathrm{per}}([-2L,2L]^d), \; g|_\Omega = f} \; \|g\|_{H^s_{\mathrm{per}}([-2L,2L]^d)}.\]

Let us now show that the extension $E$ is linear. Let $f_1 \in H^s(\Omega)$, $f_2 \in H^s(\Omega)$, and $\lambda \in \mathbb{R}$. We have shown that, for $g \in \mathscr{H}_0$,
\[\langle u(f_1) + \tilde{E}(f_1), g\rangle_{H^s_{\mathrm{per}}([-2L,2L]^d)} = 0,\]
\[\langle u(f_2) + \tilde{E}(f_2), g\rangle_{H^s_{\mathrm{per}}([-2L,2L]^d)} = 0,\]
\[\mbox{and }\langle u(f_1 + \lambda f_2) + \tilde{E}(f_1 + \lambda f_2), g\rangle_{H^s_{\mathrm{per}}([-2L,2L]^d)} = 0.\]
By subtracting the third identity to the first two ones, and observing that, since $\tilde E$ is linear, $\tilde{E}(f_1 + \lambda f_2) = \tilde{E}(f_1) + \tilde{E}(\lambda f_2)$, we deduce that
\[\langle u(f_1) + \lambda u(f_2) - u(f_1 + \lambda f_2) , g\rangle_{H^s_{\mathrm{per}}([-2L,2L]^d)} = 0.\]
As $u(f) \in \mathscr{H}_0$ for all $f\in H^s(\Omega)$, we deduce that $ u(f_1) + u(f_2) - u(f_1 + f_2) \in \mathscr{H}_0$. Therefore, taking $g = u(f_1) + u(f_2) - u(f_1 + f_2)$, we have 
\[\| u(f_1) + \lambda u(f_2) - u(f_1 + \lambda f_2) \|_{H^s_{\mathrm{per}}([-2L,2L]^d)} = 0,\]
i.e., $ u(f_1 + \lambda f_2) = u(f_1) + \lambda u(f_2) $. Thus, $E$ is linear.

Proposition \ref{prop:dec_four_lip} shows that $\|\tilde E(f)\|_{H^s([-2L,2L]^d)}^2 \leqslant \tilde C_{s, \Omega} \|f\|_{H^s(\Omega)}^2$. Moreover, by definition of~$E$, $\| E(f)\|_{H^s([-2L,2L]^d)}^2 \leqslant \|\tilde E(f)\|_{H^s([-2L,2L]^d)}^2$. Thus, $\| E(f)\|_{H^s([-2L,2L]^d)}^2 \leqslant \tilde C_{s, \Omega} \|f\|_{H^s(\Omega)}^2$, i.e., the extension $E$ is bounded. Clearly, $\| E(f)\|_{H^s([-2L,2L]^d)}^2 \geqslant \|f\|_{H^s(\Omega)}^2$. We conclude that $\|f\|_{H^s(\Omega)}$ and $\| E(f)\|_{H^s([-2L,2L]^d)}$ are equivalent norms.
\end{proof}

\begin{prop}[Kernel equivalence]
    Assume that $s>d/2$. Let $\lambda_n > 0$ and $\mu_n \geqslant 0$.  Let $\langle \cdot,  \cdot\rangle_n$ be inner products associated with kernels on $H^s(\Omega)$. Assume that there exist constants $C_1 > 0$ and $C_2 >0$ such that, for all $n \in \mathbb{N}$ and all $f \in H^s(\Omega)$,
    \[ C_1 (\lambda_n \| f\|_{H^s(\Omega)}^2 +\mu_n \| \mathscr{D}(f)\|_{L^2(\Omega)}^2 )\leqslant \langle f, f\rangle_n \leqslant C_2 (\lambda_n \| f\|_{H^s(\Omega)}^2 + \mu_n \| \mathscr{D}(f)\|_{L^2(\Omega)}^2).\]
    Then the kernels associated with $\langle \cdot, \cdot\rangle_n$ on $H^s(\Omega)$ have the same convergence rate as the kernel of Theorem \ref{thm:PDE_kernel} associated with the $\lambda_n \| f\|_{H^s_{\mathrm{per}}([-2L,2L]^d)}^2 + \mu_n \| \mathscr{D}(f)\|_{L^2(\Omega)}^2$ norm.
    \label{prop:kernel_eq}
\end{prop}
\begin{proof}
For clarity, the proof is divided into four steps. 

\paragraph{Step1: From $H^s(\Omega)$ to $H^s_{\mathrm{per}}([-2L,2L]^d)$.}
    Observe that 
\begin{align*}
    \hat f_n &= \mathrm{argmin}_{f \in H^s_{\mathrm{per}}([-2L,2L]^d)} \sum_{i=1}^n |f(X_i) -Y_i|^2  + \lambda_n  \| f\|_{H^s_{\mathrm{per}}([-2L,2L]^d)}^2 + \mu_n \| \mathscr{D}(f)\|_{L^2(\Omega)}^2\\
    &= E\Big(\mathrm{argmin}_{f \in H^s(\Omega)} \sum_{i=1}^n |f(X_i) -Y_i|^2  + \lambda_n  \| E(f)\|_{H^s_{\mathrm{per}}([-2L,2L]^d)}^2 + \mu_n \| \mathscr{D}(f)\|_{L^2(\Omega)}^2\Big),
\end{align*}
where $E(f)$ is the extension $H^s(\Omega) \to H^s_{\mathrm{per}}([-2L,2L]^d)$ with minimal $H^s_{\mathrm{per}}([-2L,2L]^d)$ norm (see Lemma \ref{lem:minimal_norm}).
Define 
\begin{align*}
    \hat f_n^{(5)} &= \mathrm{argmin}_{f \in H^s(\Omega)} \sum_{i=1}^n |f(X_i) -Y_i|^2  + \lambda_n  \| E(f)\|_{H^s_{\mathrm{per}}([-2L,2L]^d)}^2 + \mu_n \| \mathscr{D}(f)\|_{L^2(\Omega)}^2.
\end{align*}
Then $\hat f_n = E(\hat f_n^{(5)})$, which means that for all $x \in \Omega$, $\hat f_n(x) = \hat f_n^{(5)}(x)$. Thus, $\hat f_n$ and $\hat f_n^{(5)}$ have the same convergence rate to $u^\star$. 

\paragraph{Step 2: inner products equivalence.}
Lemma \ref{lem:minimal_norm} states that  $\|f\|_{H^s(\Omega)}$ and $\| E(f)\|_{H^s_{\mathrm{per}}([-2L,2L]^d)}$ are equivalent norms on $H^s(\Omega)$. Therefore, there are constants $C_3$ and $C_4$ such that
\begin{align*}
&C_3 (\lambda_n \| E(f)\|_{H^s_{\mathrm{per}}([-2L,2L]^d)}^2 + \mu_n \| \mathscr{D}(f)\|_{L^2(\Omega)}^2) \\
&\quad \leqslant \|f\|_n^2 \leqslant C_4 ( \lambda_n E(f)\|_{H^s_{\mathrm{per}}([-2L,2L]^d)}^2 + \mu_n \| \mathscr{D}(f)\|_{L^2(\Omega)}^2).
\end{align*}
This shows that the function $\langle \cdot, \cdot\rangle_n: H^s(\Omega) \times H^s(\Omega) \to \mathbb{R}$ is coercive with respect to the $(\lambda_n \| E(f)\|_{H^s_{\mathrm{per}}([-2L,2L]^d)}^2+ \mu_n \| \mathscr{D}(f)\|_{L^2(\Omega)}^2)$ norm. By the Cauchy-Schwarz inequality, $\langle \cdot, \cdot\rangle_n$ is continuous with respect to the same norm. Set $\langle f, g \rangle_n^{\mathrm{per}} =  \lambda_n\sum_{|\alpha|\leqslant s} \int_{[-2L,2L]^d} \partial^\alpha E(f) \partial^\alpha E(g) + \mu_n \int_\Omega \mathscr{D}(f)\;\mathscr{D}(g)$. 
Thus, by the Lax-Milgram theorem, there exists a linear operator $\mathscr{O}: H^s(\Omega) \to H^s(\Omega)$ such that, for all $f$, $g\in H^s(\Omega)$, 
\begin{equation}
    \langle \mathscr{O} f,g\rangle_n = \langle f, g\rangle_n^{\mathrm{per}}.
    \label{eq:transition_op}
\end{equation}
Since \[C_3 (\|\mathscr{O} f\|_n^{\mathrm{per}})^2 \leqslant \|\mathscr{O} f\|_n^2 = \langle \mathscr{O} f, f\rangle_n^{\mathrm{per}} \leqslant \|\mathscr{O} f\|_n^{\mathrm{per}} \|f\|_n^{\mathrm{per}},\]
we deduce that $\|\mathscr{O} f\|_n^{\mathrm{per}} \leqslant C_3^{-1} \| f\|_n^{\mathrm{per}}$. Similarly, the coercivity and continuity of $\langle \cdot, \cdot \rangle_n^{\mathrm{per}}$ with respect to $\langle \cdot, \cdot \rangle_n$ shows that $\|\mathscr{O}^{-1} f\|_n \leqslant C_4 \| f\|_n$, so that $\|\mathscr{O}^{-1} f\|_n^{\mathrm{per}} \leqslant C_3^{-1} C_4^2 \| f\|_n^{\mathrm{per}}$. All in all, \[ C_3 C_4^{-2} \| f\|_n^{\mathrm{per}} \leqslant \|\mathscr{O} f\|_n^{\mathrm{per}} \leqslant C_3^{-1} \| f\|_n^{\mathrm{per}}.\]
One easily verifies that $\mathscr{O}$ is self-adjoint.

\paragraph{Step 3: Link between kernels.} 
Let $f \in H^s(\Omega)$. Remember that, for all $x \in \Omega$, $K(x, \cdot) = \mathscr{O}_n(\delta_x)$ satisfies a weak formulation consistent with the weak formulation of the minimal-Sobolev norm extension in \eqref{eq:minSobExtension}. Thus, $E(K(x, \cdot)) = K(x, \cdot)$, and according to Theorem \ref{thm:PDE_kernel}, we have $f(x) = \langle f, K(x,\cdot) \rangle_n^{\mathrm{per}}$. In this proof, to distinguish between kernels, we denote the associated kernel by $K_n^{\mathrm{per}}(x,y) := K(x,y)$. 
Using the spectral theorem for bounded operators, we have that~$\mathscr{O}^{-1}$ admits a square root $\mathscr{O}^{-1/2}$ which is self-adjoint for the $\langle \cdot, \cdot \rangle_n^{\mathrm{per}}$ inner product.
Therefore, using \eqref{eq:transition_op}, we know that, for all $x \in \Omega$, $f(x) = \langle {\mathscr{O}}^{-1/2}(f), {\mathscr{O}}^{1/2}K(x, \cdot) \rangle_n^{\mathrm{per}}$. Since $\|{\mathscr{O}}^{-1/2}(f)\|_n^{\mathrm{per}} = \|f\|_n$, we deduce that $H^s(\Omega)$ is also a kernel space for the $\|\cdot\|_n$ norm, with kernel $K_n(x,y) = \langle {\mathscr{O}}(K(x, \cdot)), K(y, \cdot)\rangle_n^{\mathrm{per}}$.

\paragraph{Step 4: Eigenvalues of the integral operator.} Define the integral operators $L_n$ and $L^{\mathrm{per}}_n$ on $L^2(\Omega, \mathbb{P}_X)$ by 
\[L_n^{\mathrm{per}}(f) : x \mapsto \int_\Omega K_n^{\mathrm{per}}(x,y) f(y) d\mathbb{P}_X(y) \quad \mbox{and} \quad L_n(f) : x \mapsto \int_\Omega K_n(x,y) f(y) d\mathbb{P}_X(y).\]
Recalling that $K_n^{\mathrm{per}}(x, y) = \sum_{m\in \mathbb{N}} a_m v_m(x) v_m(y)$, we can use the same technique as in the proof of Theorem \ref{thm:eigenvalues} to apply the Fubini-Lebesgue theorem, and show that 
\[\langle f, L_n(f)\rangle_{L^2(\Omega, \mathbb{P}_X)} = (\|\mathscr{O}^{1/2} \mathscr{O}_n(f)\|_n^{\mathrm{per}})^2.\]
Thus, $C_3 C_4^{-2} \langle f, L_n^{\mathrm{per}}(f)\rangle \leqslant \langle f, L_n(f)\rangle \leqslant C_3^{-1}  \langle f, L_n^{\mathrm{per}}(f)\rangle$. The Courant-Fischer min-max theorem guarantees that the eigenvalues of $L_n^{\mathrm{per}}$ are upper and lower bounded by those of $L_n$. In particular, the effective dimensions $\mathscr{N}(\lambda_n, \mu_n)$ related to $\|\cdot\|_n^{\mathrm{per}}$ and $\mathscr{N}^{\mathrm{per}}(\lambda_n, \mu_n)$ satisfy
\[ C_3 C_4^{-2}\mathscr{N}^{\mathrm{per}}(\lambda_n, \mu_n) \leqslant \mathscr{N}(\lambda_n, \mu_n) \leqslant C_3^{-1}\mathscr{N}^{\mathrm{per}}(\lambda_n, \mu_n).\]
This implies that both kernels have equivalent effective dimensions.
\end{proof}

\subsection{Proof of Theorem \ref{thm:eq_reg}}

Proposition \ref{prop:kernel_eq} ensures that $\hat f_n^{(1)}$ and $\hat f_n^{(2)}$ converge at the same rate. 
If $\|\cdot\|$ and $\|\cdot\|_{H^s(\Omega)}$ are equivalent, then there are constants $0 <C_1 < 1$, $C_2 > 1$ such that, for all $f \in H^s(\Omega)$, $C_1 \|f\|_{H^s(\Omega)}^2 \leqslant \|f\|_2^2 \leqslant C_2 \|f\|_{H^s(\Omega)}^2$. Thus, $C_1 (\mu_n \|\mathscr{D}(f)\|_{L^2(\Omega)}^2 + \lambda_n \|f\|_{H^s(\Omega)}^2) \leqslant \mu_n \|\mathscr{D}(f)\|_{L^2(\Omega)}^2 + \lambda_n \|f\| \leqslant C_2 (\mu_n \|\mathscr{D}(f)\|_{L^2(\Omega)}^2 + \lambda_n \|f\|_{H^s(\Omega)}^2).$ Proposition \ref{prop:kernel_eq} then shows that $\hat f_n^{(2)}$ and $\hat f_n^{(3)}$ converge at the same rate.

\section{Application: the case $\mathscr{D} = \frac{d}{dx}$}
\subsection{Boundary conditions}
\begin{prop}
    Let $s = 1$, $\Omega = [-L,L]$, and $\mathscr{D} = \frac{d}{dx}$. Then any weak solution $w_m$ of the weak formulation \eqref{eq:weak_pde} satisfies  
    \begin{align*}
             (\lambda_n+\mu_n)\lim_{x \to -L, x > -L}\frac{d}{dx}w_m(x) &= \lambda_n \lim_{x \to -L, x < -L}\frac{d}{dx}w_m(x),\\
            (\lambda_n+\mu_n)\lim_{x \to L, x < L}\frac{d}{dx}w_m(x) &= \lambda_n \lim_{x \to L, x >L}\frac{d}{dx}w_m(x).
    \end{align*}
    \label{prop:raccord1d}
\end{prop}
\begin{proof}
    The proof uses the framework of distribution theory. By the inclusion $C^\infty([-2L,2L]) \subseteq H^s_{\mathrm{per}}([-2L,2L])$, we know that, considering any test function $\phi \in C^\infty([-2L,2L])$ with compact support in $]-2L,2L[$, one has $B[w_m, \phi] = a_m^{-1} \langle w_m \mathbf{1}_\Omega, \phi\rangle$. 
     Moreover, standard results of functional analysis (using the mollification of $y \mapsto \mathbf{1}_{|y-x| < 3\varepsilon/2}$ with a parameter $\eta = \varepsilon /8$ as in \citealt[Appendix C, Theorem 6]{evans2010partial}) ensures that, for any $x \in [-2L,2L]$, there exists a sequence of functions $(\xi^{x}_\varepsilon)_{\varepsilon > 0}$ such that, for all $m$,
    \begin{itemize}
        \item[$(i)$] $\xi^{x}_\varepsilon \in C^\infty([-2L,2L])$ with compact support in $D$, 
        \item[$(ii)$] $\|\xi^{x}_\varepsilon\|_\infty = 1$,
        \item[$(iii)$] and for all $y \in [-2L,2L]$,
        \begin{align*}
            |y-x| \geqslant 2\varepsilon &\Rightarrow \xi^{x}_\varepsilon(y) = 0\\
            |y-x| \leqslant \varepsilon\;\; &\Rightarrow \xi^{x}_\varepsilon(y) = 1.
        \end{align*}
    \end{itemize}
    Fix two of such sequences with $x = -L$ and $x=L$, and let $\phi_\varepsilon = \phi \times (\xi^{-L}_\varepsilon + \xi^{L}_\varepsilon)$. Notice that the following is true: 
    \begin{itemize}
        \item[$(i)$] $\phi_\varepsilon \in C^\infty([-2L,2L])$ has compact support and $\mathrm{supp}(\phi_\varepsilon) \subseteq \mathrm{supp}(\phi)$,
        \item[$(ii)$]  for all $r \geqslant 0$, $\frac{d^r}{dx^r}\phi_\varepsilon(-L) = \frac{d^r}{dx^r}\phi(-L)$, 
        \item[$(iii)$] for any function $f \in L^2([-2L,2L])$, $\lim_{\varepsilon \to 0} \langle f, \phi_\varepsilon\rangle = 0$,
        \item[$(iv)$] and $B[w_m, \phi_\varepsilon] = a_m^{-1} \langle w_m \mathbf{1}_\Omega, \phi_\varepsilon\rangle$.
    \end{itemize}
    Choose $\mathrm{supp}(\phi) \subseteq [-3L/2, -L/2]$. Clearly, $\int_{-2L}^{2L} (\frac{d}{dx}w_m) (\frac{d}{dx} \phi_\varepsilon) = \int_{-3L/2}^{-L} (\frac{d}{dx}w_m) (\frac{d}{dx} \phi_\varepsilon) + \int_{-L}^{-L/2} (\frac{d}{dx}w_m) (\frac{d}{dx} \phi_\varepsilon)$. The integration by parts formula implies 
    \begin{align*}
        \int_{-3L/2}^{-L} \Big(\frac{d}{dx}w_m\Big) \Big(\frac{d}{dx} \phi_\varepsilon\Big) &=  -\int_{-3L/2}^{-L} \Big(\frac{d^2}{dx^2}w_m\Big) \phi_\varepsilon + \lim_{x \to -L, x < -L} \phi_\varepsilon(x) \frac{d}{dx}w_m(x) \\
        &\xrightarrow{ \varepsilon \to 0 } \lim_{x \to -L, x < -L} \phi_\varepsilon(x) \frac{d}{dx}w_m(x).
    \end{align*}
    Similarly, 
    \begin{align*}
        \int_{-L}^{-L/2} \Big(\frac{d}{dx}w_m\Big)\Big (\frac{d}{dx} \phi_\varepsilon\Big) &=  -\int_{-L}^{-L/2} \Big(\frac{d^2}{dx^2}w_m\Big) \phi_\varepsilon - \lim_{x \to -L, x > -L} \phi_\varepsilon(x) \frac{d}{dx}w_m(x) \\
        &\xrightarrow{ \varepsilon \to 0 } -\lim_{x \to -L, x > -L} \phi_\varepsilon(x) \frac{d}{dx}w_m(x).
    \end{align*}
    Note that  $\lim_{x \to -L, x < -L} \phi_\varepsilon(x) = \lim_{x \to -L, x > -L} \phi_\varepsilon(x)  = \phi(-L)$. Therefore,
    \begin{equation*}
        \lim_{\varepsilon \to 0}\int_{-2L}^{2L} \Big(\frac{d}{dx}w_m\Big) \Big(\frac{d}{dx} \phi_\varepsilon\Big) 
        = \phi(-L)\Big(\lim_{x \to -L, x < -L}\frac{d}{dx}w_m(x) -\lim_{x \to -L, x > -L}\frac{d}{dx}w_m(x)\Big).
    \end{equation*}
    This means that the integral $\int_{-2L}^{2L} (\frac{d}{dx}w_m) (\frac{d}{dx} \phi_\varepsilon)$ quantifies the discontinuity in the derivative of $w_m$ at $-L$.
    Thus, since 
    \[B[w_m, \phi_\varepsilon] = a_m^{-1} \int_{-L}^L w_m\phi_\varepsilon,\]
    we obtain, letting $\varepsilon \to 0$ that
    \[(\lambda_n+\mu_n)\lim_{x \to -L, x > -L}\frac{d}{dx}w_m(x) = \lambda_n \lim_{x \to -L, x < -L}\frac{d}{dx}w_m(x).\]
    The same analysis holds in a neighborhood of $L$, and leads to
    \[(\lambda_n+\mu_n)\lim_{x \to L, x < L}\frac{d}{dx}w_m(x) = \lambda_n \lim_{x \to L, x >L}\frac{d}{dx}w_m(x).\]
\end{proof}

\subsection{Proof of Proposition \ref{prop:1dkernel}}
Combining Theorem \ref{thm:eq_reg} and  Proposition \ref{prop:kernel_eq}, we know that
\[
        \hat f_n^{(1)} = \mathrm{argmin}_{f \in H^1([-L,L])} \sum_{i=1}^n |f(X_i) -Y_i|^2 + \lambda_n \| f\|_{H^1([-L,L])}^2 + \mu_n \| \mathscr{D}(f)\|_{L^2([-L,L])}^2 
\]
and
\[
 \hat f_n^{(2)} = \mathrm{argmin}_{f \in H^1_{\mathrm{per}}([-2L,2L])} \sum_{i=1}^n |f(X_i) -Y_i|^2 + \lambda_n  \| f\|_{H^1_{\mathrm{per}}([-2L,2L])}^2 + \mu_n \| \mathscr{D}(f)\|_{L^2([-L,L])}^2,
\]
converge at the same rate to $f^\star$. Moreover, the $\lambda_n \| f\|_{H^1([-L,L])}^2 + \mu_n \| \mathscr{D}(f)\|_{L^2([-L,L])}^2 $ norm on $H^1([-L,L])$ defines a kernel. This is this particular kernel, denoted by $K$, that we compute in the remaining of the proof. Employing the exact same arguments as for the kernel on $H^1_{\mathrm{per}}([-2L,2L])$, we know that for all $x \in [-L,L]$, the function $f_x: y \mapsto K(x,y) \in H^1([-L,L])$ is a solution to the weak PDE
\[\forall \phi \in H^1([-L,L]), \quad \lambda_n \int_{[-L,L]} f_x \phi + (\lambda_n + \mu_n)\int_{[-L,L]} \frac{d}{dy}f_x \frac{d}{dy}\phi = \phi(x).\]
Using the elliptic regularity theorem as in the proof of Theorem \ref{prop:eigenfunction} and computing the boundary conditions as in Proposition \ref{prop:raccord1d} shows that
$f_x \in C^\infty([-L, x]) \cap C^\infty([x, L])$, $\frac{d}{dy}f_x(-L) = \frac{d}{dy}f_x(L) = 0$, and 
\[\lambda_n  f_x  - (\lambda_n + \mu_n)\frac{d^2}{dy^2}f_x = \delta_x,\]
where $\delta_x$ is the Dirac distribution.
Thus, since $f_x \in H^1([-L,L]) \subseteq C^0([-L,L])$, there are constants $A$ and $B$ such that
\begin{equation}
        \left\{ \begin{array}{cc}
             \forall -L \leqslant y \leqslant x,    &f_x(y) = A \cosh(\gamma_n (x-y)) + B \sinh(\gamma_n (x-y)) ,\\
             \forall x \leqslant y \leqslant L,    &f_x(y) = A \cosh(\gamma_n (x-y)) + (B + \frac{\gamma_n}{\lambda_n}) \sinh(\gamma_n (x-y)).
        \end{array}\right.
        \label{eq:kernelForm}
        \end{equation}
The boundary conditions   $\frac{d}{dy}f_x(-L) = \frac{d}{dy}f_x(L) = 0$ lead to
\[P \begin{pmatrix}
    A \\ B
\end{pmatrix} =   \begin{pmatrix}
    0 \\ - \frac{\gamma_n}{\lambda_n} \cosh(\gamma_n(x-L))
\end{pmatrix}, \]
where \[P = \begin{pmatrix}
    \sinh(\gamma_n(x+L)) & \cosh(\gamma_n(x+L))\\
    \sinh(\gamma_n(x-L)) & \cosh(\gamma_n(x-L))
\end{pmatrix}.\] Notice that $\det P = \sinh(\gamma_n(x+L)) \cosh(\gamma_n(x-L)) - \sinh(\gamma_n(x-L))\cosh(\gamma_n(x+L)) =  \sinh(2 \gamma_n L)$. Thus, 
\[P^{-1} = \sinh(2 \gamma_n L)^{-1} \begin{pmatrix}
    \cosh(\gamma_n(x-L)) & -\cosh(\gamma_n(x+L))\\
    -\sinh(\gamma_n(x-L)) & \sinh(\gamma_n(x+L))
\end{pmatrix}.\]
This leads to 
\begin{align}
    \begin{pmatrix}
    A \\ B
\end{pmatrix} &= \frac{\gamma_n}{\lambda_n\sinh(2 \gamma_n L)} \begin{pmatrix}
    \cosh(\gamma_n(x+L)) \cosh(\gamma_n(x-L)) \\ -\sinh(\gamma_n(x+L))\cosh(\gamma_n(x-L))
\end{pmatrix} \nonumber\\
&= \frac{\gamma_n}{2\lambda_n\sinh(2 \gamma_n L)} \begin{pmatrix}
    \cosh(2\gamma_nL) + \cosh(2\gamma_nx) \\ \sinh(2\gamma_nL))-\sinh(2\gamma_nx)
\end{pmatrix}. \label{eq:invMatrix}
\end{align}  
Combining \eqref{eq:kernelForm} and \eqref{eq:invMatrix}, we are led to
\begin{align*}
        K(x,y)& = \frac{\gamma_n}{2\lambda_n \sinh(2 \gamma_n L)}\Big( (\cosh(2\gamma_n L)+\cosh(2\gamma_n x))\cosh(\gamma_n (x-y))\\
        &\qquad  + ((1-2 \times \mathbf{1}_{x > y})\sinh(2\gamma_n L) - \sinh(2 \gamma_n x)) \sinh(\gamma_n(x-y))\Big).
    \end{align*}
One easily checks that $K(x,y) = K(y,x)$ and that $K(x,x) \geqslant 0$.

\subsection{Proof of Proposition \ref{prop:1ddiff}}
 The strategy of the proof is to characterize the solutions $w_m$ to the weak formulation \eqref{eq:weak_pde} with $\mathscr{D} = \frac{d}{dt}$ and $s = 1 > d/2 = 1/2$. For clarity, the proof is divided into 5 steps.

\paragraph{Step 1: Symmetry.} Recall that $\Omega = [-L, L]$. Using the Lax-Milgram theorem, let us define the operator $\tilde{\mathscr{O}}_n$ as follows. For all $f \in L^2([-2L,2L])$, $\tilde{\mathscr{O}}_n(f)$ is  
the unique function of $H^2_{\mathrm{per}}([-2L,2L])$ such that, for all $\phi \in H^2_{\mathrm{per}}([-2L,2L])$, $B[\tilde{\mathscr{O}}_n(f), \phi] = \langle C f, C \phi \rangle$. Clearly, the eigenfunctions of $\tilde{\mathscr{O}}_n$ associated to non-zero eigenvalues are the $w_m$. Let $\phi \in H^2_{\mathrm{per}}([-2L,2L])$ be a test function. Using \begin{align*}
        \int_{-2L}^{2L} \partial^\alpha \phi(-\cdot)(x) \partial^\alpha \tilde{\mathscr{O}}_n(f)(-\cdot)(x)dx &= (-1)^{2\alpha}\int_{-2L}^{2L} \partial^\alpha \phi(-x) \partial^\alpha \tilde{\mathscr{O}}_n(f)(-x)dx\\
        & = -\int_{-2L}^{2L} \partial^\alpha \phi(x) \partial^\alpha \tilde{\mathscr{O}}_n(f)(x)dx,
    \end{align*} 
    we see that $B[\tilde{\mathscr{O}}_n(f)(-\cdot), \phi(-\cdot)] = \langle Cf(-\cdot), C \phi(-\cdot)\rangle $. Therefore, since $H^2_{\mathrm{per}}([-2L,2L])$ is stable by the action $\phi \mapsto \phi(-\cdot)$, using the uniqueness statement provided by the Lax-Milgram theorem, we deduce that $\tilde{\mathscr{O}}_n(f)(-x) = \tilde{\mathscr{O}}_n(f(-\cdot))(x)$, so that $\tilde{\mathscr{O}}_n(f)$ is symmetric. According to Proposition \ref{prop:sym}, we can therefore assume that $w_m$ is either symmetric or antisymmetric. 

\paragraph{Step 2: PDE system.}
According to Theorem \ref{prop:eigenfunction}, \ref{prop:1dreg}, and \ref{prop:raccord1d}, the following statements are verified: 
 \begin{itemize}
        \item[$(i)$] The function $w_m \in C^\infty([-L,L])$ and 
        \[\forall x \in \Omega, \quad \lambda_n\Big(1- \frac{d^2}{dx^2}\Big)w_m(x)  - \mu_n \frac{d^2}{dx^2} w_m(x) = a_m^{-1} w_m(x).\]
        Since $a_m^{-1} \geqslant \lambda_n$ (see Remark \ref{rem:rayleigh}), the solutions of this ODE are  linear combinations of $\cos(\sqrt{\frac{a_m^{-1}-\lambda_n}{\lambda_n + \mu_n}}x)$ and $\sin(\sqrt{\frac{a_m^{-1}-\lambda_n}{\lambda_n + \mu_n}}x)$.
        \item[$(ii)$] The function $w_m \in C^\infty([-2L, 2L] \backslash [-L,L])$, with a $C^\infty$ junction condition at $-2L$, and 
        \[\forall x \in [-2L, 2L]^d \backslash \bar \Omega,\quad \Big(1- \frac{d^2}{dx^2}\Big)w_m(x)  = 0.\]
        The solutions of this ODE are linear combinations of $\cosh(x)$ and $\sinh(x)$. The $C^\infty$ $4L$-periodic junction condition at $-2L$ guarantees that there are two constants $A$ and $B$ such that 
\begin{align*}
             \forall -2L \leqslant x \leqslant -L,    &w_m(x) = A \cosh(x+2L) + B \sinh(x+2L),\\
              \forall L \leqslant x \leqslant 2L,   &w_m(x) = A \cosh(x-2L) + B \sinh(x-2L).
        \end{align*}
        \item[$(iii)$] The function $w_m \in C^0_{\mathrm{per}}([-2L,2L])$.
        \item[$(iv)$] One has
        \begin{align*}
        (\lambda_n+\mu_n)\lim_{x \to -L, x > -L}\frac{d}{dx}w_m(x) &= \lambda_n \lim_{x \to -L, x < -L}\frac{d}{dx}w_m(x),\\
              (\lambda_n+\mu_n)\lim_{x \to L, x < L}\frac{d}{dx}w_m(x) &= \lambda_n \lim_{x \to L, x >L}\frac{d}{dx}w_m(x).
    \end{align*}
        \item[$(v)$] One has $\int_{-2L}^{2L} w_m^2 = 1$.
    \end{itemize}

\paragraph{Step 3: Symmetric eigenfunctions.} Our goal in this paragraph is to describe the symmetric eigenfunctions, i.e., $w_m(-x) = w_m(x)$. We denote by $a_m^{\mathrm{sym}}$ the eigenvalues of such eigenfunctions. From statements $(i)$ and $(ii)$ above, we deduce that there are two constant $A$ and $C$ such that
\begin{align*}
        \forall -2L \leqslant x \leqslant -L,\;    &w_m(x) = A \cosh(x+2L) ,\\
             \forall -L \leqslant x \leqslant L,\;    &w_m(x) = C\cos\Big(\sqrt{\frac{a_m^{-1}-\lambda_n}{\lambda_n + \mu_n}}x\Big),\\
              \forall L \leqslant x \leqslant 2L,\;    &w_m(x) = A \cosh(x-2L).
        \end{align*}
Applying $(iii)$ at $x = -L$ leads to
\begin{equation}
    A \cosh(L) = C \cos\Big(\sqrt{\frac{(a_m^{\mathrm{sym}})^{-1}-\lambda_n}{\lambda_n + \mu_n}}L\Big).
    \label{eq:continuite}
\end{equation} 
Similarly, statement $(iv)$ applied at $x = -L$ shows that
\begin{equation}
    \lambda_n  A \sinh(L) = - (\lambda_n  + \mu_n) C \sqrt{\frac{(a_m^{\mathrm{sym}})^{-1}-\lambda_n}{\lambda_n + \mu_n}} \sin\Big(-\sqrt{\frac{(a_m^{\mathrm{sym}})^{-1}-\lambda_n}{\lambda_n + \mu_n}}L\Big).
    \label{eq:derivation}
\end{equation}
Dividing \eqref{eq:derivation} by \eqref{eq:continuite} leads to
\[L \sqrt{\frac{(a_m^{\mathrm{sym}})^{-1}-\lambda_n}{\lambda_n + \mu_n}} \tan\Big(\sqrt{\frac{(a_m^{\mathrm{sym}})^{-1}-\lambda_n}{\lambda_n + \mu_n}}L\Big) = L \frac{\lambda_n}{\lambda_n + \mu_n}   \tanh(L).\]
The equation $x \tan(x) = \tilde C$, where $\tilde C$ is constant, has exactly one solution in any interval $[\pi(k - 1/2), \pi(k+1/2)]$ for $k \in \mathbb{Z}$. Therefore, there is only one admissible value of $\sqrt{\frac{(a_m^{\mathrm{sym}})^{-1}-\lambda_n}{\lambda_n + \mu_n}}$ in each of these interval. So,
\[ \lambda_n + (\lambda_n + \mu_n)(m-1/2)^2 \pi^2/L^2 \leqslant (a_m^{\mathrm{sym}})^{-1} \leqslant \lambda_n + (\lambda_n + \mu_n)(m+1/2)^2 \pi^2/L^2.\]

\paragraph{Step 4: Antisymmetric eigenfunctions.} Our goal in this paragraph is to describe the antisymmetric eigenfunctions, i.e., $w_m(-x) = -w_m(x)$. We denote by $a_m^{\mathrm{anti}}$ the eigenvalues of such eigenfunctions. From statements $(i)$ and $(ii)$, we deduce that there are two constant $B$ and $D$ such that
\begin{equation*}
        \left\{ \begin{array}{cc}
             \forall -2L \leqslant x \leqslant -L,    &w_m(x) = B \sinh(x+2L) ,\\
             \forall -L \leqslant x \leqslant L,    &w_m(x) = D\sin\Big(\sqrt{\frac{a_m^{-1}-\lambda_n}{\lambda_n + \mu_n}}x\Big),\\
              \forall L \leqslant x \leqslant 2L,    &w_m(x) = B \sinh(x-2L).
        \end{array}\right.
        \end{equation*}
Applying $(iii)$ at $x = -L$, one has
\begin{equation}
    B \sinh(L) = D \sin\Big(\sqrt{\frac{(a_m^{\mathrm{anti}})^{-1}-\lambda_n}{\lambda_n + \mu_n}}L\Big).
    \label{eq:continuite2}
\end{equation} 
Similarly, applying $(iv)$ at $x = -L$ shows that
\begin{equation}
    \lambda_n  B \cosh(L) =  (\lambda_n  + \mu_n) D \sqrt{\frac{(a_m^{\mathrm{anti}})^{-1}-\lambda_n}{\lambda_n + \mu_n}} \cos\Big(-\sqrt{\frac{(a_m^{\mathrm{anti}})^{-1}-\lambda_n}{\lambda_n + \mu_n}}L\Big).
    \label{eq:derivation2}
\end{equation}
Dividing  \eqref{eq:continuite2} by \eqref{eq:derivation2} leads to
\[L \Big({\frac{(a_m^{\mathrm{anti}})^{-1}-\lambda_n}{\lambda_n + \mu_n}}\Big)^{-1/2} \tan\Big(\sqrt{\frac{(a_m^{\mathrm{anti}})^{-1}-\lambda_n}{\lambda_n + \mu_n}}L\Big) = L (1+\frac{\mu}{\lambda_n})   \tanh(L).\]
The equation $\tan(x)/x = \tilde C$, where $\tilde C$ is constant, has exactly one solution in any interval $[\pi(k - 1/2), \pi(k+1/2)]$ for $k \in \mathbb{Z}$. Therefore, there is only one admissible value of $\sqrt{\frac{(a_m^{\mathrm{anti}})^{-1}-\lambda_n}{\lambda_n + \mu_n}}$ in each of these interval. So,
\[ \lambda_n + (\lambda_n + \mu_n)(m-1/2)^2 \pi^2/L^2 \leqslant (a_m^{\mathrm{anti}})^{-1} \leqslant \lambda_n + (\lambda_n + \mu_n)(m+1/2)^2 \pi^2/L^2.\]
\paragraph{Step 5: Conclusion.} Recall that the sequence 
$(a_m)_{m\in\mathbb{N}}$ is a non-increasing re-indexing of the sequences $(a_m^{\mathrm{sym}})_{m\in\mathbb{N}}$ and $(a_m^{\mathrm{anti}})_{m\in\mathbb{N}}$. Putting the bounds obtained for $a_m^{\mathrm{sym}}$ and $a_m^{\mathrm{anti}}$ together, we obtain 
\[ \lambda_n + (\lambda_n + \mu_n)(m/2-1)^2 \pi^2/L^2 \leqslant a_m^{-1} \leqslant \lambda_n + (\lambda_n + \mu_n)(m/2+1)^2 \pi^2/L^2,\]
and 
\[ (\lambda_n + \mu_n)(m-2)^2 \pi^2/(4L^2) \leqslant a_m^{-1} \leqslant (\lambda_n + \mu_n)(m+4)^2 \pi^2/(4L^2).\]
We conclude that
\[ \frac{4L^2}{(\lambda_n + \mu_n)(m+4)^2 \pi^2} \leqslant a_m \leqslant \frac{4L^2}{(\lambda_n + \mu_n)(m-2)^2 \pi^2}.\]

\subsection{Proof of Theorem \ref{prop:kernel_speed_up}}
This is a straightforward consequence of Proposition \ref{prop:1ddiff}, identity \eqref{eq:bornepoly}, and Theorem \ref{thm:boundexp}.

    


\end{document}